\documentclass[twoside]{article}

\usepackage[accepted]{aistats2021}
\usepackage{natbib}
\usepackage[utf8]{inputenc} 
\usepackage[T1]{fontenc}    
\usepackage{hyperref}       
\hypersetup{
    colorlinks=true,
    linkcolor=blue,
    urlcolor=blue,
    linktoc=all,
    citecolor=blue}
\usepackage{url}            
\usepackage{booktabs}       
\usepackage{amsfonts}       
\usepackage{nicefrac}       
\usepackage{microtype}      
\usepackage{todonotes}
\usepackage{multirow}
\usepackage{amsmath,amsthm,amsfonts,amssymb}
\usepackage{enumerate,color,xcolor}
\usepackage{graphicx}

\usepackage{subfig}
\usepackage{caption}

\usepackage{url}

\numberwithin{equation}{section}
\theoremstyle{plain}

\newtheorem{theorem}{Theorem}

\newtheorem{corollary}[theorem]{Corollary}

\newtheorem{lemma}[theorem]{Lemma}

\newtheorem{proposition}[theorem]{Proposition}

\newtheorem{remark}[theorem]{Remark}

\newcommand{\mg}[1]{\textcolor{purple}{#1}}

\usepackage{comment}
\usepackage{wrapfig}

\usepackage{dsfont}
\numberwithin{equation}{section}

\newcommand{\Yuanhan}[1]{\textcolor{red}{#1}}
\newcommand{\beq}{\begin{eqnarray}}
\newcommand{\eeq}{\end{eqnarray}}
\newcommand{\beqs}{\begin{eqnarray*}}
\newcommand{\eeqs}{\end{eqnarray*}}

\DeclareUnicodeCharacter{FB01}{fi}

\begin{document}

\runningauthor{Mert G\"urb\"uzbalaban and Yuanhan Hu}

\twocolumn[
\aistatstitle{Fractional moment-preserving initialization schemes
for training deep neural networks}

\aistatsauthor{ Mert G\"urb\"uzbalaban\\  \And Yuanhan Hu\\ }

\aistatsaddress{\texttt{mg1366@rutgers.edu}\\Department of Management\\ Science and Information Systems\\ 
  Rutgers Business School\\ 
  Piscataway, NJ 08550 \And \texttt{yuanhan.hu@rutgers.edu}\\Department of Management\\ Science and Information Systems\\ 
  Rutgers Business School\\ 
  Piscataway, NJ 08550} 
]

%

 


\begin{abstract}
\vspace{-7pt}
A traditional approach to initialization in deep neural networks (DNNs) is to sample the network weights randomly for preserving the second moment of layer outputs.  On the other hand, recent results show that training with SGD can result in heavy-tailedness in the distribution of the network weights with a potentially infinite variance. This suggests that the traditional approach to initialization may be restrictive as SGD updates do not necessarily preserve the finiteness of the variance of layer outputs. Motivated by this, we develop initialization schemes for fully connected feed-forward networks that can provably preserve any given moment of order $s\in (0,2]$ over the layers for a class of activations including ReLU, Leaky ReLU, Randomized Leaky ReLU and linear activations. These generalized schemes recover traditional initialization schemes in the limit $s \to 2$ and serve as part of a principled theory for initialization. For all these schemes, we show that the network output admits a ﬁnite almost sure limit as the number of layers grows, and the limit is heavy-tailed in some settings. We also prove that the logarithm of the norm of the network outputs, if properly scaled, will converge to a Gaussian distribution with an explicit mean and variance we can compute depending on the activation used, the value of $s$ chosen and the network width, where log-normality serves as a further justiﬁcation of why the norm of the network output can be heavy-tailed in DNNs. We also prove that our initialization scheme avoids small network output values more frequently compared to traditional approaches. Our results extend if dropout is used and the proposed initialization strategy does not have an extra cost during the training procedure. We show through numerical experiments that our initialization can improve the initial stages of training.
\end{abstract}
\vspace{-7pt}
\section{Introduction}
\vspace{-7pt}
Initialization of the weights of a deep neural network (DNN) plays a crucial role on the training and test performance 
\citep{daniely2016toward,hanin2018start,sutskever2013importance} where random weight initialization often yields
a favorable starting point for optimization \citep{daniely2016toward}. A common traditional approach to initialization that goes back to 1990's is to initialize the weights randomly in a way to preserve the variance of the output of each network layer \citep{lecun1998efficient,bottou-88b} which avoids the network to reduce or magnify the norm of the input signal exponentially. For fully-connected networks with a fixed number of neurons $d$ at each layer with linear activations, this can be achieved by setting the bias vectors to zero and sampling the weights in an independent and identically distributed (i.i.d.) fashion from a Gaussian or uniform distribution with mean zero and variance $\sigma^2 = 1/d$ \citep{lecun1998efficient}, proposed originally for tanh activations \citep{kalman1992tanh}. This initialization is referred to as ``Lecun initialization" in the literature. More recently, \cite{he2015delving} showed that the choice of $\sigma^2 = \frac{2}{d}$ keeps the variance constant if ReLU activation is used instead; where the extra factor of 2 is to account for the fact that ReLU output is zero with probability $1/2$ when input is a mean-zero symmetric distribution without an atom at zero. This initialization is sometimes referred to as ``Kaiming initialization" in the literature \citep{luther2019variance}. A similar initialization rule that can preserve the variance for parametric ReLU and Leaky ReLU activations are also developed in \cite{he2015delving}, where parametric ReLU and Leaky ReLU are variants of ReLU, proposed to improve the performance of ReLU \citep{maas2013rectifier}. 

On the other hand, recent research shows that during the training process with stochastic gradient descent (SGD), the distribution of stochastic gradients can become heavy-tailed over time with a non-Gaussian behavior even though in the initial iterations, stochastic gradients may have a Gaussian-like behavior with a finite variance \citep{simsekli2019tail,csimcsekli2019heavy,gurbuzbalaban2020heavy}, where heavy tailedness of a distribution refers to the fact that its tail is heavier than an exponential distribution \citep{foss2011introduction}. 
In this setting, based on the empirical distribution of stochastic gradients, modelling stochastic gradients with an $\alpha$-stable distribution has been proposed \citep{simsekli2019tail,csimcsekli2019heavy}, which is a distribution that does not have a finite variance but rather has a (fractional) moment of order $s$ satisfying $s<\alpha<2$. This heavy-tailed behavior in the stochastic gradients is also naturally inherited by the network weights due to SGD updates and the amount of heavy tail is also related to the batchsize \citep{panigrahi2019non,gurbuzbalaban2020heavy}. Heavy tails for SGD have also been observed in \cite{zhang2019adam} and have been associated with better generalization \citep{martin2019traditional,csimcsekli2020hausdorff}. In \cite{martin2019traditional}, modelling weights of a well-optimized neural network with a ``Pareto distribution" with shape paremeter $\lambda$ is proposed, which is another heavy-tailed distribution with a power law tail \citep{resnick2007heavy} and an infinite variance when $\lambda<2$ but with a finite $s$-th moment for any $s \in (0,\lambda)$. These results regarding the heavy-tailedness of the network weights with a potentially infinite variance suggest that 
the traditional approach of preserving the second moment and variance of layer outputs at the initialization level may be restrictive as SGD updates do not necessarily preserve the finiteness of the variance after all. This raises the question whether more general initialization schemes that can preserve the $s$-th moment can be developed for a given $s \in (0,2]$ rather than the traditional case which covers only $s=2$.\looseness=-1 

\textbf{Contributions.} 
In this paper, we develop a novel class of initialization schemes that can preserve a fractional moment of order $s\in (0,2]$ over the layer outputs during the forward pass. The schemes are applicable to ReLU, parameteric ReLU, Leaky ReLU, Randomized Leaky ReLU and linear activations for fully-connected deep neural networks. 
We then provide experiments to show that our schemes acts as a warm start for SGD in the sense that it improves the training and test accuracy in the initial stages of training over the MNIST and CIFAR-10 datasets compared to traditional initializations. 
The main idea behind our initialization is to initialize the network weights as i.i.d. Gaussian variables  $\sim \mathcal{N}(0,\sigma^2)$ but adjust the variance $\sigma^2$ in a special way as a function of $s$ to keep the $s$-th moment invariant during the forward pass. To our knowledge, the choice of $\sigma^2$ that can preserve the $s$-th moment of layer output vectors has not been studied in the literature before our work. For this purpose, first we develop analytical formulas that express the $s$-th moment of the $k$-th layer output for any $s\in(0,2)$ and in any dimension $d$ for the ReLU, Leaky ReLU, Randomized Leaky ReLU and linear activations (Theorems \ref{thm-moment-relu}, \ref{thm-leaky-relu}). Our proof relies on adapting the techniques of \cite{cohen1984stability} developed for the products of random matrices with i.i.d. Gaussian entries to nonlinear stochastic recursions arising in forward propagation with nonlinear activations and exploiting the piecewise linear structure of ReLU and parametric ReLU activations. This yields explicit formulas regarding how to choose the initialization weight variance $\sigma^2$ to preserve the $s$-th moment (Corollary \ref{coro-crit-sigma-asymp}, \ref{coro-sigma-lin}). Our initialization scheme allows to choose a larger $\sigma^2$ compared to Kaiming initialization, and is the main reason why with our initialization scheme, network outputs small values relatively less frequently so that small gradients occur less frequently at the initialization. In fact, we show that the logarithm of the norm of the network outputs, if properly scaled, will converge to a Gaussian distribution with an explicit mean and variance we can compute as the number of layers grows (Theorem \ref{thm-relu-log-out}, \ref{thm-log-leaky-relu}), where log-normality serves as a further theoretical justiﬁcation of why the norm of the network output can be heavy-tailed in DNNs. Such a log-normality result was previously shown in \cite{hanin2019products} (see also \citep{hanin2018neural}) for ReLU and linear activations in the regime where the width and depth of the network simultaneously tend to infinity, when the weights are initialized from an arbitrary symmetric distribution with fourth moments; however explicit formulas for the asymptotic mean and variance were not given for the finite width regime. Our results are explicit for finite width and are also applicable to parametric ReLU and Leaky ReLU activations, enabling us to show that if the number of layers is sufficiently large, our scheme will have a first-order stochastic dominance property over the traditional Kaiming initialization in the sense of \cite{hadar1969rules} (see Remarks \ref{remark-stoc-dominance} and \ref{remark-non-asymptotic}). Intuitively speaking, the cumulative distribution function (cdf) of the norm of the network output with our initialization will be strictly shifted to the right compared to the cdf of Kaiming initialization (see Figure \ref{fig:relu-log-out}) and therefore will avoid taking smaller values more often. With zero bias vectors and fixed width over layers, we show that $L_p$ and almost sure limits of network outputs can be only zero or infinity depending on whether $\sigma$ exceeds an explicit threshold we provide (Theorem \ref{thm-almost-sure}). If additive noise is added to post-activations, we show that the almost sure limit of output layers is heavy-tailed for linear activations. The results show that forward pass can make the network output and (hence the gradient of the training cost) heavy-tailed if the variance of network weights exceed a certain threshold, even if the weights are i.i.d. Gaussian (Theorem \ref{thm-heavy-tail}), shedding further light into the origins of heavy tails during signal propagation in DNNs. 
Our results extend if dropout \citep{srivastava2014dropout} is used (Remark \ref{remark-dropout}). 
Also, our framework recovers a number of traditional initialization schemes such as Lecun initialization and Kaiming initialization in the limit as $s\to 2$, and therefore serves as a principled theory for initialization. Furthermore, our results extend naturally to convolutional neural networks, which we discuss in the appendix due to space considerations.

\textbf{Related literature.}  There are alternative approaches to initialization based on taking an average of the width of input and output layers to balance off efficient forward propagation with backward propagation \citep{glorot2010understanding, defazio2019scaling}. In this paper, we consider forward propagation, but backward propagation analysis is almost the same for ReLU and Leaky  ReLU activations by simply replacing the number of input layers with number of output layers in the analysis (see e.g. \citep{he2015delving,glorot2010understanding,defazio2019scaling}) and our initialization schemes can in principle be combined with such averaging strategies. There are also many other strategies that enhance signal propagation in deep networks such as orthogonal matrix initialization \citep{saxe2013exact}, random walk initialization \citep{sussillo2014random}, edge of chaos initialization \citep{yang2017mean,hayou2018selection,schoenholz2016deep} and mean field theory based approaches \citep{xiao2018dynamical,quantization}, batch normalization \citep{ioffe2015batch}, composition kernels \citep{daniely2016toward}, approaches for residual networks \citep{yang2017mean,hanin2018start,ling2019spectrum} as well as development of alternative activation functions \citep{klambauer2017self,clevert2015fast,hayou2018selection} and automating the search for good initializations \citep{dauphin2019metainit}.
\textbf{Notation.} We use standard notation, common in the machine learning literature; however we provide a detailed  discussion of the notation used in our paper in the supplementary material (Appendix \ref{sec-notation}). 
\vspace{-7pt}
\section{Preliminaries and Setting}
\vspace{-7pt}
\textbf{Fully connected feed-forward networks and activation functions.} We consider a fully connected feed-forward deep neural network. Given input data $x^{(0)}\in \mathbb{R}^d$, these networks consist of multiple layers. 
The output of the $k$-th layer which we denote by $x^{(k)}$ follows the following recursion:
\begin{eqnarray*}
&x^{(k+1)}= F^{(k+1)}(x^{(k)}), \\
&F^{(k+1)}(x) := \phi_a(W^{(k+1)}x + b^{(k+1)}),
\end{eqnarray*}
where $W^{(k+1)}\in \mathbb{R}^{d\times d}$ and $b^{(k+1)}\in\mathbb{R}$ are the weight matrix and the bias of the $(k+1)$-st layer respectively and the function $\phi_a$ denotes the parametric ReLU activation function \citep{he2015delving} applied component-wise to a vector, defined for a scalar input $z\in\mathbb{R}$ as
\beq \phi_a(z) = \begin{cases} 
    z & \mbox{if } z>0, \\
   az & \mbox{if } z \leq 0,
  \end{cases} 
\eeq  
where $a\in[0,1]$ is a parameter. The parameter $a$ can also be learned from data during training \citep{he2015delving}, but in this paper we are interested in the case where the choice of $a$ will be given and fixed. Depending on the choice of $a$, this class recovers a number of  activation functions of interest:
\begin{enumerate}  
  \item For $a = 0$, $\phi_0(x) = \max(0,x)$ is the rectified linear unit (ReLU) which is widely used in practice \citep{maas2013rectifier}.
  \item 
  For $a=0.01$, this corresponds to Leaky ReLU activation \citep{maas2013rectifier}. More recently, some other choices of $a \in (0,1)$ has also been considered \citep{he2015delving}. If $a$ is chosen randomly, this is referred to as Randomized Leaky ReLU \citep{xu2015empirical}.
  \item For $a=1$, $\phi_1(x) = x$ is the linear activation function.
\end{enumerate}  


\textbf{Gaussian initialization techniques.} We consider Gaussian initialization where the network weights are independent and identically distributed (i.i.d) following a Gaussian distribution with constant variance $\sigma^2$ and mean zero and biases are set to zero, i.e. we assume:
\begin{itemize}
\item [\textbf{(A1)}] All the weights are independent and identically distributed (i.i.d.) with a centered Gaussian distribution satisfying ${W^{(k)}}\in \mathbb{R}^{d\times d}$,  $W^{(k)}_{ij}\sim \mathcal{N}(0,\sigma^2)$ for every $i$,$j$ and $k$ where $\sigma^2>0$ is the variance of the $k$-th layer with width $d$. 
\item [\textbf{(A2)}] The biases are initialized to zero, i.e. $b^{(k)}=0$ for every $k\geq 1$. 
\end{itemize}
For simplicity of the presentation, above we assume that the width of the network is equal to $d$ and is constant over different layers. However, our results naturally extends to the case if each layer $k$ has a different width $d_k$ (see Remark \ref{remark-variable-width}). The popular Kaiming initialization corresponds to the choice of $\sigma^2 = \frac{2}{d(1+a^2)}$ which preserves the second moment of the layer outputs, we will next show that there exists a \emph{critical variance} level $\bar{\sigma}^2_a(s,d)$ that we can compute explicitly, so that the choice of $\sigma^2 = \bar{\sigma}^2_a(s,d)$ will preserve the $s$-th moment of the output over the layers in any dimension $d$ for any $s\in (0,2]$ given. We start with the ReLU case which corresponds to $a=0$. 

\vspace{-7pt}
\section{ReLU Activation}
\vspace{-7pt}
In the next result, we characterize arbitrary moments of the output of the $k$-th layer, i.e. we provide an explicit formula for $\mathbb{E}(\|x^{(k)}\|^s)$ where $s>0$ can be any real scalar where (throughout this paper) and $\|\cdot\|$ denotes the Euclidean $(L_2)$ norm. Our result identifies three regimes: For given width $d$ and moment $s>0$, there exists a threshold $\bar{\sigma}_0(s,d)$ for choosing the standard deviation $\sigma$ of the initialization: If we choose $\sigma = \bar{\sigma}_0(s,d)$, then the network with ReLU activation will preserve the $s$-th moment. The choice of $\sigma$ below (resp. above) this threshold, will lead to $s$-th moment to decay (resp. grow) exponentially fast. The result relies on expressing the output of the layers as a mixture of chi-square distributions with binomial mixture weights based on adaptations of the techniques from \cite{cohen1984stability} from linear stochastic recursions to the  nonlinear case. 
The proof of this result, and the proof of all the other results, can be found in the supplementary material. 

\begin{theorem}\label{thm-moment-relu} \textbf{(Explicit characterization of the critical variance $\bar{\sigma}_0^2(s,d)$)} Consider a fully connected network with an input $x^{(0)}\in \mathbb{R}^d$ and Gaussian initialization satisfying \textbf{(A1)}-\textbf{(A2)} with ReLU activation function $\phi_0(x)=\max(x,0)$. Let $s>0$ be a given real scalar. The $s$-th moment of the output of the $k$-th layer is given by
\begin{eqnarray}\label{def-sigma-relu}
    \mathbb{E} \left[ {\| x^{(k)}\|^s} \right]  = {\|x^{(0)}\|^s} (\sigma^s I_0(s,d))^k,\\ 
    I_0(s,d) = 2^{s/2}\sum_{n=0}^d {d \choose n} \frac{1}{2^d}  \frac{\Gamma(n/2 + s/2)}{\Gamma(n/2)},
\end{eqnarray}
where $\Gamma$ denotes Euler's Gamma function. Then, it follows that we have three possible cases: 
\begin{itemize}
    \item [$(i)$] If $\sigma =\bar{\sigma}_0(s,d)$ where 
    $\bar{\sigma}_0(s,d):=\frac{1}{\sqrt[s]{I_0(s,d)}}
    $, then the network preserves the $s$-th moment of the layer outputs, i.e. for every $k \geq 1$, $\mathbb{E} \left[ {\| x^{(k)}\|^s} \right] = \|x^{(0)}\|^s,$ 
    whereas for any $p>s$, $\mathbb{E} \|x^{(k)}\|^{p} \to \infty$ exponentially fast in $k$.
    \item [$(ii)$] If $\sigma < \bar{\sigma}_0(s,d)$, then
        $\mathbb{E} \left[ {\| x^{(k)}\|^s} \right] \to 0$ exponentially fast in $k$. 
    \item [$(iii)$] If $\sigma >\bar{\sigma}_0(s,d)$, then $\mathbb{E} \left[ {\| x^{(k)}\|^s} \right] \to \infty$ exponentially fast in $k$.
\end{itemize}
\end{theorem}


\begin{remark} \label{coro-crit-relu} \textbf{($s=2$ case)} In the special case of $s=2$, Theorem \ref{thm-moment-relu} yields $I_0(2,d)=d/2$ and $\bar{\sigma}_0^2(2,d) = 2/d$ which corresponds to Kaiming initialization, details of this derivation is in Remark \ref{coro-crit-relu-appendix} in the supplementary material. 
\end{remark}
\begin{remark}\label{remark-variable-width} \textbf{(Variable width $d_k$)} If the width $d_k$ of layer $k$ is not a constant equal to $d$ but instead varying over $k$, then our analysis extends to this case naturally where it would suffice to replace the formula \eqref{def-sigma-relu} with $\mathbb{E} \left[ {\| x^{(k)}\|^s} \right]  = {\|x^{(0)}\|^s} (\sigma^s \prod_{j=1}^k I_0(s,d_j))$.
\end{remark}
A natural question that arises is how does the critical variance depend on $d$ and $s$ when $d$ is large. 
The next result gives precise asymptotics for $\bar{\sigma}_0(s,d)$ in the large $d$ regime. The result relies on careful asymptotics for the Gamma functions and binomial coefficients arising in Theorem \ref{thm-moment-relu}.
\begin{corollary}\label{coro-crit-sigma-asymp} \textbf{(Critical variance $\bar{\sigma}_0(d,s)$ when $d$ is large)} For fixed width $d$ and $s\in (0,2]$, we have 
\begin{eqnarray*}
&\bar{\sigma}_0^2(s,d) = \frac{2}{d} + \frac{5(2-s)}{2d^2}  + o(\frac{1}{d^2}),\\
&\bar{\sigma}_0(s,d) = 
\frac{\sqrt{2}}{\sqrt{d}} + \frac{5\sqrt{2}(2-s)}{8d\sqrt{d}} + 
o(\frac{1}{d\sqrt{d}}).
\end{eqnarray*}
Therefore, it follows from  Theorem \ref{thm-moment-relu} that if $ \sigma^2  = \frac{2}{d} + \frac{5(2-s)}{2d^2},
$
then the network will preserve the moment of order $s+o(\frac{1}{d})$ of the network output.
\end{corollary}

\begin{figure*}[!htbp]
\centering
\begin{minipage}[t]{0.46\textwidth}
\centering
    \subfloat{
    \includegraphics[width=.46\columnwidth]{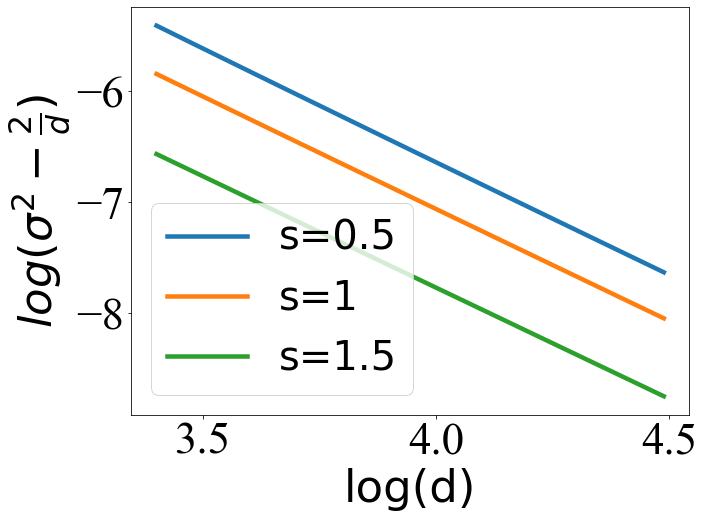}
    \label{fig:sigma-d2}
    }
    \subfloat{
    \includegraphics[width=0.46\columnwidth]{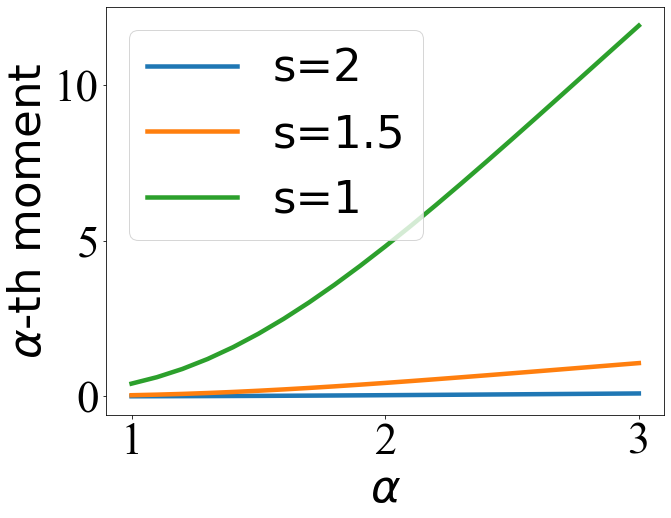}
    }
    \caption{\emph{Left}: Plot of $\log(\bar{\sigma}_0^2(s,d)-\frac{2}{d})$ vs. $\log(d)$ for ReLU. \emph{Right}: Growth of the $\alpha$-th moment of $x^{(k)}$ for $k=500$ and $\sigma = \bar{\sigma}_0(s,d)$ with different $s$ for $d=64$. 
    When $s$ gets smaller, moments grow faster.}
    \label{fig:kmoment-first}
\end{minipage}
\quad \quad
\begin{minipage}[t]{0.46\textwidth}
\centering
    \subfloat{
    \includegraphics[width=0.46\columnwidth]{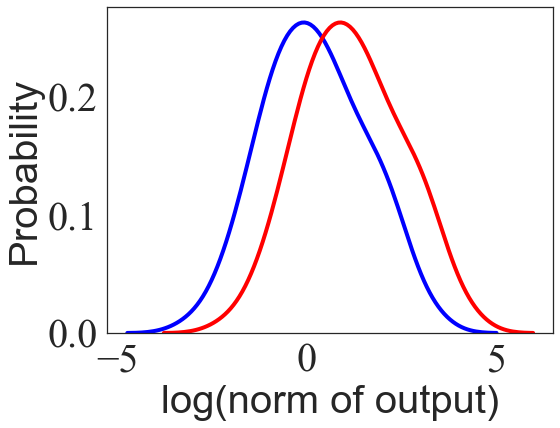}
    }
    \subfloat{
    \includegraphics[width=0.46\columnwidth]{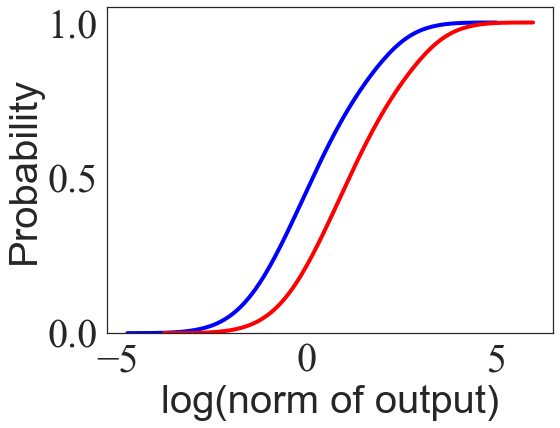}
    }
    \hfill
    \caption{\emph{Left:} Probability density $f_k(r)$ of $R_{{k,0}}$. \emph{Right:} Cumulative density function (cdf) of $R_{{k,0}}$. The blue line in both figures is the result of Kaiming's method, which corresponds to $s=2$. The red line in both figures is the result of our method with $s\approx 1$.}
    \label{fig:relu-log-out}
\end{minipage}
\end{figure*}

According to Corollary \ref{coro-crit-sigma-asymp},  $\log(\bar{\sigma}_0^2(s,d)-\frac{2}{d}) \approx -2\log(d)+ \log(\frac{5(2-s)}{2})$ for large $d$. This is illustrated on the left panel of Figure \ref{fig:kmoment-first} where we plot $\log(\bar{\sigma}_0^2(s,d))$ vs. $\log(d)$ based on the formula \eqref{def-sigma-relu} where we observe the relationships is a straight line with slope approximately $-2$ as predicted by our theory. 
The right panel of Figure \ref{fig:kmoment-first} illustrates part $(iii)$ of Theorem \ref{thm-moment-relu} about how the moments can grow if we choose $\sigma=\bar{\sigma}_0(s,d)$ depending on the value of $s$.

It is not hard to show that under Gaussian initialization with ReLU activation, the network output can be zero with a non-zero probability (see Lemma \ref{lem-zero-proba} in the appendix), which is related to the known "dying neuron" problem associated with ReLU activations \citep{lu2019dying} about the fact that ReLU networks may output zero frequently. The choice of $\sigma$ will clearly affect the variance of $x^{(k)}$ (see Theorem \ref{thm-moment-relu}), however it won't affect the probability that the $k$-th layer output $x^{(k)}=0$. A natural question that arises is what is the effect of $\sigma$ on the growth rate of $\|x^{(k)}\|$ conditional on the event that $x^{(k)}\neq 0$. For this purpose, given an initial point $x^{(0)}\in\mathbb{R}^d$ fixed, we consider the conditional probability density function of $\log(\|x^{(k)}\|)$ given that $x^{(k)}\neq 0$, i.e.
\beq f_k(r) := \mathbb{P}\left(\log(\|x^{(k)}\|) = dr ~\big |~ x^{(k)} \neq 0, a=0 \right). 
\label{def-fky}
\eeq
Let $R_{{k,0}}$ be the random variable corresponding to the density $f_k(r)$. The quantity
\beq \mu_0(\sigma) := \lim_{k\to\infty} \frac{R_{{k,0}}}{k}
\label{eq-def-mu}
\eeq
is a measure of how fast the norm of the output of the layers of the network would grow if we would allow infinitely many layers. It is closely related to the \emph{top Lyapunov exponent} in the probability and dynamical systems literature \citep{arnold1986lyapunov,cohen1984stability} which arises in the study of random Lipschitz maps, see e.g. \citep{ELTON199039}. In the next result, we will obtain an explicit formula for $\mu_0(\sigma)$ (that depends on $\sigma$ and dimension $d$), showing that $\mu_0(\sigma)$ is deterministic and does not depend on the initial point $x^{(0)}$. Furthermore, we show that a properly scaled $R_{{k,0}}$ converges to a Gaussian random variable in distribution, with an explicit mean and variance we can characterize.  
\begin{theorem}\label{thm-relu-log-out}\textbf{(Asymptotic normality of the log. of the norm of the network output)} Consider a fully connected network with an input $x^{(0)}\in \mathbb{R}^d$ and Gaussian initialization satisfying \textbf{(A1)}-\textbf{(A2)} with ReLU activation function $\phi_0(x)=\max(x,0)$. Let $f_k(r)$ be 
the conditional probability density function of $\log(\|x^{(k)}\|)$ given that $x^{(k)}\neq 0$, defined formally by \eqref{def-fky}. 
Let $R_{{k,0}}$ be the random variable corresponding to the density $f_k(r)$. Then, the limit $\mu_0(\sigma)$ defined in \eqref{eq-def-mu} exists, it is deterministic and independent of $x^{(0)}$, satisfying the following formula:
\beq 
\mu_0(\sigma) = \log(\sigma)  +\frac{1}{2} \sum_{n=1}^d \pi_d(n)  
\left[
\log(2) + \psi_0\left(\frac{n}{2}\right)
\right].
\eeq
Furthermore, 
${\frac{R_{{k,0}} - \mu_0(\sigma) k}{\sqrt{k}}}  
\Rightarrow
\mathcal{N}(0, s_0^2)$ in distribution as $k\to\infty$
 with
\beqs 
s_0^2 &=& \frac{1}{4}\sum_{n=1}^d \pi_d(n) \left[\psi_1(\frac{n}{2}) +\left[
\log(2) + \psi_0\left(\frac{n}{2}\right)
\right]^2
\right]\\
& & -\frac{1}{4} \left( \sum_{n=1}^d \pi_d(n) \left[
\log(2) + \psi_0\left(\frac{n}{2}\right)
\right]\right)^2 , 
\eeqs
where $\psi_0$ is the di-gamma function, $\psi_1$ is the tri-gamma function and 
$\pi_d(n) = {d \choose n} \frac{1}{2^d-1}.
$
\end{theorem} 
\begin{remark}\label{remark-stoc-dominance}\textbf{(First-order stochastic dominance property compared to Kaiming's method)}
Theorem \ref{thm-relu-log-out} shows that the logarithm of the norm of the $k$-th layer output $R_{{k,0}}$ will be asymptotically normal as $k\to\infty$ if $R_{{k,0}}$ is properly scaled, where the choice of $\sigma$ will only affect the mean (but not the variance) of the asymptotic normal distribution. This is illustrated in Figure \ref{fig:relu-log-out} where we plot the probability density function (pdf) on the left panel and the cumulative density function (cdf) of $R_{{k,0}}$ on the right panel where the pdf of $R_{k,0}$ has a Gaussian shape. We compare two initializations $\sigma^2 = \frac{2}{d}$ (Kaiming initialization which preserves variances) and our initialization technique $\sigma^2 = \frac{2}{d} + \frac{5}{2d^2}$ which preserves the moment of order $s=1+o(\frac{1}{d})$. We used $k=100$ layers and dimension $d=64$. We observe from the cdf's of network outputs on the right panel of Figure \ref{fig:relu-log-out} that with our choice of $\sigma$, the norm of the network output is larger in the sense that it has first-order stochastic dominance \citep{hadar1969rules} relative to Kaiming initialization. Since our results also admit non-asymptotic versions (Remark \ref{remark-non-asymptotic}), this dominance property will hold provably for large enough but finite $k$ as well (due to the fact that our initialization 
results in a larger mean value $\mu_a(\sigma)$ in the setting of Theorem \ref{thm-moment-relu}).
\end{remark}
\vspace{-7pt}
\section{Parametric ReLU, Randomized Leaky ReLU and Linear Activations}\label{sec-param-relu}
\vspace{-7pt}
For parametric ReLU activations with $a>0$, we develop an analogous result to Theorem \ref{thm-moment-relu} which characterize $s$-th moments of the $k$-th layer output $x^{(k)}$ for $s\in (0,2]$. 
\begin{theorem}\label{thm-leaky-relu}\textbf{(Explicit characterization of the critical variance $\bar{\sigma}_a^2(s,d)$)} Consider a fully connected network with an input $x^{(0)}\in \mathbb{R}^d$ and Gaussian initialization satisfying \textbf{(A1)}--\textbf{(A2)} with activation function $\phi_a(x)$ for any choice of $a\in (0,1]$ fixed. Then, for any $s\in (0,2]$, the output of the $k$-th layer satisfies
\begin{eqnarray}\mathbb{E} \left[ {\| x^{(k)}\|^s} \right]  = {\|x^{(0)}\|^s} (\sigma^s I_a(s,d))^k
\label{eq-leaky-relu-moments}
\end{eqnarray} 
with 
\begin{equation}
\begin{split}
I_a(s,d)=& 2^{s/2} \frac{1}{\Gamma(1-s/2)}\\
& \sum_{n=0}^d {{d \choose n} \frac{1}{2^d}} 
 \sum_{k=0}^\infty w_{k,n}    \mathrm {B}(k+1-\frac{s}{2},\frac{d}{2}+\frac{s}{2})
 \end{split}
 \label{def-iasd-theorem}
 \end{equation}
with the convention that $I_a(2,d)  = (1+a^2) \frac{d}{2}$,
where $\mathrm {B(\cdot,\cdot)}$ is the Beta function and
\begin{equation}
\begin{split}
w_{k,n} =& \frac{1}{2}(1-a^2)^k \\
& \left[ 
{\frac{d-n}{2}+k-1 \choose k} n + a^2(d-n){\frac{d-n}{2}+k \choose k} 
\right].
\end{split}
\end{equation}
Let {$\bar{\sigma}_a(s,d)=\frac{1}{\sqrt[s]{I_a(s,d)}}$}. We have three possible cases: 
\begin{itemize}
    \item [$(i)$] If $\sigma =\bar{\sigma}_a(s,d)$, 
    then the network preserves the $s$-th moment of the layer outputs, i.e. for every $k \geq 1$, $\mathbb{E} \left[ {\| x^{(k)}\|^s} \right] = \|x^{(0)}\|^s,$ 
    whereas for any $p>s$, $\mathbb{E} \|x^{(k)}\|^{p} \to \infty$ exponentially fast in $k$.
    \item [$(ii)$] If $\sigma < \bar{\sigma}_a(s,d)$, then
        $\mathbb{E} \left[ {\| x^{(k)}\|^s} \right] \to 0$ exponentially fast in $k$. 
    \item [$(iii)$] If $\sigma > \bar{\sigma}_a(s,d)$, then $\mathbb{E} \left[ {\| x^{(k)}\|^s} \right] \to \infty$ exponentially fast in $k$.
\end{itemize}
\end{theorem}
\begin{remark}{\textbf{(Extension to Randomized Leaky ReLU)}} For Randomized Leaky ReLU activation, $a$ is chosen randomly. Theorem \ref{thm-leaky-relu} extends simply by replacing $w_{k,n}$ with $\mathbb{E}[w_{k,n}]$ where the expectation is taken with respect to the distribution of $a$. For instance, with a uniform distribution over an interval $[\ell,u]$ with default values of $\ell=\frac{1}{3}$ and $u=\frac{1}{8}$ \citep{xu2015empirical}, $\mathbb{E}[w_{k,n}]$ can be expressed with a closed-form formula as all the moments of the uniform distribution is explicitly known \citep{walck1996hand}.  
\end{remark}


We can also show that $\bar{\sigma}_a(s,d)$ possesses some monotonicity properties.
\begin{corollary} \textbf{(Monotonicity properties of $\bar{\sigma}_a(s,d)$)} \label{coro-crit-leaky} In the setting of Theorem \ref{thm-leaky-relu}, for  $(a,s,d)\in [0,1]\times (0,\infty)\times \mathbb{Z}_+$, the function $(a,s,d) \mapsto \bar{\sigma}_a(s,d)$ is a monotonically (strictly) decreasing function of $a,d$ and $s$.
\end{corollary}
Figures~\ref{fig:sigma_1}--\ref{fig:sigma_2} illustrate $\bar{\sigma}_0(s,d)$ as a function of $d$ when $s$ is fixed where we see a monotonic behavior as proven in Corollary \ref{coro-crit-leaky}. We also observe in the figures that it is a monotonically decreasing function of $s$ when $d$ is fixed. Next, we characterize how $\bar{\sigma}_a(d,s)$ behaves for large $d$.
\begin{figure}[h!]
\centering
    \subfloat[$\log(\bar{\sigma}_0(s,d))$ versus $d$]{
    \includegraphics[width=0.45\columnwidth]{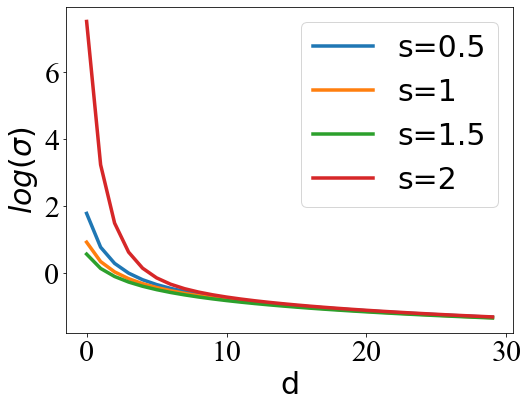}
    \label{fig:sigma_1}
    }
    \hfill
    \subfloat[$\bar{\sigma}_0(s,d)$ versus $s$]{
    \includegraphics[width=0.45\columnwidth]{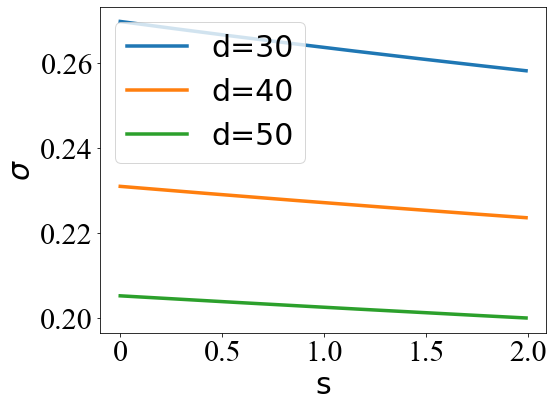}
    \label{fig:sigma_2}
    }
\caption{Dependency of $\bar{\sigma}_0(s,d)$ to parameters $s$ and $d$.}
\label{fig:sigma}
\end{figure}

\begin{corollary}\textbf{(Critical variance $\bar{\sigma}_a(d,s)$ when $d$ is large)}\label{coro-sigma-lin} For fixed width $d$ and $s\in (0,2]$, we have 
$ \bar{\sigma}_1^2(s,d) = \frac{1}{2}\left(\frac{\Gamma(\frac{d}{2})}{\Gamma(\frac{d}{2}+\frac{s}{2})}\right)^{2/s} =  \frac{1}{d} + \frac{(2-s)}{2d^2}  + o(\frac{1}{d^2})
$ with $\bar{\sigma}_1^2(2,d)=\frac{1}{d}$ in the special case $s=2$ which corresponds to Lecun initialization. Therefore, it follows from  Theorem \ref{thm-leaky-relu} that if 
$ \sigma^2  = \frac{1}{d} + \frac{(2-s)}{2d^2},
$
then the network with linear activation will preserve the moment of order $s+o(\frac{1}{d})$ of the network output. More generally, for $a>0$ small, we have 
\begin{equation*}
    \begin{split}
\bar{\sigma}_a^2(s,d) =& \left(\frac{2}{1+a^2}\right)\frac{1}{d} +  \left(\frac{5 -(12-\frac{5}{2}s)a^2}{2+(s+2)a^2}\right)\frac{(2-s)}{d^2}\\
&+ \mathcal{O}(\frac{a^4}{d}) + o(\frac{1}{d{^2}}). 
\end{split}
\end{equation*}

\end{corollary}


Similar to Corollary \ref{coro-crit-sigma-asymp} for the ReLU case, we can express $\bar{\sigma}_a^2(s,d)$ as a function of $s$ for large $d$.
Thanks to Corollary \ref{coro-sigma-lin}, we can approximate $\bar{\sigma}_a^2(s,d)$ explicitly for Leaky ReLU with $a=0.01$ without evaluating the double sums in \eqref{def-iasd-theorem}. Leaky ReLU and linear activations do not output zero unless their input is zero; due to their piecewise linear structure. This is why, they can solve the ``dying neuron" problem of ReLU activations to a certain extent \citep{lu2019dying}. Consequently, under Gaussian initialization $\textbf{(A1)}$--$\textbf{(A2)}$ with for Leaky ReLU and linear activations, i.e. when $a \in (0,1]$, for any $\sigma>0$ given, it is straightforward to show that $ \mathbb{P}(x^{(k)}= 0)= 0$. Similar to our discussion for ReLU activations, we introduce
\begin{equation*}
\begin{split}
 f_{k,a}(r) :=& \mathbb{P}\left(\log(\|x^{(k)}\|)= dr ~\big |~ x^{(k)} \neq 0\right)\\
 =&\mathbb{P}\left(\log(\|x^{(k)}\|) = dr \right),
\end{split}
\label{def-fky-leaky}
\end{equation*}
where we used $\mathbb{P}(x^{(k)}= 0)= 0$ for $a\in (0,1]$. Let $R_{k,a}$ be the random variable corresponding to the density $f_{k,a}(r)$. The quantity
\beq \mu_a(\sigma) := \lim_{k\to\infty} \frac{R_{k,a}}{k} \quad \mbox{for} \quad a \in (0,1],
\label{eq-Lyap-exp}
\eeq
is called the \emph{top Lyapunov exponent} for the random Lipschitz map $x^{k+1}=\phi_a(W^{k+1}x^k)$ where $\sigma$ scales the $W^{k+1}$ term. The following theorem derives an explicit formula for $\mu_a(\sigma)$ and shows that $R_{k,a}$ is asymptotically normal if it is properly scaled for parametric ReLU.

 
\begin{theorem}\label{thm-log-leaky-relu} \textbf{(Asymptotic normality of the log. of the norm of the network output)} Consider a fully connected network with an input $x^{(0)}\in \mathbb{R}^d$ and Gaussian initialization satisfying \textbf{(A1)}-\textbf{(A2)} with Leaky ReLU activation function $\phi_a(x)$ with $a\in(0,1]$. Let $f_k(r)$ be 
the conditional probability density function of $\log(\|x^{(k)}\|)$ given that $x^{(k)}\neq 0$, defined formally by \eqref{def-fky-leaky}. 
Let $R_{{k,a}}$ be the random variable corresponding to the density $f_{{k,a}}(r)$. Then, the limit $\mu_a(\sigma)$ defined in \eqref{eq-Lyap-exp} exists, it is deterministic and independent of $x^{(0)}$, explicitly given by the formula \eqref{def-mu-a-sigma}
in the supplementary material. Let $R_{k,a}$ be the random variable corresponding to the density $f_{k,a}(r)$. Then, 
${\frac{R_{k,a} - \mu_a(\sigma) k}{\sqrt{k}}}  
\Rightarrow
\mathcal{N}(0, s_a^2)$ in distribution as $k\to\infty$
where $s_a^2$ is defined by \eqref{eq-def-sa} in the supplementary material.
\end{theorem}

\begin{figure}[h!]
\centering
    \subfloat[Linear]{
    \includegraphics[width=0.45\columnwidth]{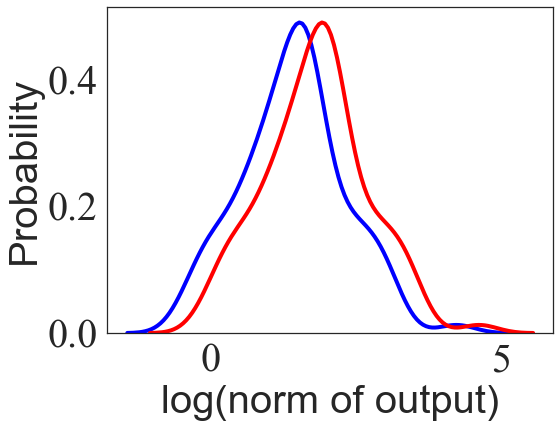}
    \label{fig:logoutput1b}
    }
    \hfill
    \subfloat[Linear]{
    \includegraphics[width=0.45\columnwidth]{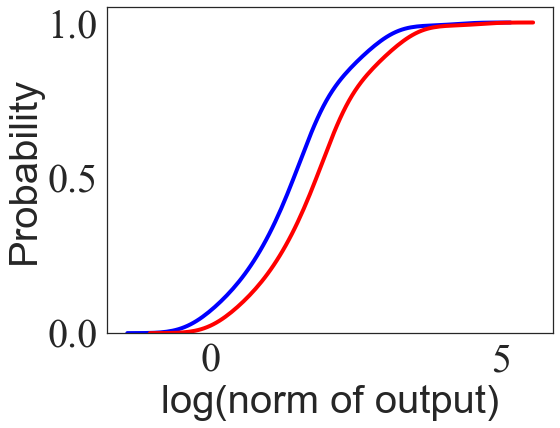}
    \label{fig:cdfoutput1}
    }
    \hfill
    \subfloat[Leaky-ReLU]{
    \includegraphics[width=0.45\columnwidth]{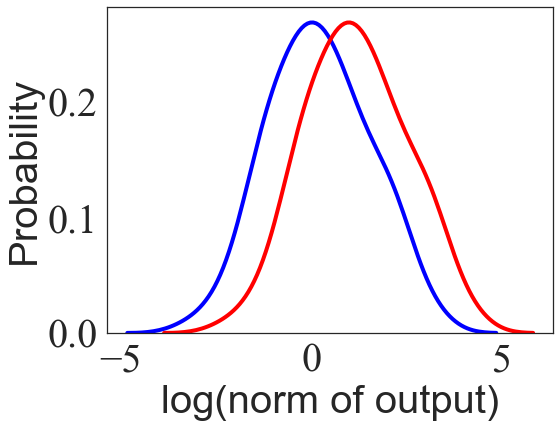}
    \label{fig:logoutput3b}
    }
   \hfill
    \subfloat[Leaky-ReLU]{
    \includegraphics[width=0.45\columnwidth]{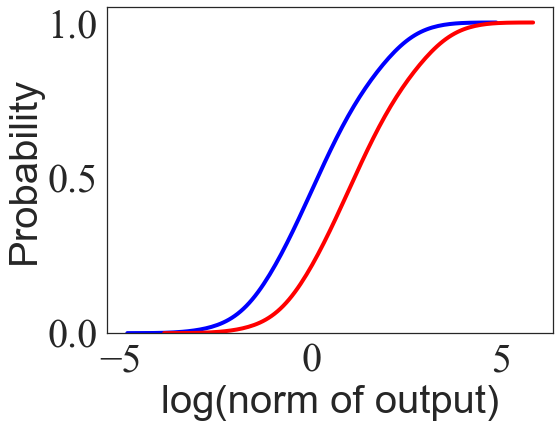}
    \label{fig:cdfoutput3}
    }
    \hfill
\caption{Distribution of the natural logarithm of the norm of the output $R_{k,a}$ through 100 layers with Linear $(a=1)$ and Leaky ReLU activation with $a=0.01$. The blue line in all figures is the result of Kaiming's method, the red line is the result of our initialization. (a): Probability density of $R_{k,1}$ where choose $\sigma=\bar{\sigma}_1(1,d)$.  (b): Cumulative density function of $R_{k,1}$. (c): Probability density of $R_{k,a}$ for $a=0.01$ where we choose $\sigma=\bar{\sigma}_a(1,d)$. 
(d): Cumulative density function of $R_{k,a}$ for $a=0.01$.}
\label{fig:logoutputb}
\end{figure}
\begin{remark}\label{remark-non-asymptotic}
\textbf{(Non-asymptotic version of Theorems \ref{thm-relu-log-out} and \ref{thm-log-leaky-relu} and stochastic dominance)} Theorems \ref{thm-relu-log-out} and \ref{thm-log-leaky-relu} are based on invoking the central limit theorem (CLT) in its proof. If we use  a non-asymptotic version of the CLT instead such as the Berry–Esseen theorem \citep{berry1941accuracy}, the results extend to finite $k$ in a straightforward fashion. In Figure \ref{fig:logoutputb}, we illustrate Theorem \ref{thm-log-leaky-relu} where we plot the distribution of the natural logarithm of the norm of the output $R_{k,a}$ and observe a Gaussian behavior. In the supplementary material (Remark \ref{remark-stoc-dominance-2}), we also discuss the stochastic dominance properties of $s<2$ with respect to $s=2$.
\end{remark}

\begin{remark}\label{remark-dropout} \textbf{(Extension of results to dropout)} Dropout is a popular technique that randomly removes some neurons to prevent overfitting \citep{srivastava2014dropout}. In this case, with zero bias, the layer recursion becomes $x^{(k+1)}:= \phi_a(W^{(k+1)}(x^{(k)} \odot \varepsilon^{(k+1)})$ where $\odot$ denotes component-wise multiplication and $\varepsilon^{(k+1)}$ is a scaled Bernouilli random variable with i.i.d. components satisfying  $\mathbb{P}(\varepsilon^{(k+1)}_i = 0) = 1-q$ and $\mathbb{P}(\varepsilon^{(k+1)}_i = \frac{1}{q}) = q$ where $q$ is the probability to keep a neuron with $q\in (0,1]$ (see e.g. \citep{pretorius2018critical}). All the results in this paper generalize naturally with minor modifications (such as scaling with $q$) in the results if dropout is used (see Appendix \ref{sec-dropout}). For example, for any $s\in (0,2]$ and $q\in (0,1]$, the critical threshold for ReLU with dropout becomes 
$\sigma_{0,q}^2(s,d) = \frac{2q}{d} + \frac{2-s}{2d^2}(6-q) + \mathcal{O}(\frac{1}{d^2})$
where our analysis recovers the results of Corollary \ref{coro-crit-sigma-asymp} in the special case when $q=1$ and results of \cite{pretorius2018critical} when $s=2$.
\end{remark}
The following theorem shows that if bias vectors are zero, then with both Kaiming initialization and our initialization, the network outputs will converge to an almost sure limit of zero, even if the network outputs preserve moments of order $s$ for every layer $k$. Roughly speaking, the reason this happens is that the network preserves the moments in a highly anisotropic manner, output is often zero but also can occassionally take large values so that $s$-th moment is preserved. This supports empirical results of \cite[Sec. 3]{saxe2013exact} that observed this anisotropic behavior for linear activations. Our result shows that similar behavior happens with non-linear activations. We also show that depending on the sign of $\mu_a(\sigma)$, both $L_p$ limit and a.s. limit can be only zero or infinity. In the special case of $s=2$ (for Kaiming initialization), such a convergence result in $L_2$ was previously proven in \cite[Thm. 5]{hanin2018start} where layer widths can take arbitrary values.
\begin{theorem}\label{thm-almost-sure} \textbf{($\sigma$ determines the almost sure (a.s.) and $L_p$ limit)} Consider Gaussian initialization $\textbf{(A1)}$--$\textbf{(A2)}$ with activation function $\phi_a(x)$ with $a\in[0,1]$ and input $x^{(0)}\neq 0$. For ReLU, i.e. when $a=0$, regardless of the choice of $\sigma$, the network output $x^{(k)}$ converges to zero a.s as $k\to\infty$. For parametric ReLU or for linear activations, i.e. when $a\in (0,1]$, then $\mu_a(\sigma) < 0$ if and only if $\sigma = {\bar{\sigma}_{a}(s,d)}$ for some $s>0$ and in this case $x^{(k)}$ converges to zero a.s. for $s\geq 1$ and for $s<1$, $x^{(k)}$ converges in $L_p$ for any $p \in (0,s)$ and has a subsequence that converges to zero almost surely, where $\mu_a(\sigma)$ is as in Theorem \ref{thm-log-leaky-relu}. On the other hand, if $\mu_a(\sigma) > 0$, then the sequence $x^{(k)}$ converges to infinity in $L_p$ for any $p>0$ and $x^{(k)}$ has a subsequence that converges to infinity a.s.

\end{theorem}
If additive zero mean Gaussian noise is added to network outputs for linear activations, we can prove that the limit is non-zero and heavy tailed whereas the limit is zero without noise injection. Our results provides a theoretical support for experimental results of \cite{poole2014analyzing} where additive noise was observed to improve performance by spreading information propagation more evenly across the network. The results also show that  forward pass can make the network output (and subsequently the gradient of the training cost) heavy-tailed for Gaussian initialization.
\begin{theorem}\label{thm-heavy-tail} \textbf{(Heavy-tailed a.s. limit)} Under Gaussian initialization $\textbf{(A1)}$--$\textbf{(A2)}$ with a linear activation function $\phi_1(x)$, input $x^{(0)}\neq 0$ and $\sigma=\bar{\sigma}_1(s,d)$ for some $s\in (0,2)$ where $\bar{\sigma}_1(s,d)$ is given explicitly in Corollary \ref{coro-sigma-lin}, if additive i.i.d. mean-zero Gaussian noise is added component-wise to post-activations, then the layer outputs $x^{(k)}$ admit a non-zero almost sure limit that is heavy tailed in the sense that it has infinite variance and its moments of order $p$ are infinite for any $p>s$.
\end{theorem}
\vspace{-7pt}
\section{Numerical Experiments}
\vspace{-7pt}
We compared our initialization method with Kaiming initialization \citep{he2015delving} and Xavier method \citep{glorot2010understanding} on fully connected networks with linear, ReLU, Leaky ReLU activation functions. For the ReLU and linear activations, we also compared our method with random walk initialization \citep{sussillo2014random}, which does not have explicit parameters for the Leaky ReLU but directly applicable to linear and ReLU activations. We only compare our method with initialization strategies that do not take additional CPU time during training for a fair comparison. We report train loss, test loss, train accuracy and test accuracy over first 30 epochs of training with SGD to focus on the impact of initialization on two benchmark problems: MNIST \citep{lecun1998gradient} and CIFAR-10 \citep{krizhevsky2009learning}.
\begin{figure}[h!]
\centering
    \subfloat[Train loss]{
    \includegraphics[width=0.4\columnwidth]{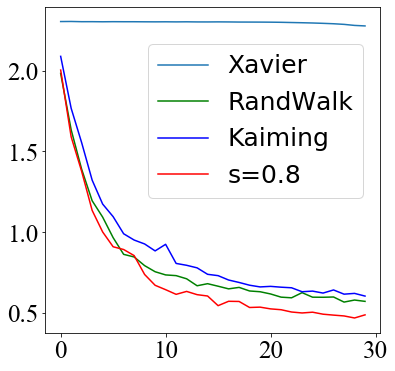}
    \label{fig:relu_1}
    }
    \hfill
    \subfloat[Test loss]{
    \includegraphics[width=0.4\columnwidth]{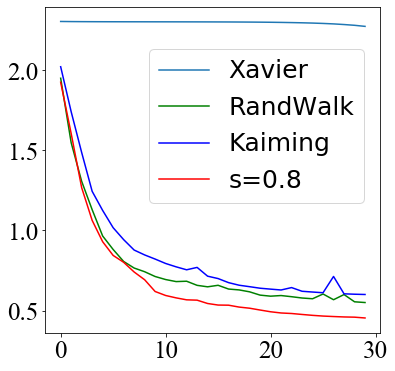}
    \label{fig:relu_2}
    }
    \hfill
    \subfloat[Train accuracy]{
    \includegraphics[width=0.4\columnwidth]{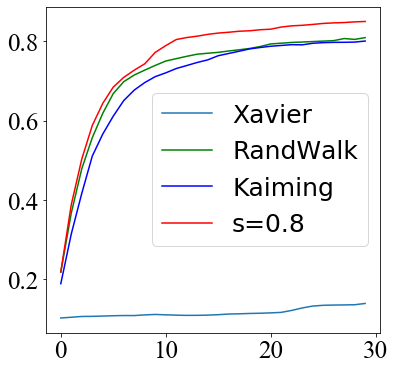}
    \label{fig:relu_3}
    }
    \hfill
    \subfloat[Test accuracy]{
    \includegraphics[width=0.4\columnwidth]{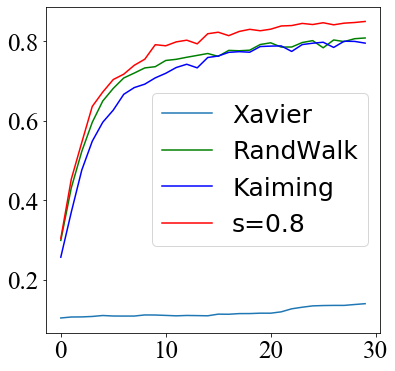}
    \label{fig:relu_4}
    }
    \\
    \hfill\\
\resizebox{\columnwidth}{!}{
\begin{tabular}{|l|l|l|l|l|}
\hline
\multirow{2}{*}{} & \multicolumn{2}{l|}{train loss} & \multicolumn{2}{l|}{test loss} \\ \cline{2-5} 
                  & mean                & std       & mean               & std       \\ \hline
Xavier            & 2.2761              & 0.0349    & 2.2723             & 0.0387    \\ \hline
Randwalk          & 0.5712              & 0.4097    & 0.5498             & 0.4211    \\ \hline
Kaiming           & 0.6039              & 0.426     & 0.6                & 0.42      \\ \hline
s=0.8             & \textbf{0.4877}     & 0.3059    & \textbf{0.4535}    & 0.314     \\ \hline
\multirow{2}{*}{} & \multicolumn{2}{l|}{train acc}  & \multicolumn{2}{l|}{test acc}  \\ \cline{2-5} 
                  & mean(\%)            & std       & mean(\%)           & std       \\ \hline
Xavier            & 13.9                & 0.0299    & 13.99              & 0.0319    \\ \hline
Randwalk          & 80.94               & 0.1649    & 80.82              & 0.1604    \\ \hline
Kaiming           & 80.08               & 0.147     & 79.53              & 0.1488    \\ \hline
s=0.8             & \textbf{85.02}      & 0.1174    & \textbf{84.98}     & 0.118     \\ \hline
\end{tabular}}
\caption{Fully connected network with width $d=64$ and depth $20$ for ReLU activation on MNIST. The plots are the average results over 20 runs, the mean and standard deviations (std) for runs are provided as a table. The $x$-axis represents the epoch number.}
\label{fig:relu}
\end{figure}

Figure \ref{fig:relu} is the summary of our results for MNIST with ReLU activation with mean and standard deviation (std) of the runs reported over 20 runs, where we see a clear improvement with our initialization for $s=0.8$. More specifically, within our initialization, we chose $\sigma^2 =\frac{2}{d} + \frac{3}{d^2}$ which preserves the moment $s\approx 0.8$ according to Corollary \ref{coro-crit-sigma-asymp} where $d=64$ over 20 layers. Further details of the experimental setup, results for ReLU and linear activations and our experiments on the CIFAR-10 dataset can be found in Section \ref{sec-numerical-appendix} of the appendix where we observed qualitatively similar results and our initialization method often improved performance.

\textbf{Heavy-tailed gradients at initialization.} The works \citep{simsekli2019tail,csimcsekli2019heavy,gurbuzbalaban2020heavy} consider training of fully-connected and convolutional neural networks with SGD and standard initialization techniques and argue that the distribution of the gradients become often more and more heavy tailed over time. To be more specific, the numerical experiments in these works suggest that with traditional initialization approaches, the stochastic gradients have often light tails in the first epochs of SGD iterations but the tails become heavier over time as the number of epochs increases while the weights are being optimized. Such observations are also consistent and inline with the earlier results of \cite{martin2019traditional}. Also, these results together with \cite{csimcsekli2020hausdorff} suggest that heavy tails often lead to better exploration and generalization properties. Our initialization technique allows this `favorable heavy-tailed phase' to kick in earlier, right at the beginning of SGD iterations as opposed to later epochs of training. This is illustrated in Figure~\ref{fig:tailindex} which displays the tail index of gradient noise over iterations with our initialization, where the tail index is defined as the value of $\alpha$ such that the pdf $p(x)$ of the gradient noise is on the order of $1/\|x\|^{\alpha+1}$ when $\|x\|$ is large enough (see \citep{simsekli2019tail,gurbuzbalaban2020heavy} for more details on the tail index). Figure~\ref{fig:tailindex} is based on a fully-connected network with 5 layers with width 64. We use ReLU activation on the MNIST dataset where we take the batch size to be 32 with $s=1$. We use the same estimator from \cite{simsekli2019tail} for the tail index. We observe in Figure~\ref{fig:tailindex} the heavy tails arise starting from the initial iterations with a tail index $\alpha$ around 1 as expected.
\begin{figure}[h!]
\centering
\includegraphics[width=0.4\columnwidth]{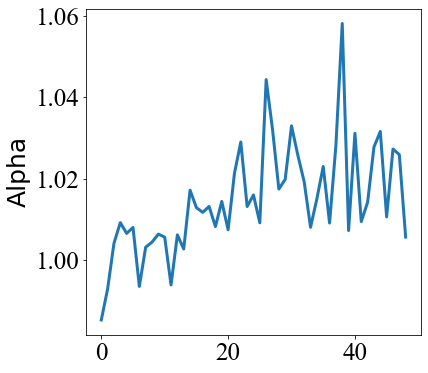}
\caption{Tail index of gradient noise over epochs.}
\label{fig:tailindex}
\end{figure}
\vspace{-7pt}

\textbf{Results on convolutional neural networks.} 
In Section $N$ of the appendix, we provide the extensions of our theoretical results to convolutional neural networks. Here, we present our numerical experiments in Figure~\ref{fig:cnn}, where we used one convolutional layer and four fully-connected layers using ReLU with width $d=64$ on MNIST and CIFAR-10 datasets. We train our networks with stochastic gradient (SGD). The stepsizes of both experiments are tuned and are same for the initializations. Figure~\ref{fig:cnn1} and ~\ref{fig:cnn2} display the first 30 and 50 epoches of the training process of MNIST. With this architechture, after 50 iterations, we achieved an accuracy of \%98.36 which is at a level of current state-of-the-art (see \url{https://benchmarks.ai/mnist} for benchmarks) on MNIST, where we see improvement compared to Kaiming initialization, especially in the first 30 epochs. Figure~\ref{fig:cnn3} shows the first 50 epochs of training processes on CIFAR-10, where we also see the improvement.
\vspace{-7pt}
\begin{figure}[h!]
\centering
    \subfloat[MNIST 30 epoches]{
    \includegraphics[width=0.42\columnwidth]{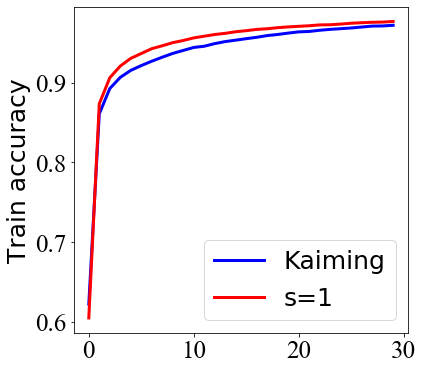}
    \label{fig:cnn1}
    }
    \hfill
    \subfloat[MNIST 50 epoches]{
    \includegraphics[width=0.42\columnwidth]{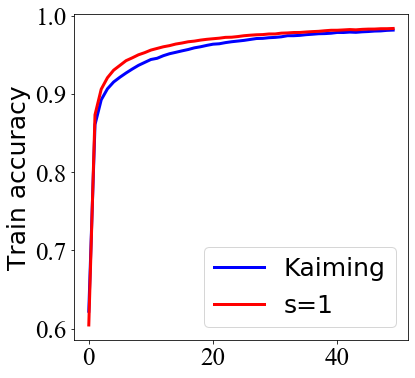}
    \label{fig:cnn2}
    }
    \hfill
    \subfloat[CIFAR-10]{
    \includegraphics[width=0.42\columnwidth]{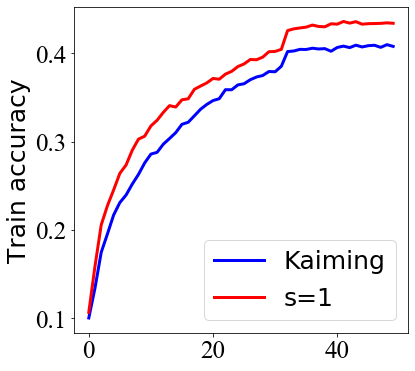}
    \label{fig:cnn3}
    }
\caption{CNN on MNIST and CIFAR-10.}
\label{fig:cnn}
\end{figure}
\vspace{-9pt}
\section{Conclusion}
\vspace{-7pt}
In this paper, we have developed a new class of initialization schemes for fully-connected neural networks with ReLU, parameteric ReLU, Leaky ReLU, Randomized Leaky ReLU or linear activations. Our schemes can preserve a fractional moment of order $s\in (0,2]$ over the layer outputs therefore generalize existing schemes which correspond to the special case $s=2$.

For all these schemes, we show that the network output admits a ﬁnite almost sure limit as the number of layers grows, and the limit is heavy-tailed in some settings. We also prove that the logarithm of the norm of the network outputs, if properly scaled, will converge to a Gaussian distribution with an explicit mean and variance we can compute. We also prove that our initialisation scheme avoids small network output values more frequently compared to traditional approaches, therefore can alleviate the dying neuron problem seen in ReLU networks that results in small network output values. We also provided numerical experiments that show that the new schemes can lead to improvement in the training process. 
\section*{Acknowledgements}
Mert
Gürbüzbalaban and Yuanhan Hu acknowledge support from the grants NSF DMS-1723085 and NSF CCF-1814888.
\bibliographystyle{apalike}
\bibliography{deep}


\appendix
\section{Notation}\label{sec-notation}
Let $\|\cdot\|$ denote the Euclidean $(L_2)$ norm. We use $x_i$ to denote the $i$-th component of a vector, and $A_{ij}$ to denote the entries of a matrix $A$. For a $d\times d$ symmetric positive semi-definite matrix $\Sigma$, the notation $\mathcal{N}(m,\Sigma)$ denotes the $d$-dimensional multi-variate normal distribution with mean $m$ and covariance matrix $\Sigma$. $I_d$ denotes the $d\times d$ identity matrix. For a real number $r\in \mathbb{R}$ and a non-negative integer $k$, we introduce the binomial coefficient 
$${r \choose k} := \frac{r  (r-1) (r-2)\cdots (r-k+1)}{k!}.$$ For $z>0$, (Euler's) Gamma function is defined as the integral $${\displaystyle \Gamma (z)=\int _{0}^{\infty }x^{z-1}e^{-x}\,dx}.$$ 
For $z>0$, the digamma function is defined as $$\psi_0(z): = \left(\frac{d}{dz}\Gamma(z)\right)/\Gamma(z)$$ and the trigamma function $$\psi_1(z):=\frac{d}{dz}\psi_0(z).$$ 
For $x,y>0$, the Beta function is defined as the integral $${\displaystyle \mathrm {B} (x,y):=\int _{0}^{1}t^{x-1}(1-t)^{y-1}\,dt}.$$ If a sequence of random variables $X_k$ converges to a random variable $X$ in distribution as $k \to \infty$, we denote this by $X_k \Rightarrow X$. The set of positive integers will be denoted by $\mathbb{Z}_+$. Let $f, g$ be real valued functions, defined on some unbounded subset of $\mathbb{R}$, and let $g(x)$ be strictly positive for all large enough values of $x$. We denote $f(x)=\mathcal{O}(g(x))$ as $x \to \infty$ if there exists a constant $M>0$ and $x_0 \in \mathbb{R}$ such that $|f(x)| \leq Mg(x)$ for all $x \geq x_0$. We denote $f(x)=o(g(x))$ as $x \to \infty$ if for all $\varepsilon >0$ there exists a constant $x_0$ such that $|f(x)| \leq \varepsilon g(x)$ for all $x\geq x_0$. 
\section{Proof of Theorem \ref{thm-moment-relu}}
\begin{proof} When the activation function is linear, $k$-th layer output $x^{(k)}$ obeys a linear recursion where the proof technique of \cite{cohen1984stability} about the moments of random Gaussian matrix products are directly applicable. Our main proof idea is to extend this proof technique to non-linear recursions obeyed by $x^{(k)}$ when ReLU activation is  used, where we exploit piecewise-linearity properties of the ReLU function. 
We first note that for $x^{(k)}\neq 0$,
\begin{equation}
    \begin{split}
        \frac{\|x^{(k+1)}\|}{\|x^{(k)}\|} =&\frac{ \|\phi_{{0}}(W^{(k+1)}x^{(k)})\|}{\|x^{(k)}\|} \\
        =& \left\| \phi_{{0}}\left(W^{(k+1)} \frac{x^{(k)}}{\|x^{(k)}\|}\right)\right\| ,
    \end{split}
    \label{eq-ratio-norm}
\end{equation} 

which we used the equality $\phi_0(cy) = c\phi_0(y)$ for any given $c>0$ and arbitrary vector $y$ (with the choice of $y = W^{(k)}x^{(k)}$ and $c= 1/\|x^{(k)}\|$). On the other hand, the entries of the $W^{(k)}$ matrix are i.i.d. Gaussians, where each row is a spherically symmetric random vector (in the sense of \cite[Ch. 4]{Fourdrinier2018}) with i.i.d. entries. From this symmetry property it follows that the distribution of $W^{(k)}z$ is independent of the choice of $z$ on the unit sphere in $\mathbb{R}^d$. Therefore, if we choose $z=e_1$, we have
$$ W^{(k)} \frac{x^{(k)}}{\|x^{(k)}\|} \sim W^{(k)} e_1 ,$$
where $e_1=[1,0,\dots,0]^T$ is the first basis vector. Therefore from \eqref{eq-ratio-norm}, we obtain
\begin{equation*}
    \frac{\|x^{(k+1)}\|}{\|x^{(k)}\|} \sim \left\| \phi_{{0}}\left(W^{(k+1)} e_1\right)\right\|,
\end{equation*}
which says that the distribution of the ratio $\frac{\|x^{(k+1)}\|}{\|x^{(k)}\|}$ is independent of $x^{(k)}$ and the history $x^{(j)}$ for $j<k$. Then, by the independence of the random variables $\frac{\|x^{(j+1)}\|}{\|x^{(j)}\|}$, we can write 

\begin{equation}
\begin{split}
    \mathbb{E} \left[\left(\frac{\|x^{(k)}\|}{\|x^{(0)}\|}\right)^s\right] =& \mathbb{E} \left[\Pi_{j=1}^{k} \frac{\| x^{(j)}\|^s}{\|x^{(j-1)}\|^s} \right]\\ 
    =&  \Pi_{j=1}^{k} \mathbb{E} \left[ \frac{\| x^{(j)}\|^s}{\|x^{(j-1)}\|^s} \right] \\
=& \Pi_{j=1}^{k} \mathbb{E} \left\| \phi_{{0}}\left(W^{(j)} e_1\right)\right\|^s  \\
=& \left(\sigma^s \mathbb{E} \left\| \phi_{{0}}\left(z\right)\right\|^s \right)^k,
\end{split}
\label{eqn-product}
\end{equation}  

where $z$ is a $d$-dimensional random vector with standard normal distribution $\mathcal{N}(0,I_{d})$.
The rest of the proof is about explicit computation of the term $\mathbb{E} \left\| \phi_0\left(z \right)\right\|^s$ which appear in the product \eqref{eqn-product} and showing that it is equal to $I_0(s,d)$ where $I_0(s,d)$ is defined by \eqref{def-sigma-relu}.  Note that 
$$ \mathbb{E} \left\| \phi_{{0}}\left(z \right)\right\|^s = \mathbb{E} \left[ \phi_{{0}}^2(z_1) + \phi_{{0}}^2(z_2) + \dots +
\phi_{{0}}^2(z_d) \right]^{s/2},
$$
where $z = (z_1, z_2, \dots, z_d)$ and $z_i$ are i.i.d. standard normal random variables. We first note that the function $\phi_0(x):\mathbb{R}^d \to \mathbb{R}$ has a piecewise linear structure on $\mathbb{R}^d$ depending on the sign of the components $x_i$ of a vector $x$. In particular, we observe that by the definition of the $\phi_0$ function,

\begin{equation}
    \begin{split}
        \|\phi_{{0}}(z)\|^2 =& \phi_{{0}}^2(z_1) + \phi_{{0}}^2(z_2) + \dots +
\phi_{{0}}^2(z_d) \\
=&\sum_{i: z_i>0}(z_i)^2 ,
    \end{split}
    \label{eq-target-dist}
\end{equation}

which depends on the orthant that the vector $z$ resides in $\mathbb{R}^{d}$. In particular, there are  $2^d$ (open) orthants in dimension $d$, where each orthant is defined by a system of inequalities:
$$\varepsilon_1 x_1 > 0, \quad       \varepsilon_2 x_2 > 0, \quad \varepsilon_3 x_3 > 0, \quad   \dots  \varepsilon_n x_n > 0,$$
where each $\varepsilon_i$ is $1$ or $-1$.  Therefore, we can identify each orthant from an element of the set $\{+,-\}^d$. For example, the non-negative (open) orthant corresponds to $\{+,+,\dots,+\}$ whereas the non-positive (open) orthant corresponds to  $\{-,-,\dots,-\}$. On every quadrant that corresponds to $n$ plus signs and $d-n$ minus signs (with an arbitrary order of the signs), the distribution of \eqref{eq-target-dist} is the same as the distribution of
 \beq Y_{n}: = \chi^2(n),
 \label{def-Yn}
 \eeq
where 
$\chi^2(n)$ denotes a chi-squared distribution with $n$ degrees of freedom as long as $n\geq1$. If we choose a random quadrant; with probability 
\beq p_d(n) = {d \choose n} \frac{1}{2^d},
\label{def-p-n-d}
\eeq
we will be in such a quadrant. \footnote{Note that 
$p_d(n) = \mathbb{P}(B_d = n)$ where $B_d \sim \mbox{Bi}(d,\frac{1}{2})$ is a random variable with a Binomial distribution, where the parameter $d$ represents the total number of Bernouilli trials and the parameter $\frac{1}{2}$ is the success probability for each trial.}  Therefore, we can interpret $\|\phi_0(z)\|^2$ as a mixture of chi-square distributions with weights from the Binomial distribution. It follows from \eqref{eq-target-dist}--\eqref{def-p-n-d} that we can write
$$ \mathbb{E} \left\| \phi_0\left(z \right)\right\|^s= \sum_{n=1}^d p_d(n) \mathbb{E}(Y_{n}^{s/2}).$$
The moments of $Y_n$ are explicitly known, and we have
$$\mathbb{E}(Y_n^s) = 2^s \frac{\Gamma(n/2 + s)}{\Gamma(n/2)} \quad \mbox{for} \quad s\geq 0, $$
(see \cite[Sec.	 8]{walck1996hand}) for any $s\geq 0$ where $\Gamma(\cdot)$ denotes the Gamma function. Therefore, we obtain
$$ \mathbb{E} \left\| \phi_{{0}}\left(z \right)\right\|^s=I_0(s,d) = \sum_{n=0}^d p_d(n) 2^{s/2} \frac{\Gamma(n/2 + s/2)}{\Gamma(n/2)}.$$
We conclude from \eqref{eqn-product} that \eqref{def-sigma-relu} holds. This also implies directly that part $(ii)$ and $(iii)$ are true. Finally, for any $p>s$, we have $\bar{\sigma}_0(s,d) > \bar{\sigma}_0(p,d)$ by Corollary \ref{coro-crit-leaky}. Therefore, if $\sigma = \bar{\sigma}_0(s,d)$, then  $\sigma> \bar{\sigma}_0(p,d)$ and by part $(iii)$, we obtain $\mathbb{E} \|x^{(k)}\|^{p} \to \infty$ exponentially fast in $k$. This completes the proof.
\end{proof}
\begin{remark}\label{coro-crit-relu-appendix} In the setting of Theorem \ref{thm-moment-relu}, in the special case when $s=2$, we obtain $I_0(2,d) = d/2$ and 
we obtain $ I_0(2,d) = 2\sum_{n=1}^d {d \choose n} \frac{1}{2^d}  \frac{\Gamma(n/2 + 1)}{\Gamma(n/2)} = \sum_{n=0}^d {d \choose n} \frac{1}{2^d}  n = \frac{d}{2} 
$ where we used the identity $\Gamma(x+1) = x\Gamma(x)$ for $x>0$ and the last equality can be obtained from the properties of the Binomial distributions, see e.g. \cite[Section 5.2]{walck1996hand}.
Therefore, from part $(i)$ of Theorem \ref{thm-moment-relu}, 
$\bar{\sigma}_0(2,d) = 1/ \sqrt{I_0(2,d)} =\sqrt{2}/\sqrt{d}$. In particular, the choice of $\bar{\sigma}_0(2,d)$ corresponds to Kaiming initialization. Theorem \ref{thm-moment-relu} is more general in the sense that it is applicable to any moment $s>0$.
\end{remark}
\section{Proof of Corollary \ref{coro-crit-sigma-asymp}}
\begin{proof} This result follows from analyzing the asymptotics of $I_0(s,d)$ for large $d$. It is known that for any real $\alpha>0$ and $z>0$, we can write the series expansion
\begin{equation*}\frac{\Gamma(z+\alpha)}{\Gamma(z)} = z^\alpha S(z,\alpha),
\end{equation*}
with
\begin{equation}
    \begin{split}
        S(z,\alpha) :=& \sum_{m=0}^\infty  A_m(\alpha)(\frac{1}{z})^m \\
        =&  1 + \frac{\alpha(\alpha-1)}{2z} + \mathcal{O}(\frac{1}{z^2}),
    \end{split}
     \label{eq-def-sz-alpha}
\end{equation}

where $A_m(\alpha)$ are coefficients of the expansion that admits an explicit representation (see \citep{tricomi1951asymptotic}). 
Therefore, choosing $z=n/2$ and $\alpha=s/2$,
\begin{equation} \frac{\Gamma(n/2+s/2)}{\Gamma(n/2)} = (\frac{n}{2})^{s/2} S(n/2,s/2), 
\label{eq-helper-gamma}
\end{equation}
so that
\begin{equation*}
    I_0(s,d) =\sum_{n=1}^d p_d(n) n^{s/2}S(n/2,s/2).
\end{equation*} 
Since the $\Gamma$ function is log-convex \citep{merkle1996logarithmic}, we also have
\begin{eqnarray*} \Gamma(\frac{n}{2} + \frac{s}{2}) &=& \Gamma\left((1-\frac{s}{2}) \frac{n}{2}  +\frac{s}{2}(\frac{n}{2} + 1)\right) \\
&\leq& \left(\Gamma(\frac{n}{2})\right)^{1-\frac{s}{2}} \left(\Gamma(\frac{n}{2}+1)\right)^{\frac{s}{2}} \\
&=& \Gamma(\frac{n}{2}) (\frac{n}{2})^{s/2},
\end{eqnarray*}
where we used the identity $\Gamma(z+1) = z\Gamma(z)$ for $z>0$. Therefore, we see from \eqref{eq-helper-gamma} that 
\beq 0 \leq S(n/2,s/2) \leq 1,
\label{eq-bdd-series}
\eeq
for every $s>0$ and $n>0$. 
Note that
\beq \frac{(\frac{d}{2})^{s/2}}{I_0(s,d)}
&=& \frac{1}{\mathbb{E}(F_{d}(B_d))},
\label{eq-inverse-target}
\eeq
where $B_d$ is a Binomial random variable, i.e.
  \beq \mathbb{P}(B_d = n) = 
  {d \choose n}\frac{1}{2^d} \quad \mbox{for} \quad n=  0,1,\dots,d,         
	\label{def-Binom-Bd}  
  \eeq
and
$$ F_{d}(X) :=  \begin{cases}
 2^{s/2} \frac{X^{s/2}}{d^{s/2}}S(X/2,s/2) & \mbox{if} \quad X>0, \\
 0 & \mbox{if} \quad X = 0,
 \end{cases} 
$$
satisfying for all $X>0$
\beq 
F_{d}(X) =  2^{s/2} \frac{X^{s/2}}{d^{s/2}} \left(1 + \frac{\frac{s}{2}(\frac{s}{2}-1)}{X} + \mathcal{O}(\frac{1}{X^2})\right)
\label{def-func-Fd}
\eeq 
where we used \eqref{eq-def-sz-alpha}. 
By the normal approximation of the binomial distribution, we also have
\begin{equation} Z_d := \frac{B_d - \mathbb{E}(B_d)}{\sqrt{\mbox{var}{B_d}}} = \frac{B_d - \frac{d}{2}}{\sqrt{d}/{2}}
\Rightarrow
\mathcal{N}(0,1)
\label{weak-conv}
\end{equation}
in distribution. We also have

\begin{eqnarray*} 
& &\mathbb{E}(F_{d}(B_d))\\
&=& \mathbb{E}\left(F_{d}(\frac{d}{2} + \frac{\sqrt{d}}{2}Z_d)\right) \\
&=&  2^{s/2} \mathbb{E} \left[\frac{(\frac{d}{2} + \frac{\sqrt{d}}{2}Z_d)^{s/2}}{d^{s/2}}S(\frac{d}{{4}} + \frac{\sqrt{d}}{{4}}Z_d,s/2)\right] \\
&=&  \mathbb{E} \left[(1 + \frac{1}{\sqrt{d}}Z_d)^{s/2}S(\frac{d}{{4}} + \frac{\sqrt{d}}{{4}}Z_d,s/2)\right].
\end{eqnarray*}
Using the Binomial expansion formula,
$$(1+x)^{s/2} = \sum_{k=0}^\infty {s/2 \choose k} x^k \quad \mbox{for} \quad |x| < 1,$$
for $Z_d < \sqrt{d}$, we can write 

\begin{small}
\begin{equation} 
\begin{split}
&(1 + \frac{1}{\sqrt{d}}Z_d)^{s/2}S(\frac{d}{4} + \frac{\sqrt{d}}{4}Z_d,s/2)\\ 
=& \left[\sum_{k=0}^\infty {s/2 \choose k} \frac{1}{(\sqrt{d})^k}Z_d^k\right]\! \left[\sum_{m=0}^M A_m(s/2)\left(\frac{2}{\frac{d}{2} + \frac{\sqrt{d}}{2}Z_d}\right)^m \right]\\
=& \left[\sum_{k=0}^\infty {s/2 \choose k} \frac{1}{(\sqrt{d})^k}Z_d^k\right]\\
& \left[ \sum_{m=0}^M A_m(s/2)\frac{4^m}{d^m}\left(
\sum_{\ell=0}^\infty \frac{{(-1)^l}}{\sqrt{d}^{\ell}} Z_d^\ell \right)^m\right] \\
=& \left(1 + {s/2 \choose 1}\frac{1}{\sqrt{d}}Z_d + {s/2 \choose 2}\frac{1}{d}Z_d^2 + \dots\right)\\
& \cdot\left(1 +{s/2 \choose 2}\frac{4}{d}\left(\sum_{\ell=0}^\infty \frac{{(-1)^l}}{\sqrt{d}^{\ell}} Z_d^\ell \right) + \dots\right)\\
=& 1+{s/2 \choose 2}\frac{4}{d}+\left[ {s/2 \choose 1}\frac{1}{\sqrt{d}}+{s/2 \choose 2}\frac{4}{d\sqrt{d}}\right] Z_d\\
& +\left[ {s/2 \choose 2}\frac{1}{d}+{s/2 \choose 2}\frac{4}{d^2}(\frac{s^2}{8}-\frac{3s}{4}+1)\right] Z_d^2 \\
&+\dots\,
\end{split}
\label{eq-prod-expansion}
\end{equation}
\end{small}
where we used the identity $A_1(s/2) = \frac{\frac{s}{2}(\frac{s}{2}-1)}{2}$.
Since $\mathbb{P}(Z_d\geq \sqrt{d}) = \mathcal{O}(e^{-d/2})$ and the function $S$ is non-negative and bounded by $1$ according to \eqref{eq-bdd-series}, we have
\begin{equation}
\begin{split}
&\mathbb{E}\left[(1 + \frac{1}{\sqrt{d}}Z_d)^{s/2}S_M(\frac{d}{2} + \frac{\sqrt{d}}{2}Z_d,s)\right]\\
=& \mathcal{O}(e^{-\frac{d}{2}}) + \mathbb{E}\left[
1+  {s/2 \choose 1}\frac{1}{\sqrt{d}}Z_d + {s/2 \choose 2}\frac{5}{d}Z_d^2 + \dots\right] \\
=& \mathcal{O}(e^{-d/2}) +1+{s/2 \choose 2}\frac{4}{d}+{s/2 \choose 2}\frac{1}{d}\\
&+{s/2 \choose 2}\frac{4}{d^2}(\frac{s^2}{8}-\frac{3s}{4}+1) + o(\frac{1}{d^2})\\
=& 1 + {s/2 \choose 2}\frac{5}{d} + o(\frac{1}{d}),
\end{split}
\label{eq-exp-H}
\end{equation}
where we used the fact that $\mathbb{E}(Z_d^k)\to \mathbb{E}(Z^k)$ as $d\to \infty$ for any fixed $k$ implied by \eqref{weak-conv} where $Z$ is a standard-normal variable in $\mathbb{R}$ which satisfies $\mathbb{E}(Z)=0$ and $\mathbb{E}(Z^2)=1$. Then, it follows from \eqref{eq-inverse-target} that 
\begin{equation}
    \begin{split}
        \frac{(\frac{d}{2})^{s/2}}{{I}_0(s,d)}
=& 1 - {s/2 \choose 2}\frac{5}{d} + o(\frac{1}{d}) \\
=& 1 + \frac{5s(2-s)}{8d} + o(\frac{1}{d}),
    \end{split}
    \label{eq-approximate-relu}
\end{equation}  

which implies 
\begin{equation*}  
\begin{split}
    \bar{\sigma}_{0}(s,d) =&\frac{1}{\sqrt[s]{{I}_0(s,d)}}\\
=& \frac{\sqrt{2}}{\sqrt{d}} \left(1 + \frac{5s(2-s)}{8d}  + o(\frac{1}{d})\right)^{1/s} \\
=& \frac{\sqrt{2}}{\sqrt{d}} \left( 1+\frac{5(2-s)}{8d} + 
o(\frac{1}{d})\right).
\end{split}
\end{equation*} 
Similarly, taking square of both sides,
\begin{equation*}
     \bar{\sigma}_{0}^2(s,d) =\frac{1}{\sqrt[s/2]{{I}_0(s,d)}}
= \frac{2}{d} \left(1 + \frac{5(2-s)}{4d}  + o(\frac{1}{d})\right)
\end{equation*} 
which completes the proof.

\end{proof}

\section{Probability of zero network output for ReLU activations}
When $X$ is a Gaussian random variable with distribution $\mathcal{N}(0,\sigma^2 I_d)$, we have 
$\mathbb{P}(\phi_0(X)= 0) = \mathbb{P}(\max(X,0) = 0)= \prod_{i=1}^n \mathbb{P}(X_i \leq 0) = \frac{1}{2^d}$ due to the symmetry of the $i$-th component $X_i$ with respect to the origin, independent of the choice of $\sigma>0$. The output of the $k$-th layer is actually not Gaussian, nevertheless exploiting its symmetry properties and piecewise linearity of the ReLU activations, the probability that the output $x^{(k)}$ will be zero can be computed with a similar calculation as follows and this probability is independent of the choice of $\sigma$. 
\begin{lemma}\label{lem-zero-proba} Under Gaussian initialization $\textbf{(A1)}$--$\textbf{(A2)}$ with ReLU activation, i.e. when $a=0$, for any $\sigma>0$ given, $ \mathbb{P}(x^{(k)}= 0) = 1- (1-\frac{1}{2^d})^k 
$.
\end{lemma}

\begin{proof}
Consider the first layer
\begin{equation*}
    x^{(1)}=[x_1^{(1)},x_2^{(1)}, \dots, x_d^{(1)} ]^T,
\end{equation*}
where $x_i^{(1)}=\phi(\sum_{j=1}^d W^{(1)}_{ij}x^{(0)}_j)$.
According to the assumption, $W^{(1)}_{ij}$ is normally distributed with a zero mean. Then, $\sum_{j=1}^d W^{(1)}_{ij}x^{(0)}_j$ is also normally distributed with zero mean $\mathbb{P}(\sum_{j=1}^d W^{(1)}_{ij}x^{(0)}_j \geq 0)=\frac{1}{2}$. Therefore, $\mathbb{P}(x_i^{(1)}\neq 0)=1-\frac{1}{2}=\frac{1}{2}$ and $\mathbb{P}(x^{(1)}\neq 0)=\frac{1}{2^d}$. If consider the $k$-th layer, we can get similarly
\begin{equation*}
    \mathbb{P}(x^{(k)}\neq 0 | x^{(k-1)}\neq 0)=1-\frac{1}{2^d}.
\end{equation*}
Since 
\begin{equation*}
\begin{split}
    \mathbb{P}(x^{(k)}\neq 0)=&\mathbb{P}(x^{(k)}\neq 0 | x^{(k-1)}\neq 0) \\
    &\mathbb{P}(x^{(k-1)}\neq 0 | x^{(k-2)}\neq 0) \dots \mathbb{P}(x_i^{(1)}\neq 0)\\
    =&(1-\frac{1}{2^d})^k,
    \end{split}
\end{equation*}
we can obtain the result
\begin{equation*}
    \mathbb{P}(x^{(k)}= 0)=1-\mathbb{P}(x^{(k)}\neq 0)=1-(1-\frac{1}{2^d})^k.
\end{equation*}
\end{proof}

\section{Proof of Theorem \ref{thm-relu-log-out}}
\begin{proof} By the same argument given in the proof of Theorem \ref{thm-moment-relu}, for any $k\geq 0$, if $x^{(k)}\neq 0$, we have 
\begin{equation}
    \frac{\|x^{(k+1)}\|}{\|x^{(k)}\|} \sim \left\| \phi_{{0}}\left(z\right)\right\|,
    \label{eq-dist-post-act}
\end{equation} 
where $z \sim \mathcal{N}(0,I_d)$ is a $d$-dimensional standard normal random vector, and in particular $\frac{\|x^{(k+1)}\|}{\|x^{(k)}\|}$ is independent from the choice of $x^{(k)}$ and the past history $x^{(j)}$ for $j<k$. Let $A_k$ be the event that $x^{(k)}\neq 0$. We note that 
\beq A_k = \cap_{j=0}^k A_j,
\label{eq-set-intersect}
\eeq
that is $x^{(k)}\neq 0$ if and only if $x^{(j)}\neq 0$
for $j{\leq} k$. This fact follows simply from the piecewise linear structure of the ReLU activation function. Conditioning on the event $A_k$, we can write
\begin{small}
\begin{equation} \frac{1}{k} \left( \log \frac{\|x^{(k)}\|}{\|x^{(0)}\|} | A_k \right) = \frac{1}{k} \sum_{j=0}^{k-1}  \left(\frac{1}{2} \log \frac{\|x^{(j+1)}\|^2}{\|x^{(j)}\|^2} | A_j \right),
\label{eq-average-iid}
\end{equation}
\end{small}
where\footnote{Here, the equality is to be understood in the sense of distributions, i.e. the left-hand side and the right-hand side have the same distribution.} the logarithm is well-defined as the ratio $\frac{\|x^{(j+1)}\|}{\|x^{(j)}\|}>0$ conditional on $A_j$. Due to \eqref{eq-dist-post-act} and \eqref{eq-set-intersect}, the right-hand side of \eqref{eq-average-iid} can be viewed as an average of i.i.d. random variables with mean
\begin{equation*}
\begin{split}    
 m_1 =& \frac{1}{2}\mathbb{E}\log \left( \frac{\|x^{(1)}\|^2}{\|x^{(0)}\|^2} \big| A_0 \right)\\
 =& \frac{1}{2}\mathbb{E}\log \left(\left\| \phi_{{0}}\left(\sigma z\right)\right\|^2 \big| z\not\in {\mathbb{R}^d_{-}} \right),
 \end{split}
\end{equation*}
where ${\mathbb{R}^d_{-}} = \{ x\in\mathbb{R}^d | x_i \leq 0 \mbox{ for } i=1,2,\dots,d\}$ denotes the (closed) non-positive orthant of vectors and variance
\begin{equation} 
\begin{split}
m_2 =&  \mbox{var} \left( \frac{1}{2}
\log \left(\left\| \phi_{{0}}\left(\sigma z\right)\right\|^2 \big| z\not\in {\mathbb{R}^d_{-}} \right)
\right) \\
=& \frac{1}{4} \mbox{var}  
\left(
\log \left(\left\| \phi_{{0}}\left( z\right)\right\|^2 \big| z\not\in {\mathbb{R}^d_{-}} \right)
\right).
\end{split}
\label{def-m-2}
\end{equation}
 
In the rest of the proof, we compute $m_1$ and $m_2$ explicitly showing them that they are finite; then by the central limit theorem and the law of large numbers, the theorem will hold with
\beq
\mu_0(\sigma) = m_1 \quad \mbox{and} \quad s_0^2 = m_2.  
\label{eq-mu-s}
\eeq 

We note that
\beq m_1 
&=& \log(\sigma) +\mathbb{E}\left(\log \left\| \phi_{{0}}\left(z\right)\right\| \big| z\not\in {\mathbb{R}^d_{-}}\right) \\
&=& \log(\sigma) +\frac{1}{2}\mathbb{E}\left( \log  \|\phi_{{0}}(z)\|^2 \big| z\not\in {\mathbb{R}^d_{-}} \right).
\label{eq-mean-log-moment}
\eeq
By \eqref{eq-target-dist} and following the same proof technique in Theorem \ref{thm-moment-relu}, we can show that given that $z\not\in {\mathbb{R}^d_{-}}$,
\beq \left( \|\phi_{{0}}(z)\|^2 \big| z\not\in {\mathbb{R}^d_{-}} \right) \sim Y_n 
\label{eq-mixture}
\eeq
with probability
\begin{equation*}
    \pi_d(n) = \frac{p_d(n)}{\sum_{n=1}^d p_d(n)} = {d \choose n}\frac{1}{2^d - 1}, 
\end{equation*}
for $n\geq 1$ where $Y_n$ is a chi-square distribution with $n$ degrees of freedom 
where $p_d(n)$ is given by \eqref{def-p-n-d}. We have also
\begin{equation*}
\begin{split}
    m_1 =& \log(\sigma) + \mathbb{E}\log\left(\|\phi_{{0}}(z)\|^2 \big| z\not\in {\mathbb{R}^d_{-}} \right)\\
    =& \log(\sigma) + \sum_{n=1}^d \pi_d(n) \left[  \mathbb{E} \log[Y_n]\right].
    \end{split}
\end{equation*}

Using the mixture representation \eqref{eq-mixture} and according to Lemma \ref{lemma-var-mix}, we have
\begin{equation*}
\begin{split}
&\mbox{var} (\log\|\phi_{{0}}(z)\|^2 \big| z\not\in {\mathbb{R}^d_{-}})\\
=&  \sum_{n=0}^d \pi_d(n) \mbox{var}(\log(Y_n)) + \sum_{n=0}^d \pi_d(n) (\mathbb{E}\log(Y_n))^2 \\
&- \left( \sum_{n=0}^d \pi_d(n) \mathbb{E}\log(Y_n)\right)^2.   
\end{split}
\end{equation*}
Logarithmic moments of chi-square distributions are explicitly available as 
$$ \mathbb{E}\log[Y_n] =
\log(2) + \Psi\left(\frac{n}{2}\right)
,$$
and 
$$\mbox{var}(\log(Y_n)) = \psi_1(n/2),$$
where $\psi_1(z)$ is the tri-gamma function (see \cite[Lemma 2.3]{cohen1984stability}). Therefore, from \eqref{eq-mean-log-moment}, we obtain
\begin{equation*}
    m_1 = \log(\sigma) + \frac{1}{2}\sum_{n=1}^d \pi_d(n)
\left[
\log(2) + \Psi\left(\frac{n}{2}\right)
\right]. 
\end{equation*} 
Then, from \eqref{def-m-2} we get, 
\begin{equation*}  
\begin{split}
m_2 =& \frac{1}{4}\mbox{var} (\log\|\phi_{{0}}(z)\|^2 \big| z\not\in {\mathbb{R}^d_{-}}) \\
= &\frac{1}{4}\left(\sum_{n=1}^d \pi_d(n) \psi_1(n/2)\right) \\
&+ \frac{1}{4}\left(\sum_{n=1}^d \pi_d(n) \left[
\log(2) + \Psi\left(\frac{n}{2}\right)
\right]^2 \right)\\
&- \frac{1}{4}\left( \sum_{n=1}^d \pi_d(n) \left[
\log(2) + \Psi\left(\frac{n}{2}\right)
\right]\right)^2.
\end{split}
\end{equation*}
We conclude from \eqref{eq-mu-s}. 
\end{proof}

\section{Proof of Theorem \ref{thm-leaky-relu}}
\begin{proof} 

The approach is similar to the proof of Theorem \ref{thm-moment-relu}.
We first note that for $x^{(k)}\neq 0$,
\begin{equation*}
\begin{split}
    \frac{\|x^{(k+1)}\|}{\|x^{(k)}\|} =&\frac{ \|\phi_a(W^{(k+1)}x^{(k)})\|}{\|x^{(k)}\|}\\
    =& \left\| \phi_a\left(W^{(k+1)} \frac{x^{(k)}}{\|x^{(k)}\|}\right)\right\| ,
    \end{split}
\end{equation*} 
which we used the equality $\phi_{{a}}(cy) = c\phi_{{a}}(y)$ for any given $c>0$ and arbitrary vector $y$ (with the choice of $y = W^{(k)}x^{(k)}$ and $c= 1/\|x^{(k)}\|$). By a similar reasoning to \eqref{eq-ratio-norm}--\eqref{eqn-product}, we obtain
\beq
\mathbb{E} \left[\left(\frac{\|x^{(k)}\|}{\|x^{(0)}\|}\right)^s\right] =
(\sigma^s \mathbb{E} \left\| \phi_{{a}}\left(z \right)\right\|^s)^k,
\label{eq-moment-s-prod}
\eeq
where $z$ is a $d$-dimensional random vector with standard normal distribution $\mathcal{N}(0,I_{d})$. In the rest of the proof, we compute the term $\mathbb{E} \left\| \phi_{{a}}\left(z\right)\right\|^s$ explicitly and establish that it is equal to $I_a(s,d)$ where $I_a(s,d)$ is given by \eqref{def-iasd-theorem}. 
Consider
\begin{equation*}
    \mathbb{E} \left\| \phi_{{a}}\left(z \right)\right\|^s = \mathbb{E} \left[ \phi_{{a}}^2(z_1) + \phi_{{a}}^2(z_2) + \dots +
\phi_{{a}}^2(z_d) \right]^{s/2},
\label{eq-Iads}
\end{equation*} 
where $z = (z_1, z_2, \dots, z_d)$ is a $d$-dimensional standard normal random vector. We first note that the function $\phi_{a}(x):\mathbb{R}^d \to \mathbb{R}$ has a piecewise linear structure on $\mathbb{R}^d$ depending on the sign of the components $x_i$ of a vector $x$. In particular, we observe that by the definition of the $\phi_{a}$ function,

\begin{equation} 
\begin{split}
\left\| \phi_{{a}}\left(z\right)\right\|^2 =& \phi_{{a}}^2(z_1) + \phi_{{a}}^2(z_2) + \dots +
\phi_{{a}}^2(z_d)\\
=& \sum_{i: z_i>0}(z_i)^2 + a^2 \sum_{i: z_i<0}(z_i)^2  ,
\end{split}
\label{eq-target-dist-a}
\end{equation}

which depends on the orthant that the vector $z$ resides in $\mathbb{R}^{d}$. In particular, there are  $2^d$ (open) orthants in dimension $d$, where each orthant is defined by a system of inequalities:
$$\varepsilon_1 x_1 > 0, \quad       \varepsilon_2 x_2 > 0, \quad \varepsilon_3 x_3 > 0, \quad   \dots  \varepsilon_n x_n > 0,$$
where each $\varepsilon_i$ is $1$ or $-1$.  Therefore, we can identify each orthant from an element of the set $\{+,-\}^d$. For example, the non-negative (open) orthant corresponds to $\{+,+,\dots,+\}$ whereas the non-positive (open) orthant corresponds to  $\{-,-,\dots,-\}$. On every quadrant that corresponds to $n$ plus signs and $d-n$ minus signs (with arbitrary order of the signs), the distribution of \eqref{eq-target-dist-a} is the same as the distribution of
 \beq  X_{n}: = Y_n + Z_n,
 \label{def-Xn}
 \eeq
where 
 $$ Y_n = \chi^2(n) \quad \mbox{and} \quad Z_n = a^2 \chi^2(d-n).$$
$\chi^2(v)$ denotes a chi-squared distribution with v degrees of freedom. In this representation, $Y_n$ and $Z_n$ are independent as they are related to i.i.d. entries of the $z$ vector. If we choose a random quadrant; with probability 
\beq\mathbb{P}(B_d = n) = p_d(n) = {d \choose n} \frac{1}{2^d}
\label{eq-binom-proba}
\eeq we will be in such a quadrant where $B_d$ is a Binomial random variable defined in \eqref{def-Binom-Bd}. Therefore, we can write
\begin{equation}
     \mathbb{E} \left\| \phi_{{a}}\left(z\right)\right\|^s = \sum_{n=0}^d p_d(n) \mathbb{E}(X_{n}^{s/2}).
\label{eq-rds}
\end{equation} 

 In the special case $s=2$, we have
\begin{eqnarray*}
\mathbb{E} \left\| \phi_{{a}}\left(z\right)\right\|^2 &=& \sum_{n=0}^d p_d(n) \mathbb{E}(X_n)\\
&=& \sum_{n=0}^d p_d(n) (\mathbb{E}(Y_n) +\mathbb{E}(Z_n)) \\
&=& \sum_{n=0}^d p_d(n) (n + a^2(d-n)) \\
&=& (1+a^2) \frac{d}{2}\\
&=& I_a(2,d),
\end{eqnarray*}
where we used $\mathbb{E}(B_d) = \sum_{n=0}^d p_d(n) n = \frac{d}{2}$ and \eqref{eq-moment-s-prod} implies directly that \eqref{eq-leaky-relu-moments} holds for the $s=2$ case. Next, we consider the case $s<2$ where we compute $\mathbb{E} \left\| \phi_{{a}}\left(z\right)\right\|^s$ through moment generating function techniques. We will show that it is equal to $I_a(s,d)$ defined by \eqref{def-iasd-theorem}.

Let $M_{X}(t) = \mathbb{E}(e^{tX})$ denote the moment generating function (MGF) of a random variable $X$. If we consider arbitrary moments $\alpha>0$ (where $\alpha$ is not necessarily a positive integer) of a non-negative random variable $X$; we have 
\beq \mathbb{E} \left[X^\alpha\right] = D^\alpha M_X(0),
\label{eq-frac-deriv-mgf}
\eeq
where $D^\alpha$ denotes the fractional derivative of order $\alpha$ in the Riemann-Louiville sense \citep{cressie1986moment}.
\footnote{In the special case when $\alpha$ is a positive integer, the fractional derivative reduces to the ordinary derivative and we obtain
$
    \mathbb{E} \left[X^\alpha\right]  = D^\alpha M_x(0) = \frac{d^\alpha M_X(t)}{dt^\alpha}|_{t=0}.$}
It is well-known that
$$ M_{Y_n}(t) = \frac{1}{(1-2t)^{{n/2}}}, \quad M_{Z_n}(t)= \frac{1}{(1-2a^2 t)^{{(d-n)/2}}},$$
(see e.g. \citep{bulmer1979principles}). By independence of $Y_n$ and $Z_n$, we have also
\begin{equation} 
\begin{split}
M_{X_n}(t) =& M_{Y_n}(t) M_{Z_n}(t) \\
             =&\frac{1}{(1-2t)^{n/2}} \cdot \frac{1}{(1-2a^2 t)^{(d-n)/2}}.
\end{split}
\label{eq-deriv-mgf}
\end{equation}
Hence,
\begin{equation}
\begin{split}
    \frac{d}{dt} M_{X_n}(t) =& \frac{n}{(1-2t)^{\frac{n}{2}+1}} \cdot \frac{1}{(1-2a^2 t)^{\frac{d-n}{2}}}\\
    &+ \frac{1}{(1-2t)^\frac{n}{2}} \cdot \frac{a^2(d-n)}{(1-2a^2 t)^{\frac{d-n}{2}+1}}, 
    \end{split}
    \label{eq-mgf-derivative}
\end{equation}
and
 by \cite[eqn. (7)]{cressie1986moment}, for $\alpha\in (0,1)$, we have also 
\begin{equation} 
\begin{split}
D^\alpha M_{X_n}(0) =\frac{1}{\Gamma(1-\alpha)}\int_{-\infty}^0 (-z)^{-\alpha} \frac{d M_{X_n}(z)}{dz} dz \\ 
= \frac{1}{\Gamma(1-\alpha)}\int_{0}^\infty (z)^{-\alpha} \frac{d M_{X_n}(-z)}{dz} dz \label{eq-int-target}.
\end{split}
\end{equation}

Evaluating this integral requires computing integrals of the form
\begin{equation*}
    J_{m,\ell}(\alpha) = \int_0^\infty z^{-\alpha}\frac{1}{(1+2z)^{m/2}} \cdot \frac{1}{(1+2a^2z)^{\frac{\ell}{2}}} dz\,,
\end{equation*} 

for integer values of $m$ and $\ell$ satisfying $m+\ell = d+2$. If we substitute $u = 1-  \frac{1}{2z+1}$, then 
$dz = \frac{1}{2(1-u)^2}du$ which leads to
\begin{small}
\begin{equation}
\begin{split}
    &J_{m,l}(\alpha)=\\
    & \frac{1}{2^{-\alpha+1}} \int_{0}^1 u^{-\alpha} (1-u)^{\frac{m+\ell}{2}+\alpha-2} \big(1-(1-a^2)u\big)^{-\frac{\ell}{2}} du. \label{eq-jml}
\end{split}
\end{equation}
\end{small}
Using the binomial series
\beq  (1+x)^{-n} = \sum_{k=0}^\infty (-1)^k {n+k-1 \choose k} x^{k} 
\label{eq-binomial-series}
\eeq

for $|x|<1$, we obtain
\begin{equation}
\begin{split}
J_{m,l}(\alpha)
=& \frac{1}{2^{-\alpha+1}} \sum_{k=0}^1 {\frac{\ell}{2}+k-1 \choose k} (1-a^2)^k  \\
&\int_{0}^{{1}} u^{-\alpha+k} (1-u)^{\frac{m+\ell}{2}+\alpha-2}  du \\
=& \frac{1}{2^{-\alpha+1}} \sum_{k=0}^\infty  {\frac{\ell}{2}+k-1 \choose k} (1-a^2)^k \\
&\mathrm {B}(k+1-\alpha,\frac{m+\ell}{2}+\alpha-1),
\end{split}
\label{eq-jml-integral}
\end{equation}
where 
$$ {\displaystyle \mathrm {B} (x,y)=\int _{0}^{1}t^{x-1}(1-t)^{y-1}\,dt}
$$
is the Beta function. From \eqref{eq-frac-deriv-mgf}, \eqref{eq-mgf-derivative} and \eqref{eq-int-target}; we have 
\begin{equation}
\begin{split}
&\mathbb{E}(X_n^\alpha)\\
=& D^\alpha M_{X_n}(0)\\
=& \frac{1}{\Gamma(1-\alpha)} (n J_{n+2,d-n}(\alpha) + a^2 (d-n) J_{n, d-n+2}(\alpha)).
\end{split}
\label{eq-frac-moment-leaky-relu} 
\end{equation}
From \eqref{eq-rds}, choosing $\alpha = s/2$ for $s\in (0,2)$, we conclude that
\begin{small}
\begin{equation}
\begin{split}
&\mathbb{E} \left\| \phi_{{a}}\left(z\right)\right\|^s
= \sum_{n=1}^d p_d(n) \mathbb{E}(X_{n}^{s/2})\\
=&  \frac{1}{\Gamma(1-s/2)} \sum_{n=1}^d p_d(n) \\
&\left(n J_{n+2,d-n}({\frac{s}{2}}) + a^2 (d-n) J_{n, d-n+2}({\frac{s}{2}}) \right) \\
=& \frac{1}{2^{-s/2}} \frac{1}{\Gamma(1-s/2)} \sum_{n=1}^d p_d(n) 
\sum_{k=0}^\infty w_{k,n}    \mathrm {B}(k+1-\frac{s}{2},\frac{d}{2}+\frac{s}{2})\\
=& I_a(d,s),
\end{split} 
\label{eq-Iasd-as-expectation}
\end{equation}
\end{small}
where $I_a(s,d)$ is as in \eqref{def-iasd-theorem} and
\begin{equation}
\begin{split}
w_{k,n} = &\frac{1}{2}(1-a^2)^k \\&
\left[ 
{\frac{d-n}{2}+k-1 \choose k} n
+ a^2(d-n){\frac{d-n}{2}+k \choose k} \right].
\end{split}
\label{def-wkn}
\end{equation}
This proves \eqref{eq-leaky-relu-moments}. The proofs of remaining parts of the theorem follow with a similar reasoning to the proof of Theorem \ref{thm-moment-relu} and are omitted.
\end{proof}

\section{Proof of Corollary \ref{coro-crit-leaky}}
\begin{proof}
First, we consider the case of fixed $a$ and $s$ where we vary $d$. Note that, by definition $\bar{\sigma}_a(s,d)=\frac{1}{\sqrt[s]{I_a(s,d)}}$ where
\begin{equation*}
    I_a(s,d)=\mathbb{E}\left[\phi_a^2(z_1)+\phi^2_a(z_2)+\dots +\phi^2_a(z_d)\right]^{s/2},
\end{equation*}
and $z_i$ are i.i.d. standard normal random variables. Clearly, 
\begin{equation*}
\begin{split}
    I_a(s,d+1)=&\mathbb{E}\left[\phi^2_a(z_1)+\phi^2_a(z_2)+\dots +\phi^2_a(z_{d+1})\right]^{s/2} \\
    >& I_a(s,d).
\end{split}
\end{equation*}
where the strict inequality stems from the fact that $\phi^2_a(z_{d+1})>0$ for $z_{d+1}>0$. Since $\bar{\sigma}_a(s,d)=\frac{1}{\sqrt[s]{I_a(s,d)}}$, we can conclude that $\bar{\sigma}_a(s,d+1) < \bar{\sigma}_a(s,d)$.

Secondly, we consider the case of fixed $s$ and $d$ and vary $a$. According to \eqref{eq-target-dist-a}, for every $a\in [0,1]$ we have 
\begin{equation*}
\begin{split}
    I_a(d,s) =& \mathbb{E}\left[\sum_{i:z_i \geq 0}(z_i)^2+a^2 \sum_{i:z_i<0}(z_i)^2\right]^{s/2}\\
    &\leq \mathbb{E}\left[\sum_{i:z_i \geq 0}(z_i)^2+ \sum_{i:z_i<0}(z_i)^2\right]^{s/2} \\
    &=I_1(d,s)
    \end{split}
\end{equation*}
Differentiating the left hand side with respect to $a$, for $a>0$ we obtain
\begin{small}
\beq 
&&\frac{d}{da}I_a(d,s)\\
&=& \mathbb{E}\left[\frac{d}{da} \left(\sum_{i:z_i \geq 0}(z_i)^2+a^2 \sum_{i:z_i<0}(z_i)^2\right)\right]^{s/2} \label{eq-interchange-exp} \\
&=& \mathbb{E}\left[2a \sum_{i:z_i<0}(z_i)^2\right]^{s/2} > 0,  \label{eq-helper-exp}
\eeq
\end{small}
where the interchangeability of the differentiation and expectation in \eqref{eq-interchange-exp} follows from the fact that both $I_a(d,s)$ and the expectation in \eqref{eq-helper-exp} are finite. This proves that $I_a(d,s)$ is monotonically strictly increasing in $a$. Since $\bar{\sigma}_a(s,d)=\frac{1}{\sqrt[s]{I_a(s,d)}}$, this implies that $\bar{\sigma}_a(s,d)$ is (monotonically) strictly decreasing in $a$.

Finally, we consider fixed $a$ and $d$ and consider the monotonicity of $\bar{\sigma}_a(s,d)$ with respect to $s$ for $s>0$. By the definition of $\bar{\sigma}_a(s,d)$, $\sigma=\bar{\sigma}_a(s,d)$ solves the implicit equation
\beq F(s, \sigma) = \sigma^s I_a(s,d) = 1 
\label{eq-nonlinear-eqn}
\eeq
  
where $\bar{\sigma}_a(s,d)>0$ for $s>0$. 
Differentiating both sides with respect to $s$, by the chain rule,
$$  \frac{dF}{d\sigma}(s,\bar{\sigma}_a(s,d))
\frac{d\bar{\sigma}_a(s,d)}{ds}+
 \frac{dF}{ds}(s,\bar{\sigma}_a(s,d)) = 0\,, 
$$
where the derivatives exist as the function $F$ is continuously differentiable in $s$ and $\sigma$. This is equivalent to 
\begin{equation} \frac{s}{\bar{\sigma}_a(s,d)} \frac{d\bar{\sigma}_a(s,d)}{ds} + 
\frac{dF}{ds}(s,\bar{\sigma}_a(s,d))  = 0.
\label{eq-implicit-fun-thm}
\end{equation}
Note that we have also
\beq F(s,\bar{\sigma}_a(s,d)) = 1.
\label{eq-F-equals-one}
\eeq
For $\sigma=\bar{\sigma}_a(s,d))$, consider the function 
$$\kappa(\tilde{s})  := F(\tilde{s}, \sigma) = \sigma^{\tilde{s}} I_a(\tilde{s},d) = \mathbb{E}
\|\phi_a(We_1)\|^{\tilde{s}}. $$
Clearly, $\kappa(\tilde{s})$ is continuously differentiable with respect to $\tilde{s}$. It is also known that $\kappa(\tilde{s})$ is a log-convex function of $\tilde{s}$ for $\tilde{s}>0$ (see e.g. \citep{buraczewski2014multidimensional}), a fact which follows from the non-negativity of the second derivative of $\log\kappa(\tilde{s})$.
 Therefore, $\kappa(\tilde{s})$ is convex in $\tilde{s}$. If we consider the tangent line to the function $\kappa(\tilde{s})$ at $\tilde{s}=0$ and $\tilde{s}=s$, by convexity of the function $\kappa$,  we have
\beq 
\kappa(\tilde{s}) &\geq& \kappa(s) + \kappa'(s)(\tilde{s}-s),  \label{eq-kappa-ineq-1}\\
\kappa(\tilde{s})&\geq& \kappa(0) + \kappa'(0)\tilde{s}, \label{eq-kappa-ineq-2}
\eeq
for any $\tilde{s}\geq 0$ where $\kappa'(\tilde{s}):=\frac{d\kappa}{d\tilde{s}}(\tilde{s})$. Noticing that $\kappa(0)=\kappa(s)=1$ and plugging in $\tilde{s}=0$ in \eqref{eq-kappa-ineq-1} and plugging in $\tilde{s}=s$ in \eqref{eq-kappa-ineq-2}, we obtain
\beq 
1&\geq& 1 -s \kappa'(s) ,  \\
1&\geq& 1 + \kappa'(0)s.
\eeq
Since $s>0$, we conclude that we have necessarily $\kappa'(0)\leq 0$ and $\kappa'(s)\geq 0$. Assume $\kappa'(s)=0$. Then \eqref{eq-kappa-ineq-1} would imply $\kappa(\tilde{s})\geq \kappa(s) =1$ for $\tilde{s}\geq 0$ and we would obtain $\kappa'(0)=0$ and $\kappa(\tilde{s})=1$ for $\tilde{s} \in [0,s]$ which would be a contradiction.
Therefore, we have necessarily
$$ \kappa'(s) = \frac{dF}{ds}(s,\bar{\sigma}_a(s,d)) >0. $$ 
%
Then, this implies 
that
\begin{equation}
 \frac{d\bar{\sigma}_a(s,d)}{ds} = - \left( \frac{\bar{\sigma}_a(s,d)}{s}\right)
\frac{dF}{ds}(s,\bar{\sigma}_a(s,d))<0\,,
\end{equation}
for $s>0$ and therefore $\bar{\sigma}_a(s,d)$ is a monotonically (strictly) decreasing function of $s$.
\end{proof}

\section{Proof of Corollary \ref{coro-sigma-lin}}
\begin{proof}
For a linear activation function, we have $a=1$. In this case, $w_{k,n} = d/2$ for $k=0$ and $w_{k,n}=0$ for $k>0$. Then, it follows that 
$$I_1(s,d) 
= 2^{s/2} \sum_{n=1}^d p_d(n)
\frac{\Gamma(\frac{d}{2}+\frac{s}{2})}{\Gamma(\frac{d}{2})} =2^{s/2}\frac{\Gamma(\frac{d}{2}+\frac{s}{2})}{\Gamma(\frac{d}{2})},
$$
where we used $\mathrm {B}(x,y) = \Gamma(x)\Gamma(y)/\Gamma(x+y)$ and the fact that $\Gamma(\frac{d}{2}+1) = \frac{d}{2}\Gamma(\frac{d}{2})$. This yields
$\bar{\sigma}_1(s,d) = \frac{1}{\sqrt{2}}\left(\frac{\Gamma(\frac{d}{2})}{\Gamma(\frac{d}{2}+\frac{s}{2})}\right)^{1/s}.
$
In the special case with $s=2$, using the identity $\Gamma(\frac{d}{2}+1) =\frac{d}{2}\Gamma(\frac{d}{2})$ again, we obtain
$ \bar{\sigma}_1(2,d)=\frac{1}{\sqrt{d}} 
$
which recovers the results of \cite{lecun1998efficient} for linear activations and is the basis for Lecun initialization.

The rest of the proof follows a similar approach to the proof of Corollary \ref{coro-crit-sigma-asymp}. From \eqref{eq-helper-gamma} and \eqref{eq-def-sz-alpha}, we obtain
$$ I_{{1}}(s,d) = 2^{s/2} (\frac{d}{2})^{s/2} 
\left( 1 + \frac{\frac{s}{2}(\frac{s}{2}-1)}{d}+ \mathcal{O}(\frac{1}{d^2})\right).
$$
This implies that 
\begin{eqnarray*}
\bar{\sigma}_{1}(s,d) &=& \frac{1}{\sqrt[s]{I_{{1}}(s,d)}} \\
&=& \frac{1}{\sqrt{d}} 
\left[ 1 + \frac{\frac{s}{2}(\frac{s}{2}-1)}{d}+ \mathcal{O}(\frac{1}{d^2})
\right]^{-1/s} \\
&=& \frac{1}{\sqrt{d}} - \frac{(\frac{s}{2}-1)}{2d\sqrt{d}} 
+ \mathcal{O}(\frac{1}{d\sqrt{d}}),
\end{eqnarray*}
where we used $(1+x)^s = 1 + sx + \mathcal{O}(x^2)$. Taking square of both sides, we obtain

\begin{equation*}
    \bar{\sigma}_{1}^2(s,d) = \frac{1}{\sqrt[s/2]{I_1(s,d)}} = \frac{1}{d} + \frac{2-s}{2d^2} + \mathcal{O}(\frac{1}{d^2\sqrt{d}}).
\end{equation*}

Next, we approximate $\bar{\sigma}_{a}^2(s,d)$ for $a>0$ small.  
Following the notation in the proof of Theorem \ref{thm-leaky-relu}, from \eqref{eq-frac-moment-leaky-relu} we have,
\begin{small}
\begin{equation}
\mathbb{E}(X_n^\alpha) = \frac{1}{\Gamma(1-\alpha)} (n J_{n+2,d-n}(\alpha) + a^2 (d-n) J_{n, d-n+2}(\alpha)) .
\label{eq-moment-a-asympt}
\end{equation}
\end{small}

For $m+\ell = d+2$, from \eqref{eq-jml}, we have 
\begin{equation}  
\begin{split}
&J_{m,\ell}(\alpha)\\
=& \frac{1}{2^{-\alpha+1}} \int_{0}^1 u^{-\alpha} (1-u)^{m/2+\alpha-2} \big(1 + \frac{a^2u}{1-u}\big)^{-\ell/2} du\\
=& \frac{1}{2^{-\alpha+1}} \int_{0}^1 u^{-\alpha} (1-u)^{m/2+\alpha-2}\\
&\big(1 - \frac{\ell}{2}\frac{a^2u}{1-u}+\frac{\frac{\ell}{2}(\frac{\ell}{2}+1)}{2}\frac{a^4u^2}{(1-u)^2}+\mathcal{O}(a^6)) du\\ 
=& J_{m,\ell}|_{a=0} 
- \frac{a^2}{2^{-\alpha+1}}\frac{\ell}{2} B(2-\alpha, \frac{m}{2}+\alpha-2)\\
&+ \frac{\ell(\ell+2)}{2^{-\alpha+4}}a^4 B(3-\alpha,\frac{m}{2}+\alpha-3) +\mathcal{O}(a^6), 
\end{split}
\label{eq-jml-alpha}
\end{equation}
where we used the Binomial formula and \eqref{eq-binomial-series}. Plugging $a=0$ in \eqref{eq-jml-alpha},
\begin{equation*}
\begin{split}
    J_{m,\ell}|_{a=0}  =& \frac{1}{2^{-\alpha+1}} \int_{0}^1 u^{-\alpha} (1-u)^{m/2+\alpha-2} du\\
    =&\frac{1}{2^{-\alpha+1}} B(1-\alpha,m/2+\alpha-1).
\end{split}
\end{equation*}
Therefore, from \eqref{eq-moment-a-asympt}, 
\begin{small}
\begin{equation}
\begin{split}
    &\mathbb{E}(X_n^\alpha)\\
    =& \mathbb{E}(X_n^\alpha)|_{a=0} +  \frac{1}{\Gamma(1-\alpha)} \mathbb{E}\big[ \\
    &-\frac{a^2(d-n)n}{2^{-\alpha+2}} B(2-\alpha, \frac{n}{2}+\alpha-1) \\
&+\frac{n(d-n)(d-n+2)a^4}{2^{-\alpha+4}} B(3-\alpha, \frac{n}{2}+\alpha-2)\\
&+\frac{a^2(d-n)}{2^{-\alpha+1}} B(1-\alpha, \frac{n}{2}+\alpha-1)\\
&- \frac{a^4(d-n)(d-n+2)}{2^{-\alpha+2}} B(2-\alpha, \frac{n}{2}+\alpha-2) + \mathcal{O}(a^6) \big] \\
=& \mathbb{E}(X_n^\alpha)|_{a=0}  + \mathbb{E}\left[\frac{\alpha a^2(d-n)}{2^{-\alpha+1}\Gamma(1-\alpha)} B(1-\alpha, \frac{n}{2}+\alpha-1)\right]\\
&- \mathbb{E}\left[\frac{(d-n)(d-n+2)a^4 \alpha}{2^{-\alpha+3}\Gamma(1-\alpha)} B(2-\alpha, \frac{n}{2}+\alpha-2)\right] \\
&+ \mathcal{O}(a^6), 
\end{split}
\label{eq-helper-a4-exp}
\end{equation}
\end{small}
where we used the identities $\mathrm{B}(x,y) = \Gamma(x)\Gamma(y)/\Gamma(x+y)$ and $\Gamma(x+1) = x\Gamma(x)$ for $x,y>0$. We denote
\begin{small}
\begin{equation*}
\begin{split}
T_1(n,\alpha)&:=\frac{\alpha a^2(d-n)}{2^{-\alpha+1}\Gamma(1-\alpha)} B(1-\alpha, \frac{n}{2}+\alpha-1),\\
T_2(n,\alpha)&:=\frac{(d-n)(d-n+2)a^4 \alpha}{2^{-\alpha+3}\Gamma(1-\alpha)} B(2-\alpha, \frac{n}{2}+\alpha-2).
\end{split}
\end{equation*}
\end{small}
We notice from \eqref{eq-Iasd-as-expectation} that
\beq 
I_{a}(s,d) =\mathbb{E}\left((X_{B_d})^{s/2}\right) = 
\mathbb{E}\left((X_{B_d})^{s/2}\right),
\label{eq-iads-binomial-2}
\eeq
where $B_d$ follows a binomial distribution with $\mathbb{P}(B_d=n) = p_d(n)$. From 
\eqref{eq-helper-a4-exp}, it follows that
\begin{equation*}
\begin{split}
I_a(s,d) =&  I_0(s,d)+\mathbb{E}[T_1(B_d,s/2)]-\mathbb{E}[T_2(B_d,s/2)]\\
&+\mathcal{O}(a^6 d^{s/2})
\end{split}
\end{equation*}
Recall that from \eqref{weak-conv} we have 
\begin{equation*} 
    Z_d = \frac{B_d - \mathbb{E}(B_d)}{\sqrt{\mbox{var}{B_d}}} = \frac{B_d - \frac{d}{2}}{\sqrt{d}/{2}}\xrightarrow[~~~~]{} \mathcal{N}(0,I).
\end{equation*} 
Similar to \eqref{eq-prod-expansion}, we consider 
\begin{equation*}
    H(\frac{s}{2}):=(1 + \frac{1}{\sqrt{d}}Z_d)^{s/2}S(\frac{d}{4} + \frac{\sqrt{d}}{4}Z_d,s/2),
\end{equation*}
which admits the expansion
\begin{equation*}
\begin{split}
    \mathbb{E}[H(\frac{s}{2})]=&\mathcal{O}(e^{-d/2}) +1+{s/2 \choose 2}\frac{4}{d}+{s/2 \choose 2}\frac{1}{d}\\
    &+{s/2 \choose 2}\frac{4}{d^2}(\frac{s^2}{8}-\frac{3s}{4}+1) + {o(\frac{1}{d^2})}.
    \end{split}
\end{equation*}
If we let $\alpha=\frac{s}{2}$, we also have
\begin{eqnarray*}
T_1(B_d,\frac{s}{2}) &=& \frac{a^2s(\frac{d}{2}-\frac{\sqrt{d}}{2}Z_d)}{2^{-\frac{s}{2}+2}}\frac{\Gamma(\frac{d}{4}+\frac{\sqrt{d}}{4}Z_d+\frac{s}{2}-1)}{\Gamma(\frac{d}{4}+\frac{\sqrt{d}}{4}Z_d)}\\
&=& \frac{a^2s(\frac{d}{2}-\frac{\sqrt{d}}{2}Z_d)}{2^{-\frac{s}{2}+2}} \left (\frac{d}{4}\right)^{\frac{s}{2}-1} H(\frac{s}{2}-1)\\
&=& \left (\frac{d}{2}\right)^{\frac{s}{2}} \frac{a^2 s(1-\frac{1}{\sqrt{d}}Z_d)}{2} H(\frac{s}{2}-1).
\end{eqnarray*}
According to \eqref{eq-exp-H}, we have
\begin{equation*}
\begin{split}
\mathbb{E}[T_1(B_d,\frac{s}{2})]=&\left (\frac{d}{2}\right)^{\frac{s}{2}} \frac{a^2 s}{2}\\
&\left[1+(\frac{5}{8}s-3)(s-2)\frac{1}{d}+o(\frac{1}{d}) \right].
\end{split}
\end{equation*}
Similarly, we can write
\begin{equation*}
\begin{split}
T_2(B_d,\frac{s}{2})=& \left (\frac{d}{2}\right)^{\frac{s}{2}}\frac{a^4(2-s)s}{2}H(\frac{s}{2}-2)\\
&\left( \frac{1}{4}-\frac{1}{2\sqrt{d}}Z_d-\frac{1}{d\sqrt{d}}Z_d+\frac{1}{d}+\frac{1}{4d^2}Z_d^2\right) ,
\end{split}
\end{equation*}
and we have
\begin{equation*}
\begin{split}
    \mathbb{E}[T_2(B_d,\frac{s}{2})]=&\left (\frac{d}{2}\right)^{\frac{s}{2}}\frac{a^4(2-s)s}{2}\\
    &\left[\frac{1}{4}+(\frac{5}{32}s^2-\frac{33}{16}s+7)\frac{1}{d} + o(\frac{1}{d})\right].
    \end{split}
\end{equation*}
Therefore, we can calculate
\begin{equation*}
\begin{split}
&I_a(s,d)\\ 
=& I_0(s,d)+\mathbb{E}[T_1(B_d,\frac{s}{2})]-\mathbb{E}[T_2(B_d,\frac{s}{2})] +O(a^6 d^{s/2})\\
=& \left (\frac{d}{2}\right)^{\frac{s}{2}} K \Big{[} 1+\frac{1}{K}(s-2)\Big{[}\frac{5s}{8}+\frac{a^2s}{2}(\frac{5}{8}s-3)\\
&+\frac{a^4s}{2}(\frac{5}{32}s^2-\frac{33}{16}s+7)\Big]\frac{1}{d} + O(a^6) + o(\frac{1}{d})\Big],
\end{split}
\end{equation*}
where $K$ is defined as
\begin{equation*}
    K:=1+\frac{a^2s}{2}+\frac{a^4s(s-2)}{8}.
\end{equation*}
Then we obtain
\begin{equation*}
    \begin{split}
        &\bar{\sigma}_a^2(s,d) = I_a^{-2/s}(s,d)\\
    =&\frac{2}{d} K^{-2/s}\Big[ 1+\frac{1}{K}(s-2)\Big[\frac{5s}{8}+\frac{a^2s}{2}(\frac{5}{8}s-3)\\
    &+\frac{a^4s}{2}(\frac{5}{32}s^2-\frac{33}{16}s+7)\Big]\frac{1}{d} +\mathcal{O}(a^6) + o(\frac{1}{d})\Big]^{-2/s}\\
    =&\frac{2}{d}  K^{-2/s} \Big[ 1+\frac{1}{K}(2-s)\Big[\frac{5}{4}+a^2(\frac{5}{8}s-3)\\
    &+a^4(\frac{5}{32}s^2-\frac{33}{16}s+7)\Big]\frac{1}{d} +\mathcal{O}(a^6) + o(\frac{1}{d})\Big]
    \end{split}
\end{equation*}
    
If $a$ is small, we can write
\begin{equation*} 
\begin{split}
&\bar{\sigma}_a^2(s,d)\\
=&\frac{2}{d}
\left( 1+\frac{a^2s}{2} +\mathcal{O}(a^4)\right) ^{-2/s}\\
&\left[1+\frac{2}{2+a^2s}(2-s)(\frac{5}{4}+a^2(\frac{5}{8}s-3)) \frac{1}{d} +\frac{\mathcal{O}(a^4)}{d}\right]\\
=& \frac{2}{1+a^2}\frac{1}{d}+(2-s)\frac{(5s-24)a^2+10}{2(s+2)a^2+4}\frac{1}{d^2}  \\
&+ \mathcal{O}(\frac{a^4}{d}) + o(\frac{1}{d^2})
\end{split}
\end{equation*}
which is equivalent to the claimed result for $\bar{\sigma}_a^2(s,d)$. This completes the proof.
\end{proof}
\section{Proof of Theorem \ref{thm-log-leaky-relu}}

\begin{proof} The proof follows by a similar reasoning to the proof of Theorem \ref{thm-relu-log-out}. The same proof technique applies where we can show that the theorem holds with constants
\begin{equation}
\begin{split}
\mu_a(\sigma) =& \frac{1}{2}  \mathbb{E}\log \|\phi_a(\sigma z)\|^2 \\
=&\log(\sigma) + \frac{1}{2}\mathbb{E} \left( \log \| \phi_a(z)\|^2 \right),
\end{split}
\label{eq-local-1}
\end{equation}
and 
\beq
s_a^2 =   \frac{1}{4} \mbox{var}  
\left(
\log \left(\left\| \phi_{{a}}\left( z\right)\right\|^2  \right)
\right).
\label{eq-local-2}
\eeq
We also recall from \eqref{eq-target-dist-a}--\eqref{eq-binom-proba} that
\beq \|\phi_a(z)\|^2 \sim X_n \quad \mbox{with probability} \quad p_d(n),
\label{eq-mixture-phia}
\eeq
for $n\geq 1$ where $X_n$ is defined by \eqref{def-Xn} and $p_d(n)$ is defined by \eqref{eq-binom-proba}. In the rest of the proof we compute $\mathbb{E}(\log(X_n))$ and $\mbox{var}(\log(X_n))$ for every $n\geq 1$ and then use the identities \eqref{eq-local-1}, \eqref{eq-local-2} and \eqref{eq-mixture-phia} to obtain an explicit formula for $\mu_a(\sigma)$ and $s_a^2$. 

Note that $X_n$ is non-negative, and we have
\begin{equation*}
    \mathbb{E}(\log(X_n)) = \frac{d}{d\alpha}\mathbb{E}(X_n^\alpha)|_{a=0} \,,
\end{equation*}
and 
\begin{equation*}
\begin{split}
    \mbox{var}(\log(X_n)) =& \frac{d^2}{d\alpha^2} \left(\log \mathbb{E}(X_n^\alpha)\right)|_{a=0}\\
    =& \frac{d}{d\alpha}\left(\frac{\frac{d}{d\alpha}\mathbb{E}(X_n^\alpha)}{\mathbb{E}(X_n^\alpha)}
    \right)|_{a=0}\,,
    \end{split}
\end{equation*}
provided that the expectations are finite (see e.g. \citep{cohen1984stability}). For computing these expectations, we calculate
\begin{equation}
\begin{split}
\frac{d}{d\alpha}\mathbb{E}(X_n^\alpha) =& \frac{d}{d\alpha}
\bigg( \frac{1}{\Gamma(1-\alpha)} (n J_{n+2,d-n}(\alpha) \\
&+ a^2 (d-n) J_{n, d-n+2}(\alpha))
\bigg).
\end{split}
\label{deriv-xnalpha}
\end{equation}
By the product rule for derivatives, for an integer $m>0$, 
\begin{small}
\begin{equation}
    \begin{split}
        &\frac{d}{d\alpha} J_{m,d+2-m}(\alpha) =
    \log(2) J_{m,d+2-m}(\alpha) \\
    & + 
    \frac{1}{2^{-\alpha+1}} \sum_{k=0}^\infty  {\frac{d-m}{2}+k \choose k} (1-a^2)^k \frac{d}{d\alpha}\mathrm {B}(k+1-\alpha,\frac{d}{2} + \alpha)
    \end{split}
    \label{eq-deriv-J}
\end{equation}
\end{small}

We also have
\begin{equation*}
\begin{split}
\frac{d}{d\alpha}\mathrm {B}(k+1-\alpha,\frac{d}{2} + \alpha) =& \frac{d}{d\alpha} \frac{\Gamma(k+1-\alpha)\Gamma(\frac{d}{2}+\alpha)}{\Gamma(\frac{d}{2}+k+1)} \\
=& b_k \mathrm{B}(k+1-\alpha, \frac{d}{2}+\alpha),
\end{split}
\end{equation*}
where
\begin{equation*}
    b_{k,\alpha} = \psi_0(\frac{d}{2}+\alpha) - \psi_0(k+1-\alpha),
\end{equation*}

and we used the fact $\mathrm{B}(x,y) = \Gamma(x)\Gamma(y)/\Gamma(x+y)$ for real scalars $x,y>0$. Inserting this formula into \eqref{eq-deriv-J}, 
\begin{equation*}
\begin{split}
&\frac{d}{d\alpha} J_{m,d+2-m}(\alpha)\\
=&
\log(2) J_{m,\ell}(\alpha)
+ \frac{1}{2^{-\alpha+1}} \sum_{k=0}^\infty  {\frac{d-m}{2}+k \choose k} (1-a^2)^k\\
&b_{k,\alpha}\mathrm {B}(k+1-\alpha,\frac{d}{2} + \alpha).
\end{split}
\end{equation*}
From \eqref{deriv-xnalpha}, we also get
\begin{small}
\begin{equation}
\begin{split}
&\frac{d}{d\alpha}\mathbb{E}(X_n^\alpha)\\
=& \frac{1}{\Gamma(1-\alpha)} \frac{d}{d\alpha}
  (n J_{n+2,d-n}(\alpha) + a^2 (d-n) J_{n, d-n+2}(\alpha))\\
& + \frac{\Gamma'(1-\alpha)}{\Gamma^2(1-\alpha)}(n J_{n+2,d-n}(\alpha) + a^2 (d-n) J_{n, d-n+2}(\alpha))\\
=& \frac{1}{\Gamma(1-\alpha)} \frac{d}{d\alpha}
  (n J_{n+2,d-n}(\alpha) + a^2 (d-n) J_{n, d-n+2}(\alpha))
\\
&+ \frac{\psi_0(1-\alpha)}{\Gamma(1-\alpha)}(n J_{n+2,d-n}(\alpha) + a^2 (d-n) J_{n, d-n+2}(\alpha))  \\
=& \frac{1}{\Gamma(1-\alpha)} \frac{1}{2^{-\alpha}}\sum_{k=0}^\infty w_{k,n} b_{k,\alpha} \mathrm {B}(k+1-\alpha,\frac{d}{2} + \alpha) \\
&+ \left[\log(2) + \psi_0(1-\alpha)\right] \mathbb{E}(X_n^\alpha),
\end{split}
\label{eq-deriv-frac-moment}
\end{equation} 
\end{small}
where we used \eqref{eq-jml-integral}, \eqref{eq-frac-moment-leaky-relu} and $w_{k,n}$ is defined by \eqref{def-wkn}. Therefore, 
\begin{equation*}
\begin{split}
     \frac{\frac{d}{d\alpha}\mathbb{E}(X_n^\alpha)}{\mathbb{E}(X_n^\alpha)} =& \frac{
\sum_{k=0}^\infty w_{k,n} b_{k,\alpha} \mathrm {B}(k+1-\alpha,\frac{d}{2}+\alpha)}{\sum_{k=0}^\infty w_{k,n} \mathrm {B}(k+1-\alpha,\frac{d}{2}+\alpha)}\\
&+\log(2) + \psi_0(1-\alpha),
\end{split}
\end{equation*}
where we used \eqref{eq-frac-moment-leaky-relu} again. Differentiating with respect to $\alpha$, we find 
\begin{equation}
    \begin{split}
    &\frac{d}{d\alpha}\left(\frac{\frac{d}{d\alpha}\mathbb{E}(X_n^\alpha)}{\mathbb{E}(X_n^\alpha)}\right)\\
=&     \frac{
\sum_{k=0}^\infty w_{k,n} (b^2_{k,\alpha} + \frac{d}{d\alpha}b_{k,\alpha}) \mathrm {B}(k+1-\alpha,\frac{d}{2}+\alpha)}{\sum_{k=0}^\infty w_{k,n} \mathrm {B}(k+1-\alpha,\frac{d}{2}+\alpha)}\\
&-\left(\frac{
\sum_{k=0}^\infty w_{k,n} b_{k,\alpha} \mathrm {B}(k+1-\alpha,\frac{d}{2}+\alpha)}{\sum_{k=0}^\infty w_{k,n} \mathrm {B}(k+1-\alpha,\frac{d}{2}+\alpha)}\right)^2 \\
&-\psi_1(1-\alpha),
    \end{split}
    \label{eq-second-der-frac-moment}
\end{equation}   
where 
$$ \frac{d}{d\alpha}b_{k,\alpha}= \psi_1(\frac{d}{2}+\alpha) + \psi_1(k+1-\alpha).
$$ 
Evaluating the expression \eqref{eq-deriv-frac-moment} 
at $\alpha=0$, we find 
\begin{equation}
\begin{split}
   m_n :=& \mathbb{E}(\log(X_n)) =  \frac{d}{d\alpha}\mathbb{E}(X_n^\alpha)|_{\alpha=0} \\
   =& \sum_{k=0}^\infty w_{k,n} b_{k,0} \mathrm {B}(k+1,\frac{d}{2})
+  \left[\log(2) -\gamma \right] \\
=& \sum_{k=0}^\infty w_{k,n} \left(\psi_0(\frac{d}{2}) - \psi_0(k+1)\right) \mathrm {B}(k+1,\frac{d}{2})\\
&+  \left[\log(2) -\gamma \right]
\end{split}
\label{eq-def-mn}.
\end{equation}

In the last two steps, we used the fact that $\psi_0(1)=\gamma$ where $\gamma$ is the Euler–Mascheroni constant. Similarly,
 \begin{equation}
 \begin{split}
 v_n :=& \mbox{var}(\log(X_n)) \\
 =&  \frac{d^2}{d\alpha^2}\mathbb{E}(X_n^\alpha)|_{\alpha=0} \\
 =& \frac{
\sum_{k=0}^\infty w_{k,n} (b^2_{k,0} + \frac{d}{d\alpha}b_{k,0}) \mathrm {B}(k+1,\frac{d}{2})}{\sum_{k=0}^\infty w_{k,n} \mathrm {B}(k+1,\frac{d}{2})}\\
&-\left(\frac{
\sum_{k=0}^\infty w_{k,n} b_{k,0} \mathrm {B}(k+1,\frac{d}{2})}{\sum_{k=0}^\infty w_{k,n} \mathrm {B}(k+1,\frac{d}{2})}\right)^2 -\psi_1(1).  
\end{split}
\label{eq-sn-sq}
 \end{equation}
On the other hand, by \eqref{eq-frac-moment-leaky-relu} and \eqref{eq-jml-integral}, we have 
\begin{equation}
\mathbb{E}(X_n^\alpha)  = \frac{1}{\Gamma(1-\alpha)}\frac{1}{2^{-\alpha}}\sum_{k=0}^\infty w_{k,n}  \mathrm {B}(k+1-\alpha,\frac{d}{2}+\alpha).
\label{eq-target-to-take-limit}
\end{equation}
We note that $X_n^\alpha \leq S_n := 1+X_n$ for $\alpha\in [0,1]$ where $\mathbb{E}(S_n)<\infty$. Therefore, by the dominated convergence theorem we have
\begin{equation*}
    \lim_{\alpha\to 0}\mathbb{E}(X_n^\alpha)  = \mathbb{E}(X_n^0)= 1.
\end{equation*}
Taking limits in \eqref{eq-target-to-take-limit} as $\alpha\to 0$, 
\begin{equation*}
    1 = \lim_{\alpha\to 0}\mathbb{E}(X_n^\alpha)  = \frac{1}{\Gamma(1)}\sum_{k=0}^\infty w_{k,n}  \mathrm {B}(k+1,\frac{d}{2}).
\end{equation*} 
Since $\Gamma(1)=1$, this is equivalent to 
\begin{equation*}
    \sum_{k=0}^\infty w_{k,n}  \mathrm {B}(k+1,\frac{d}{2}) =  1 \quad \mbox{for every} \quad n\geq 1.
\end{equation*}
Plugging this identity into \eqref{eq-sn-sq}, 
\begin{small}
\begin{equation*}
\begin{split}
    v_n =& 
\sum_{k=0}^\infty w_{k,n} (b^2_{k,0} + \frac{d}{d\alpha}b_{k,0}) \mathrm {B}(k+1,\frac{d}{2})\\
&-\left({
\sum_{k=0}^\infty w_{k,n} b_{k,0} \mathrm {B}(k+1,\frac{d}{2})}\right)^2 -\psi_1(1)\\
=& \psi_1(\frac{d}{2}) + \sum_{k=0}^\infty  \left[ \psi_1(k+1) - \psi_1(1)\right] w_{k,n} \mathrm {B}(k+1,\frac{d}{2}) \nonumber\\
&+ \sum_{k=0}^\infty  \left[ \psi_0(\frac{d}{2}) - \psi_0(k+1)\right]^2 {w_{k,n}} \mathrm {B}(k+1,\frac{d}{2}) \nonumber\\
& -\left[ \sum_{k=0}^\infty  \left( \psi_0(\frac{d}{2}) - \psi_0(k+1)\right) {w_{k,n}} \mathrm {B}(k+1,\frac{d}{2})
\right]^2.
\end{split}
\end{equation*}
\end{small}
We conclude that
\begin{equation}
    \mu_a(\sigma) = \log(\sigma) +\frac{1}{2}\mathbb{E} \log  X_n = \log(\sigma) + \frac{1}{2} \sum_{n=0}^d p_d(n) m_n
    \label{def-mu-a-sigma}
\end{equation}
and
\begin{equation}
\begin{split}
     s_a^2  =& \frac{1}{4}\Bigg[  \sum_{n=0}^d p_d(n) v_n + \sum_{n=0}^d p_d(n) (m_n)^2 \\
     &- \left( \sum_{n=0}^d p_d(n) m_n\right)^2 \Bigg]
\end{split}
\label{eq-def-sa}
\end{equation}  
where $p_d(n)$ is defined by \eqref{def-p-n-d}. This completes the proof.
\end{proof}
\begin{remark}\label{remark-stoc-dominance-2} \textbf{(First-order stochastic dominance property compared to Kaiming's method)} Figure \ref{fig:logoutputb} illustrates Theorem \ref{thm-log-leaky-relu}, showing the pdf and cdf of $R_{k,a}$ for linear activations ($a=1$) and Leaky ReLU activations with $a=0.01$ after $k=100$ layers with two choices of $\sigma$ according to Kaiming initialization and our initialization technique which preserves approximately the fractional moment of order $s=1$. We observe that the distribution of $R_{k,a}$ is similar to a Gaussian distribution, and with our initialization, the network output $R_{k,a}$ possesses a first-order stochastic dominance property in the sense of \cite{hadar1969rules} (Remark \ref{remark-stoc-dominance}). This dominance property will hold for large enough $k$, as our initialization can choose a larger $\sigma$ and hence results in a larger mean value $\mu_a(\sigma)$ in the setting of Theorem \ref{thm-leaky-relu} and as the results also admit non-asymptotic versions according to Remark \ref{remark-non-asymptotic}.
\end{remark}



\section{Extensions of results to dropout}\label{sec-dropout}
In this section, we consider extensions of our results reported in the main text to dropout which is a mechanism where some neurons are removed randomly to prevent overfitting (see Remark \ref{remark-dropout} in the main text for more details). 

\subsection{ReLU activation with dropout}
\begin{theorem}\label{thm-moment-relu-drop} \textbf{(Explicit characterization of the critical variance $\sigma_0^2(s,d)$ with dropout)} Consider a fully connected network with an input $x^{(0)}\in \mathbb{R}^d$ and Gaussian initialization satisfying \textbf{(A1)}-\textbf{(A2)} with ReLu activation function $\phi_0(x)=\max(x,0)$ with dropout where the probability to keep the neurons is given by $q\in (0,1]$. Let $s>0$ be a given real scalar. The $s$-th moment of the output of the $k$-th layer is given by
\begin{equation}
\begin{split}
    \mathbb{E} \left[ {\| x^{(k)}\|^s} \right]  = {\|x^{(0)}\|^s} (\sigma^s I_{0,q}(s,d))^k,\\
    I_{0,q}(s,d) = \frac{1}{q^s}  2^{s/2}\sum_{n=0}^d q_d(n)  \frac{\Gamma(n/2 + s/2)}{\Gamma(n/2)},
    \end{split}
    \label{def-sigma-relu-drop}
\end{equation}
where 
\beq q_d(n) = {d \choose n} (\frac{q}{2})^n (1-\frac{q}{2})^{d-n},
\label{def-q-n-d}
\eeq
and 
$\Gamma$ is the Gamma function. Then, it follows that we have three possible cases: 
\begin{itemize}
    \item [$(i)$] If $\sigma =\bar{\sigma}_{0,q}(s,d)$ where 
    $\bar{\sigma}_{0,q}(s,d):=\frac{1}{\sqrt[s]{I_{0,q}(s,d)}}
    $, then the network preserves the $s$-th moment of the layer outputs, i.e. for every $k \geq 1$, $\mathbb{E} \left[ {\| x^{(k)}\|^s} \right] = \|x^{(0)}\|^s,$ 
    whereas for any $p>s$, $\mathbb{E} \|x^{(k)}\|^{p} \to \infty$ exponentially fast in $k$.
    \item [$(ii)$] If $\sigma < \bar{\sigma}_{0,q}(s,d)$, then
        $\mathbb{E} \left[ {\| x^{(k)}\|^s} \right] \to 0$ exponentially fast in $k$. 
    \item [$(iii)$] If $\sigma >\bar{\sigma}_{0,q}(s,d)$, then $\mathbb{E} \left[ {\| x^{(k)}\|^s} \right] \to \infty$ exponentially fast in $k$.
\end{itemize}
\end{theorem}
\begin{proof} The proof follows from minor adaptations to the proof of Theorem \ref{thm-moment-relu}. In the proof of Theorem \ref{thm-moment-relu}, it suffices to replace $p_d(n)$ with 
$$q_d(n) := {d\choose n} (\frac{q}{2})^n (1-\frac{q}{2})^{d-n}$$
and $X_n$ with $X_n/q^2$ as the effect of dropout is to scale the network output and change the mixing probabilities of the chi-square distributions arising in the proof of Theorem \ref{thm-moment-relu}. This yields 
\begin{equation}
\begin{split}
    \mathbb{E} \left[ {\| x^{(k)}\|^s} \right]  = {\|x^{(0)}\|^s} (\sigma^s I_{0,q}(s,d))^k,\\
    I_{0,q}(s,d) = \frac{1}{q^s}  2^{s/2}\sum_{n=0}^d q_d(n) \frac{\Gamma(n/2 + s/2)}{\Gamma(n/2)}.
    \end{split}
    \label{def-sigma-relu-dropout}
\end{equation}
The proofs of remaining parts follow from a reasoning similar to the proof of Theorem \ref{thm-moment-relu} and are omitted.
\end{proof}
\begin{corollary}\label{coro-crit-sigma-asymp-drop} \textbf{(Critical variance $\bar{\sigma}_{0,q}(d,s)$ when $d$ is large with dropout)} For fixed width $d$ and $s\in (0,2]$, we have
$$ \bar{\sigma}_{0,q}^2(s,d) = \frac{2q}{d} + \frac{(2-s)(6-q)}{2d^2}  + o(\frac{1}{d^2}), $$
Therefore, it follows from  Theorem \ref{thm-moment-relu} that if $ \sigma^2  = \frac{2q}{d} + \frac{(2-s)(6-q)}{2d^2},
$
then the network will preserve the moment of order $s+o(\frac{1}{d})$ of the network output if dropout is used.
\end{corollary}
\begin{proof} The proof follows from minor modifications to the proof of {Corollary} \ref{coro-crit-sigma-asymp}. Following the proof technique of {Corollary} \ref{coro-crit-sigma-asymp}, we can write
\beq \frac{(\frac{d}{2})^{s/2}}{I_{0,q}(s,d)}
&=& \frac{q^s}{\mathbb{E}(F_{d}(B_{d,q}))},
\label{eq-inverse-target-drop}
\eeq
where $B_{d,q}$ is a Binomial random variable, i.e.
  \beq \mathbb{P}(B_{d,q} = n) = q_d(n) = 
  {d \choose n}(\frac{q}{2})^n (1-\frac{q}{2})^{d-n} 
  \label{def-qdn}
  \eeq
for $n=  0,1,\dots,d,$ where $F_d$ is defined by \eqref{def-func-Fd}. By the normal approximation to binomial distribution, we have
\begin{equation} Z_{d,q} := \frac{B_d - \mathbb{E}(B_d)}{\sqrt{\mbox{var}{B_d}}} = \frac{B_d - \frac{dq}{2}}{\frac{\sqrt{d}}{2}\sqrt{2q-q^2}}\xrightarrow[~~~~]{} \mathcal{N}(0,I)
\label{weak-conv-drol}
\end{equation} 
which is similar to \eqref{weak-conv}. Then, we follow similar computations to the proof of Corollary \ref{coro-crit-sigma-asymp}:
\begin{equation*} 
\begin{split}
&\mathbb{E}(F_{d}(B_{d,q}))\\
=& \mathbb{E}\left(F_{d}(\frac{dq}{2} + \frac{\sqrt{d}}{2}\sqrt{2q-q^2} Z_d)\right) \\
=&  2^{s/2} \mathbb{E} \Bigg[\frac{(\frac{dq}{2} + \frac{\sqrt{d}}{2}\sqrt{2q-q^2}Z_d)^{s/2}}{d^{s/2}}\\
&S(\frac{dq}{2} + \frac{\sqrt{d}}{2}\sqrt{2q-q^2} Z_d,s/2)\Bigg] \\
=& q^{s/2} \mathbb{E} \Bigg[(1 + \frac{1}{\sqrt{d}}\frac{\sqrt{2-q}}{\sqrt{q}}Z_d)^{s/2}\\
&S(\frac{dq}{2} + \frac{\sqrt{d}}{2}\sqrt{2q-q^2}Z_d,s/2)\Bigg].
\end{split}
\end{equation*}
Using the Binomial expansion,
$$(1+x)^{s/2} = \sum_{k=0}^\infty {s/2 \choose k} x^k \quad \mbox{for} \quad |x| < 1.$$
Therefore, for $Z_{d,q} < \sqrt{d}\left(\frac{\sqrt{q}}{\sqrt{2-q}}\right)$, we can write 
\begin{small}
\begin{equation*}
    \begin{split}
        &(1 + \frac{1}{\sqrt{d}}\frac{\sqrt{2-q}}{\sqrt{q}}Z_{d,q})^{s/2}S(\frac{dq}{2} + \frac{\sqrt{d}}{2}\sqrt{2q-q^2}Z_{d,q},s/2)\\
=&\left[\sum_{k=0}^\infty {s/2 \choose k} \frac{\sqrt{2-q}^k}{(\sqrt{dq})^k}Z_{d,q}^k\right] \Bigg[ \sum_{m=0}^M A_m(s/2)\\
&\left(\frac{2}{\frac{dq}{2} + \frac{\sqrt{d}}{2}\sqrt{2q-q^2}Z_{d,q}}\right)^m \Bigg]\\
=& \left[\sum_{k=0}^\infty {s/2 \choose k} \frac{\sqrt{2-q}^k}{(\sqrt{dq})^k}Z_{d,q}^k \right] \Bigg[ \sum_{m=0}^M A_m(s/2)\frac{2^m}{(dq)^m}\\
&\left(\frac{2}{1 + \frac{1}{\sqrt{d}}\frac{\sqrt{2-q}}{\sqrt{q}}Z_{d,q}}\right)^m\Bigg]\\
=& \left[\sum_{k=0}^\infty {s/2 \choose k} \frac{\sqrt{2-q}^k}{(\sqrt{dq})^k}Z_{d,q}^k \right] \Bigg[ \sum_{m=0}^M A_m(s/2)\frac{4^m}{d^m q^m}\\
&\left(
\sum_{\ell=0}^\infty \frac{1}{\sqrt{d}^{\ell}} \frac{\sqrt{2-q}^\ell}{\sqrt{q}^\ell} Z_{d,q}^\ell \right)^m\Bigg] \\
=&\left(1 + {s/2 \choose 1} \frac{\sqrt{2-q}}{\sqrt{dq}} Z_{d,q} + {s/2 \choose 2} \frac{{2-q}}{{dq}} Z_{d,q}^2 + \dots\right) \\
&\cdot\left(1 +
{s/2 \choose 2}\frac{4}{dq}\left(
\sum_{\ell=0}^\infty \frac{1}{\sqrt{d}^{\ell}} \frac{\sqrt{2-q}^\ell}{\sqrt{q}^\ell} Z_{d,q}^\ell  \right) + \dots\right)\\
=& 1+  {s/2 \choose 1}\frac{1}{\sqrt{d}}Z_{d,q} + {s/2 \choose 2}\frac{6-q}{dq}Z_d^2 + \dots,
    \end{split}
\end{equation*} 
\end{small}

where we used the identity $A_1(s/2) = {s/2 \choose 2} =\frac{\frac{s}{2}(\frac{s}{2}-1)}{2}$.
Since $\mathbb{P}(Z_{d,q} \geq \sqrt{d}) = \mathcal{O}(e^{-d/2})$ and the function $S$ is non-negative and bounded by $1$ according to \eqref{eq-bdd-series}, we have
\begin{small}
\begin{equation*} 
\begin{split}
&\mathbb{E}\left[(1 + \frac{1}{\sqrt{d}}\frac{\sqrt{2-q}}{\sqrt{q}}Z_{d,q})^{\frac{s}{2}}S(\frac{dq}{2} + \frac{\sqrt{d}}{2}\sqrt{2q-q^2}Z_{d,q},\frac{s}{2})\right] \\
=&  \mathbb{E}\Bigg[
1+  {s/2 \choose 1}\frac{1}{\sqrt{d}}Z_d+{s/2 \choose 2}\frac{(6-q)}{dq}Z_d^2 + \dots\Bigg]\\
&+\mathcal{O}(e^{-\frac{d}{2}(2q-q^2)})  \\
=& 1 + {s/2 \choose 2}\frac{6-q}{dq} + o(\frac{1}{d}),
\end{split}
\end{equation*}
\end{small}
where we used the fact that $\mathbb{E}(Z_{d,q}^k)\to \mathbb{E}(Z^k)$ as $d\to \infty$ for any fixed $k$ implied by \eqref{weak-conv} where $Z$ is a standard-normal variable in $\mathbb{R}$ with the property that $\mathbb{E}(Z)=0$ and $\mathbb{E}(Z^2)=1$. Then, it follows from \eqref{eq-inverse-target-drop} that 
\begin{equation}
    \begin{split}
         \frac{(\frac{d}{2})^{s/2}}{{I}_{0,q}(d,s)}
=& q^{s/2} \left[ 1 - {s/2 \choose 2}\frac{6-q}{dq} + o(\frac{1}{d}) \right]\\
=& q^{s/2} \left[1 - \frac{(6-q)s(s-2)}{8dq} + o(\frac{1}{d})\right],
    \end{split}
    \label{eq-approxi-relu}
\end{equation} 

which implies
\begin{equation}
    \begin{split}
        \sigma^2_{0,q}(d,s) =& (\frac{1}{{I}_{0,q}(d,s)})^{2/s}\\
        =& \frac{2q}{d} \left[1 - \frac{(6-q)s(s-2)}{8dq} + o(\frac{1}{d})\right]^{2/s} \\
=&  \frac{2q}{d} \left[1 - \frac{(6-q)(s-2)}{4dq} + o(\frac{1}{d})\right]\\
=& \frac{2q}{d} - \frac{(6-q)(s-2)}{2d} + o(1/d^2).
    \end{split}
\end{equation}  

This completes the proof.
\end{proof}

\subsection{Parametric ReLU activation with dropout}
\begin{theorem}\label{thm-leaky-relu-q}\textbf{(Explicit characterization of the critical variance $\sigma_{a,q}^2(s,d)$ with dropout)} Consider a fully connected network with an input $x^{(0)}\in \mathbb{R}^d$ and Gaussian initialization satisfying \textbf{(A1)}--\textbf{(A2)} with activation function $\phi_a(x)$ for any choice of $a\in (0,1]$ fixed and with dropout where the probability to keep a neuron is $q\in (0,1]$. Then, for any $s\in (0,2]$, the output of the $k$-th layer satisfies
\begin{eqnarray}\mathbb{E} \left[ {\| x^{(k)}\|^s} \right]  = {\|x^{(0)}\|^s} (\sigma^s I_{a,q}(s,d))^k
\end{eqnarray} 
with $I_{a,q}(s,d)$ defined as:
\begin{equation*}
\begin{split}
    I_{a,q}(s,d)  =& \frac{1}{q^s} 2^{s/2} \frac{1}{\Gamma(1-s/2)} \sum_{n={0}}^d {\sum_{m=0}^{d-n}} q_d(n,m) \\
 &\sum_{k=0}^\infty w_{k,n,m}    \mathrm {B}(k+1-\frac{s}{2},\frac{n+m}{2}+\frac{s}{2})
 \end{split}
\end{equation*}
for $s\in(0,2)$, especially if $s=2$
\begin{equation*}
    I_{a,q}(s,d)=\frac{1}{q^2}(1+a^2) \frac{d}{2}
\end{equation*}
where $q_d(n,m)$ is defined by \eqref{def-qdnm}, $\mathrm {B(\cdot,\cdot)}$ is the Beta function and
\begin{small}
\begin{equation} w_{k,n,m} = \frac{1}{2}(1-a^2)^k \left[ 
{\frac{m}{2}+k-1 \choose k} n + a^2m{\frac{m}{2}+k \choose k} 
\right].
\end{equation}
\end{small}
Let {$\bar{\sigma}_{a,q}(s,d)=\frac{1}{\sqrt[s]{I_{a,q}(s,d)}}$}. We have three possible cases: 
\begin{itemize}
    \item [$(i)$] If $\sigma =\bar{\sigma}_{a,q}(s,d)$ where 
    then the network preserves the $s$-th moment of the layer outputs, i.e. for every $k \geq 1$, $\mathbb{E} \left[ {\| x^{(k)}\|^s} \right] = \|x^{(0)}\|^s,$ 
    whereas for any $p>s$, $\mathbb{E} \|x^{(k)}\|^{p} \to \infty$ exponentially fast in $k$.
    \item [$(ii)$] If $\sigma < \bar{\sigma}_{a,q}(s,d)$, then
        $\mathbb{E} \left[ {\| x^{(k)}\|^s} \right] \to 0$ exponentially fast in $k$. 
    \item [$(iii)$] If $\sigma > \bar{\sigma}_{a,q}(s,d)$, then $\mathbb{E} \left[ {\| x^{(k)}\|^s} \right] \to \infty$ exponentially fast in $k$.
\end{itemize}
\end{theorem}

\begin{proof} The proof follows from minor changes to the proof of Theorem \ref{thm-leaky-relu}. In the abscence of dropout (i.e. when $q=1$), the quantity defined in the proof of Theorem \ref{thm-leaky-relu}, $I_a(d,s)$ has the distribution
$$ X_{n}: = \chi_1^2(n)+ a^2 \chi_2^2(d-n),
$$
with probability $p_d(n)$ where $\chi_1^2(n)$ and $ \chi_2^2(d-n)$ are independent chi-square distributions with degrees of freedom $n$ and $d-n$ respectively. When there is dropout, the distribution of $X_n$ and the corresponding binomial probabilities will be subject to change because now there is the possibility of zero output from some neurons due to dropout and scaling the neuron outputs. The corresponding probabilities will come from the trinomial distribution instead. More specifically, it suffices to replace $X_n$ with
$$ X_{n,m} = \frac{1}{q^2} \left(\chi_1^2(n) + a^2\chi_2^2(m) \right)$$
with probabilities from the trinomial distribution
\begin{equation} 
q_{d}(n,m) = \frac{d!}{n!m!(d-n-m)!}(\frac{q}{2})^{n+m} (1-q)^{d-n-m}.
\label{def-qdnm}
\end{equation}

Moreover, we can compute $\mathbb{E}X_{n,m}^{s/2}$ by simply replacing $d$ with $n+m$ in the formula for $\mathbb{E}((X_n)^{s/2})$ we obtained in \eqref{eq-frac-moment-leaky-relu}. After following similar steps to the proof of Theorem \ref{thm-leaky-relu}, we obtain the desired result.
\end{proof}

\begin{corollary}\textbf{(Critical variance $\bar{\sigma}_{1,q}(d,s)$ when $d$ is large with dropout)}\label{coro-sigma-lin-drop} For fixed width $d$ and $s\in (0,2]$, we have 
$$ \bar{\sigma}_{1,q}^2(s,d)  =  \frac{q}{d} + \frac{(3-q)(2-s)}{4d^2}  + o(\frac{1}{d^2})
$$
with $\bar{\sigma}_{1,q}^2(2,d)=\frac{q}{d}$. Therefore, it follows from Theorem \ref{thm-leaky-relu} that if 
$ \sigma^2  = \frac{1}{d} + \frac{(3-q)(2-s)}{4d^2},
$
then the network with linear activation will preserve the moment of order $s+o(\frac{1}{d})$ of the network output. 
\end{corollary}

\begin{proof} 
In the linear activation function case, we have $a=1$, then we obtain $w_{k,n,m}=\frac{m+n}{2}$ for $k=0$, $w_{k,n,m}=0$ for $k>0$. Then it follows that 
\begin{equation*}
\begin{split}
    &I_{1,q}(s,d)\\
    =&\frac{1}{q^s}2^{s/2}\frac{1}{\Gamma(1-\frac{s}{2})} \sum_{n=0}^d \sum_{m=0}^{d-n} q_d(n,m) \frac{m+n}{2}\\
    &B(1-\frac{s}{2},\frac{m+n}{2}+\frac{s}{2})\\
    =& \frac{1}{q^s}2^{s/2} \sum_{n=0}^d \sum_{m=0}^{d-n} q_d(n,m) \frac{\Gamma(\frac{m+n}{2}+\frac{s}{2})}{\Gamma(\frac{m+n}{2})}. 
    \end{split}
\end{equation*}
where we use the identity $\Gamma(\frac{m+n}{2}+1) = \frac{m+n}{2} \Gamma(\frac{m+n}{2})$  
and the fact that $\mathrm B(x,y) = \Gamma(x)\Gamma(y)/\Gamma(x+y)$ for $x,y>0$. 
Then denote $t=m+n$, we get
\begin{equation}
    I_{1,q}(s,d)=\frac{1}{q^s}2^{s/2} \sum_{t=0}^{d} h_d(t) \frac{\Gamma(\frac{t}{2}+\frac{s}{2})}{\Gamma(\frac{t}{2})},
    \label{def-I-linear-drop}
\end{equation}
where
\begin{equation*}
    h_d(t):={d \choose t}q^t(1-q)^{d-t}.
\end{equation*}
We see that in the special case $q=1/2$, this formula reduces to the analysis provided in Corollary \ref{coro-crit-sigma-asymp}. The proof will be similar where we will follow a similar approach to the proof of Corollary \ref{coro-crit-sigma-asymp}. Similar to the proof technique of Corollary \ref{coro-crit-sigma-asymp}, We write
\begin{equation*}
    I_{1,q}(s,d) =\frac{1}{q^s}2^{s/2}\sum_{t=1}^d h_d(t) (\frac{t}{2})^{s/2}S(t/2,s/2).
\end{equation*} 

Note that
\beq \frac{(\frac{d}{2})^{s/2}}{I_{1,q}(s,d)}
&=& \frac{q^s}{\mathbb{E}(F_{d}(H_d))},
\label{eq-inverse-target-linear}
\eeq
where $H_d$ is a Binomial random variable, i.e.
  $$ \mathbb{P}(H_d = n) = 
  {d \choose t}q^t(1-q)^{d-t} \quad \mbox{for} \quad t=  0,1,\dots,d,         
  $$

where $F_d$ is defined by \eqref{def-func-Fd}. By the normal approximation of the binomial distribution, we also have
\begin{equation} \xi_{d,q} := \frac{H_d - \mathbb{E}(H_d)}{\sqrt{\mbox{var}{H_d}}} = \frac{H_d - dq}{\sqrt{dq(1-q)}}\xrightarrow[~~~~]{} \mathcal{N}(0,1)
\end{equation}
in distribution. We also have
\begin{equation*}
\begin{split}
    &\mathbb{E}(F_{d}(H_d))\\
    =& \mathbb{E}\left(F_{d}(dq + \sqrt{dq(1-q)}\xi_{d,q})\right) \\
=&  2^{s/2} \mathbb{E} \Bigg[\frac{(dq + \sqrt{dq(1-q)}\xi_{d,q})^{s/2}}{d^{s/2}}\\
&S(\frac{1}{2}(dq + \sqrt{dq(1-q)}\xi_{d,q}),s/2)\Bigg] \\
=&  (2q)^{s/2}\mathbb{E} \Bigg[(1 + \sqrt{\frac{1-q}{dq}}\xi_{d,q})^{s/2}\\
&S(\frac{1}{2}(dq + \sqrt{dq(1-q)}\xi_{d,q}),s/2)\Bigg].
\end{split}
\end{equation*}
Recall the Binomial expansion formula,
$$(1+x)^{s/2} = \sum_{k=0}^\infty {s/2 \choose k} x^k \quad \mbox{for} \quad |x| < 1.$$
For $\xi_{d,q} < \sqrt{\frac{dq}{1-q}}$, we can write 
\begin{small}
\begin{equation*} 
\begin{split}
&\left[(1 + \sqrt{\frac{1-q}{dq}}\xi_{d,q})^{s/2}S(\frac{1}{2}(dq + \sqrt{dq(1-q)}\xi_{d,q}),s/2)\right] \\
=&\left[ \sum_{k=0}^\infty {s/2 \choose k} (\frac{1-q}{dq})^{k/2}\xi_{d,q}^k\right] \Bigg[\sum_{m=0}^M \\
&A_m(s/2)\left(\frac{2}{dq + \sqrt{dq(1-q)}\xi_{d,q}}\right)^m \Bigg]\\
=& \left[ \sum_{k=0}^\infty {s/2 \choose k} (\frac{1-q}{dq})^{k/2}\xi_{d,q}^k\right] \Bigg[\sum_{m=0}^M A_m(s/2) (\frac{2}{dq})^m\\
&\left(\frac{1}{1 + \sqrt{\frac{1-q}{dq}}\xi_{d,q}}\right)^m \Bigg]\\
=& \left[ \sum_{k=0}^\infty {s/2 \choose k} (\frac{1-q}{dq})^{k/2}\xi_{d,q}^k\right] \Bigg[\sum_{m=0}^M A_m(s/2) (\frac{2}{dq})^m\\
&\left(\sum_{\ell=0}^\infty {(-1)^l}(\sqrt{\frac{1-q}{dq}})^{\ell} \xi_{d,q}^\ell \right)^m \Bigg]\\
=& 1+{s/2 \choose 2}\frac{2}{dq}+\Bigg[ {s/2 \choose 1}\sqrt{\frac{1-q}{qd}}-{s/2 \choose 2} \frac{2}{dq}\sqrt{\frac{1-q}{dq}} \\
&+ {s/2 \choose 1} {s/2 \choose 2} \frac{2}{dq}\sqrt{\frac{1-q}{dq}}\Bigg]\xi_{d,q} \\
&+ {s/2 \choose 2}\left[\frac{2(1-q)}{d^2q^2}+\frac{1-q}{dq}+{s/2 \choose 2}\frac{2(1-q)}{d^2q^2}\right] \xi_{d,q}^2\\
& -{s/2 \choose 1} {s/2 \choose 2}\frac{2(1-q)}{d^2q^2} \xi_{d,q}^2 +\dots.
\end{split}
\end{equation*}
\end{small}

where we used the identity $A_1(s/2) = \frac{\frac{s}{2}(\frac{s}{2}-1)}{2}$.
Since $\mathbb{P}(\xi_{d,q}\geq \sqrt{\frac{dq}{1-q}}) = \mathcal{O}(e^{-dq/2(1-q)})$ and the function $S$ is non-negative and bounded by $1$ according to \eqref{eq-bdd-series}, we have
\begin{small}
\begin{equation*} 
\begin{split}
& \mathbb{E}\left[(1 + \sqrt{\frac{1-q}{dq}}\xi_{d,q})^{s/2}S(\frac{1}{2}(dq + \sqrt{dq(1-q)}\xi_{d,q}),s/2)\right]\\
=& \mathcal{O}(e^{-dq/2(1-q)}) + 1+{s/2 \choose 2}\frac{2}{dq}\\
&+ {s/2 \choose 2}\Bigg[\frac{2}{dq}\frac{1-q}{dq}+\frac{1-q}{dq}+{s/2 \choose 2}\frac{2}{dq}\frac{1-q}{dq}\\
&-{s/2 \choose 1} \frac{2}{dq}\frac{1-q}{dq}\Bigg] +\dots. \\
=& 1 + {s/2 \choose 2}\frac{3-q}{dq} + o(\frac{1}{d}),
\end{split}
\end{equation*}
\end{small}
where we used the fact that $\mathbb{E}(\xi_{d,q}^k)\to \mathbb{E}(Z^k)$ as $d\to \infty$ for any fixed $k$ implied by \eqref{weak-conv} where $Z$ is a standard-normal variable in $\mathbb{R}$ which satisfies $\mathbb{E}(Z)=0$ and $\mathbb{E}(Z^2)=1$. Then, it follows from \eqref{eq-inverse-target-linear} that 
\begin{eqnarray*}
    \frac{(\frac{d}{2})^{s/2}}{{I}_{1,q}(s,d)} &=& q^s\frac{1}{\mathbb{E}(F_{d}(B_d))}\\
&=& (\frac{q}{2})^{s/2}( 1 - {s/2 \choose 2}\frac{3-q}{dq} + o(\frac{1}{d}))\\
&=& (\frac{q}{2})^{s/2}\left(1 - \frac{(3-q)s(s-2)}{8dq} + o(\frac{1}{d})\right),
\end{eqnarray*}  
which implies 
\begin{eqnarray*}
     \bar{\sigma}_{1,q}^2(s,d) &=& \frac{1}{\sqrt[s/2]{{I}_{1,q}(s,d)}}\\
&=& \frac{q}{d} \left(1 + \frac{(3-q)(2-s)}{4dq}  + o(\frac{1}{d})\right),
\end{eqnarray*} 
which completes the proof for the case $s\in (0,2]$. For $s=2$, \eqref{def-I-linear-drop} simplifies to
\begin{equation}
    I_{1,q}(2,d)=\frac{1}{q^2}2 \sum_{t=0}^{d} h_d(t) \frac{t}{2}
    = \frac{d}{q}
\end{equation}
where we used $\Gamma(\frac{t}{2}+1) = \frac{t}{2}\Gamma(\frac{t}{2})$ and the fact that $\mathbb{E}(H_d)=qd$. This leads to ${\bar{\sigma}}_{1,q}(2,s) = q/d$ as desired. 
\end{proof}

\section{Proof of Theorem \ref{thm-almost-sure}}
\begin{proof} We first consider the ReLU case where $a=0$. In this case, the fact that $x^{(k)}$ goes to zero a.s. follows from a relatively simple argument. After a simple computation (see Lemma \ref{lem-zero-proba}), we find that $ \mathbb{P}(x^{(k)}= 0) = 1- (1-\frac{1}{2^d})^k 
$ regardless of the choice of $\sigma>0$. This implies that for all $\varepsilon>0$, 
$$\sum_{k\geq 1}\mathbb{P}(|x^{(k)}| > \varepsilon) \leq \sum_{k\geq 1} (1-\frac{1}{2^d})^k <\infty.$$

Therefore, $x^{(k)} \to 0$ almost surely. We next consider the $a\in (0,1]$ case and build on the theory of iterated random Lipschitz maps. Recall that the layer outputs obey the stochastic recursion
\beq 
x^{(k+1)}= F^{(k+1)}(x^{(k)})
= M_{W^{(k+1)},a} (x^{(k)})
\label{eq-stoc-recur}
\eeq
where $M_{W,a}(x) := \phi_a(Wx)$. 
We also note that for a non-negative random variable
$$\mathbb{E} \log(M) = \frac{d}{ds}\mathbb{E}(M^{s})|_{s=0}\,,$$
when the expectations are finite. Therefore, choosing $M = \|\phi_a(We_1)\|$, 
\beq \mu_a(\sigma) =  \frac{d}{ds}\mathbb{E}
\|\phi_a(We_1)\|^s|_{s=0}\,.
 \eeq 
Then, similar to the proof of Corollary \ref{coro-crit-leaky}, we consider $\kappa(s) = \sigma^s I_a(s,d)= \mathbb{E}
\|\phi_a(We_1)\|^s$ where $I_a(s,d)$ is defined by \eqref{def-iasd-theorem}. The function $\kappa(s)$ is convex and continuously differentiable (see the proof of Corollary \ref{coro-crit-leaky}). Notice that $\mu_a(\sigma) = \kappa'(0)$ and $\kappa(0)=1$. 
If $\mu_a(\sigma) <0$, then $\kappa(s)<1$ for $s>0$ small enough. Since $\kappa(s)= \mathbb{E}
\|\phi_a(We_1)\|^s$ goes to infinity as $s$ goes to infinity, we conclude that there exists $s_*>0$ such that $\kappa(s_*)=1$. From the definition of the $\kappa$ function, this is equivalent to saying $\sigma = \frac{1}{\sqrt[s_*]{I_a(s_*,d)}} = \bar{\sigma}_{a}(s_*,d)$ for some $s_*>0$. Correspondingly, if $\sigma = \frac{1}{\sqrt[s_*]{I_a(s_*,d)}} = \bar{\sigma}_{a}(s_*,d)$ for some $s_*>0$, then $\kappa(s_*)=1$ and since $\kappa(0)=1$, by convexity of $\kappa$ we find that $\kappa(s)<1$ for $s\in(0,s_*)$ which implies $\mu_a(\sigma) = \kappa'(0)<0$. For $s\in (0,s_*)$, by Corollary \ref{coro-crit-leaky} we have $\bar{\sigma}_a(s_*,d) < \bar{\sigma}_a(s,d) $. If apply part $(ii)$ of Theorem \ref{thm-leaky-relu} with the fact that $\sigma = \bar{\sigma}_a(s_*,d) < \bar{\sigma}_a(s,d)$, then we obtain $\mathbb{E}(\|x^{(k)}\|^s) \to 0$, i.e. $x^{(k)}$ converges to zero in the space $L_s$. 




We next prove that $x^{(k)}$ has a subsequence that converges to zero a.s. when $\mu_a(\sigma)<0$. From Theorem \ref{thm-log-leaky-relu}, we see that for any constant $C>0$, we have
\begin{small}
\begin{equation}
    \begin{split}
        &\lim_{k\to\infty}\mathbb{P}(\|x^{(k)}\| > C)\\
        =& \lim_{k\to\infty}\mathbb{P}(\log\|x^{(k)}\| > \log(C)) \\
=& \lim_{k\to\infty}\mathbb{P}\left(\frac{\log\|x^{(k)}\| - \mu_a(\sigma) k}{\sqrt{k}}> \frac{\log(C)}{\sqrt{k}} - \mu_a(\sigma) \sqrt{k}\right) \\
=& \begin{cases} 
0 & \mbox{if} \quad \mu_a(\sigma) < 0, \\
1 & \mbox{if} \quad \mu_a(\sigma) > 0. 
\end{cases}
    \end{split}
    \label{eq-prob-limit}
\end{equation}     
\end{small}
We have two cases, depending on the sign of $\mu_a(\sigma)$.
\begin{itemize}
    \item [$(i)$]  $(\mu_a(\sigma) < 0)$: In this case, for $C=1/2$, based on \eqref{eq-prob-limit}, we can choose $n_1$ large enough so that 
$$ \mathbb{P}(\|x^{(n_1)}\| > \frac{1}{2})\leq  \frac{1}{2}.$$
Continuing by a recursive fashion choose $n_k$ large enough such that 
$$ \mathbb{P}(\|x^{(n_k)}\| > \frac{1}{2^k})\leq  \frac{1}{2^k}$$
with $n_1 < n_2 < \dots < n_k$. Then the event $A_k= \{\|x^{(n_k)}\| > \frac{1}{2^k} \}$ is such that $\sum_k \mathbb{P}(A_k) < \infty$. By the Borel-Cantelli lemma, we find that
$$ \mathbb{P}\left(\limsup_{k\to\infty} \{\|x^{(n_k)}\| > \frac{1}{2^k}\} \right) = 0.
$$
This proves that for any $\varepsilon>0$ given $\mathbb{P}(\|x^{(n_k)}\| \geq \varepsilon \mbox{ infinity often})=0$ which is equivalent to
saying $\mathbb{P}(\|x^{(n_k)}\| <  \varepsilon) = 1$ or yet equivalently $x^{(n_k)}\to 0$ almost surely.

In the special case when $s_*>1$, we can stronger results. In particular, we can consider
\begin{small}
\begin{equation}
    \begin{split}
         \sum_{j=0}^\infty \mathbb{E}\| x^{(j+1)}-x^{(j)}\| \leq& \sum_{j=0}^\infty \left( \mathbb{E}\| x^{(j+1)}\| + \mathbb{E}\| x^{(j)}\| 
     \right) \\
    \leq & 2 \sum_{j=0}^\infty  \mathbb{E}\| x^{(j)}\|    < \infty
    \end{split}
\end{equation}
\end{small}
where we applied part $(ii)$ of Theorem \ref{thm-leaky-relu} with the fact that $\sigma = \bar{\sigma}_a(s_*,d) < \bar{\sigma}_a(1,d)$. Then, by \cite[Lemma 1]{steinsaltz1999locally}, $x^{(k)}$ converges almost surely to a limit. Since the subsequence $x^{(n_k)}$ converges to zero, we obtain that $x^{(k)}$ converges a.s. to zero.

\item [$(ii)$] ($\mu_a(\sigma)>0$): The proof follows from a similar approach to part $(i)$.When $\mu_a(\sigma)>0$, we see from Theorem \ref{thm-leaky-relu} that all the moments $\mathbb{E}(\|x^{(k)}\|^\alpha)$ diverges for any $\alpha >0$ (because if it were not, then $\sigma = \sigma_{a}(s,d)$ for some $s>0$ which would imply $\mu_a(\sigma)<0$ by the discussion above). Furthermore, based on \eqref{eq-prob-limit}, $x^{(k)}$ diverges to infinity in probability and we can choose a subsequence $\bar{n}_k$ such that $$ \mathbb{P}(\|x^{(\bar{n}_k)}\| > 2^k)\geq 1- \frac{1}{2^k}$$
with $\bar{n}_1 < \bar{n}_2 < \dots < \bar{n}_k$.
Then the event $\bar{A}_k= \{\|x^{(n_k)}\| < 2^k \}$ is such that $\sum_k \mathbb{P}(\bar{A}_k) < \infty$. 
By the Borel-Cantelli lemma, we find that
$$ \mathbb{P}\left(\limsup_{k\to\infty} \{\|x^{(n_k)}\| < 2^k\} \right) = 0.
$$
This proves that for any $\varepsilon>0$ given $\mathbb{P}(\|x^{(n_k)}\| \leq \varepsilon \mbox{ infinity often})=0$ which is equivalent to
saying $\mathbb{P}(\|x^{(n_k)}\| >  \varepsilon) = 1$ or yet equivalently $x^{(n_k)}\to \infty$ almost surely.
\end{itemize}


\end{proof}


\section{Proof of Theorem \ref{thm-heavy-tail}}
\begin{proof} Due to the addition of Gaussian noise to post-activations, we have the recursion over the layers
\begin{equation}
\tilde{x}^{(k+1)}=  M_{W^{(k+1)},\xi^{(k+1)}}(\tilde{x}^{(k)}):=W^{(k+1)}\tilde{x}^{(k)} +\xi^{(k+1)}
\end{equation}
where $\tilde{x}^{(k)}$ denotes the input to the $(k+1)$-st layer and $\xi^{(k)}$ is a random vector with components $\xi^{(k)}_i$ that are i.i.d. mean zero random variables. The map $M_{W^{(k)},\xi^{(k)}}$ is a random Lipschitz (linear) map whose convergence behavior has been studied in the literature. If the following conditions hold
\beq 
\mathbb{E}\left[ \max\left(0,\log(\|W^{(k+1)}\|)\right)\right] < \infty,\\ 
\mathbb{E} \left[\max\left(0,\log(\|\xi^{(k+1)}\|)\right)\right]<\infty, \label{eq-as-cond-1}\\
c_1 =\inf_k \frac{1}{k}\mathbb{E}\log\|W^{(k)}W^{(k-1)}\dots W^{(1)}\| < \infty,
\label{eq-as-cond-2}
\eeq
then it is known that $\tilde{x}^{(k)}$ admits an almost sure limit $\tilde{x}^{(\infty)}$
in which case the limit is given by 
the formula 
\beq x^{(\infty)}=\sum_{j=1}^\infty \left( \prod_{i=1}^{j-1}W^{(i)}\right) \xi^{(j)},
\label{eq-series-sum-as}
\eeq
(see e.g. \cite[Thm. 2.1]{diaconis1999iterated}). We check the conditions in \eqref{eq-as-cond-1} and \eqref{eq-as-cond-2}. The second condition in \eqref{eq-as-cond-1} is satisfied by the assumption on the noise $\xi^{(k)}$, and the first condition in \eqref{eq-as-cond-1} is satisfied as
\begin{equation}
    \begin{split}
        &\mathbb{E}\left[ \max\left(0,\log(\|W^{(k+1)}\|)\right)\right]\\
        \leq& \mathbb{E}\left[ \max\left(0, \frac{1}{2}\log\left( \sum_{i,j=1}^d (W_{ij}^{(k+1)})^2\right)\right)\right] < \infty\,,
    \end{split}
\end{equation}

where we used the fact that $ \sum_{i,j=1}^d (W_{ij}^{(k+1)})^2$ is a chi-square distribution with $d^2$ degrees of freedom. Finally, the condition \eqref{eq-as-cond-2} is equivalent to
\beq 
c_1 =  \inf_k \frac{1}{k} \mathbb{E}\log \frac{\|x^{(k)}\|}{\|x^{(0)}\|}, 
\eeq
where $x^{(k)}$ are the iterations without noise, i.e. $x^{(k)}$ satisfies $x^{(k+1)} = W^{(k+1)}x^{(k)}$ starting from $x^{(0)}$. It follows from the analysis of Theorem \ref{thm-log-leaky-relu} that we have also
\beq 
c_1 = \mu_1. 
\label{eq-c1-vs-mu1}
\eeq
Due to the choice of $\sigma=\bar{\sigma}_1(s,d)$, by Theorem \ref{thm-almost-sure}, we have also $\mu_1 < 0$. We conclude from \eqref{eq-c1-vs-mu1} that $c_1 <0$ and \eqref{eq-as-cond-2} is also satisfied. Hence, having checked that assumptions \eqref{eq-as-cond-1}--\eqref{eq-as-cond-2} hold, we conclude that the limit $\tilde{x}^{(\infty)}$ exists, it is non-zero and is given by the series sum \eqref{eq-series-sum-as}. With the addition of i.i.d. noise to activations, moments cannot grow slower; i.e. it is not hard to show that 
$$\mathbb{E}(\|\tilde{x}^{(k)}\|^p)\geq  \mathbb{E}(\|{x}^{(k)}\|^p)$$
with the same initialization i.e. $ x^{(0)} = \tilde{x}^{(0)}$. By Theorem \ref{thm-leaky-relu}, we also know that 
$
\mathbb{E}(\|x^{(k)}\|^p)\to\infty$ for any $p>s$ as $k \to \infty$. Therefore we conclude that 
$$\mathbb{E}(\|\tilde{x}^{(k)}\|^p)\to \infty, \quad \mbox{for any } p>s, $$
as $k\to \infty$. Then, we have necessarily $\mathbb{E}(\|\tilde{x}^{(\infty)}\|^p)=\infty$ because otherwise $\tilde{x}^{(k)}$ would converge to $\tilde{x}^{(\infty)}$ in $L_p$ which would be a contradiction as $\mathbb{E}(\|\tilde{x}^{(k)}\|^p)\to \infty$. This proves that the limit $\tilde{x}^{(\infty)}$ is heavy tailed in the sense that its moments of order $p$ are infinite for any $p>s$. In particular for $s<2$, this implies that the variance of the limit $\tilde{x}^{(\infty)}$ is infinite. This completes the proof. 
\end{proof}

\section{A Supporting Lemma}
\begin{lemma}\label{lemma-var-mix}
Let $M_i$ be the random variables and $p_i$ be the constant weights. Let $M$ be the mixture distribution $M:= \sum_{i} p_iM_i$. We have
\begin{small}
\begin{equation*}
    \mathrm{var}(M)=\sum_i p_i \mathrm{var}(M_i) + \sum_i p_i (\mathbb{E}[M_i])^2 - \left( \sum_i p_i \mathbb{E}[M_i]\right)^2
\end{equation*}
\end{small}
\end{lemma}
\begin{proof}
Let $\mu^{(r)}$ denote the $r$-th (raw) moment of $M$, and $\mu^{(r)}_i$ the $r$-th moment of $M_i$. Then we obtain
\begin{equation*}
    \mu^{(r)}=\sum_i p_i \mathbb{E} [M_i^r]=\sum_i p_i \mu^{(r)}_i.
\end{equation*}
The variance of $M$ can be written as
\begin{equation*}
    \mbox{var}(M)=\mu^{(2)} -\left( \mu^{(1)} \right)^2=\sum_i p_i \mu^{(2)}_i-\left(\sum_i p_i \mu^{(1)}_i\right)^2.
\end{equation*}
Since $\mu^{(2)}_i=\mbox{var}(M_i)+(\mu^{(1)}_i)^2$, we have
\begin{small}
\begin{equation*}
\begin{split}
    \mbox{var}(M) =& \sum_i p_i (\mbox{var}(M_i)+(\mu^{(1)}_i)^2)-\left(\sum_i p_i \mu^{(1)}_i\right)^2\\
    =& \sum_i p_i \mbox{var}(M_i) + \sum_i p_i (\mathbb{E}[M_i])^2 \\
    &- \left( \sum_i p_i \mathbb{E}[M_i]\right)^2.
    \end{split}
\end{equation*}
\end{small}
\end{proof}

\section{Extensions of results to Convolutional networks}
For a convolutional layer, we can write the process as
\begin{equation*}
    x^{(k+1)}=\phi_a(W^kx^k+b^k),
\end{equation*}
where $x^k$ is a $m_k^2c_k\times 1$ vector which represents co-located $m_k\times m_k$ pixels in $c_k$ input channels, where $m_k$ here is the spatial filter size of the layer $k$. If we introduce the quantities $d_k=m_k^2c_k$, and $n_k$ as the number of filters in layer $k$ then $W^k$ is a $n_k\times d_k$ matrix and each row of $W^k$ represents the weights of one filter. Moreover, we also have $c_{k+1}=n_k$ by the definition. Therefore, we can use our method to initialize the convolutional neural networks where we take $d_k=m_k^2 c_k$.

\section{Further Numerical Experiments and Illustrations}\label{sec-numerical-appendix}
\subsection{Numerical Illustrations}
In this section, we present additional figures and numerical experiments that were not part of the main text due to space considerations.



\begin{figure}[h!]
\centering
    \subfloat[Linear]{
    \includegraphics[width=0.45\columnwidth]{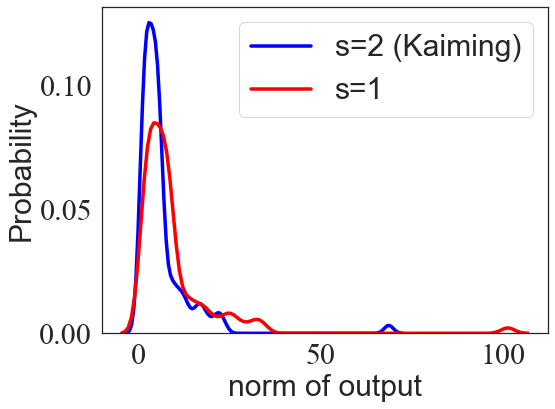}
    \label{fig:output1}
    }
    \hfill
    \subfloat[ReLU]{
    \includegraphics[width=0.45\columnwidth]{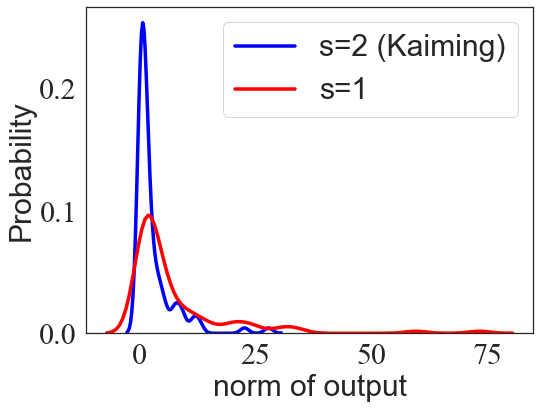}
    \label{fig:output2}
    }
    \hfill
    \subfloat[Leaky-ReLU]{
    \includegraphics[width=0.45\columnwidth]{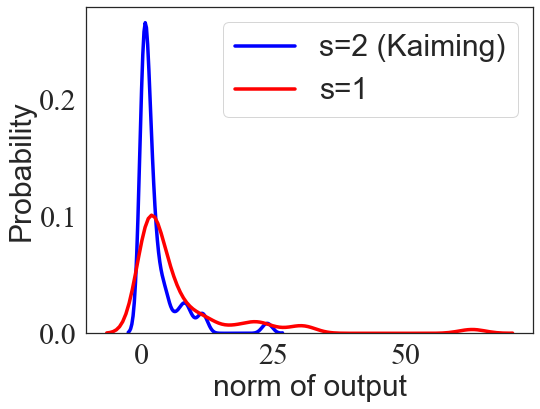}
    \label{fig:output3}
    }
    \hfill
\caption{Distribution of norm of the output $\|x^{(k)}\|$ through $k=100$ layers. \ref{fig:output1}: Probability density of $\|x^{(k)}\|$ for linear activation, where we set $\sigma^2=\frac{1}{d}+\frac{1}{2d^2}\approx \bar{\sigma}_a^2(s,d)$ with $a=0$ and $s=1$ with our initialization. \ref{fig:output2}: Probability density of $\|x^{(k)}\|$ for ReLU activation, where we set $\sigma^2=\frac{2}{d}+\frac{5}{2d^2}\approx \bar{\sigma}_a^2(s,d)$ with $a=0$ and $s=1$ with our initialization. \ref{fig:output3}: Probability density of $\|x^{(k)}\|$, where we set $\sigma^2\approx \bar{\sigma}_a^2(s,d)$ with $a=0.01$ and $s=1$ with our initialization. Kaiming initialization corresponds to $\sigma^2 =\bar{\sigma}_a^2(s,d)$ for $s=2$.}
\label{fig:output}
\end{figure}
\textbf{Distribution of the network output.} 
The distribution of the natural logarithm of the norm of the output $R_{k,0}$ is plotted in Figure \ref{fig:relu-log-out} in the main text. Figure \ref{fig:output} illustrates the distribution the norm of the $k$-th layer output for linear, ReLU and Leaky ReLU activations which supplements Figures \ref{fig:relu-log-out} and \ref{fig:logoutputb}.
The distribution is obtained from the samples by standard kernel density estimation methods provided in the Python package \verb+seaborn+.\protect\footnotemark[1] We observe that our initialization leads to heavier tails compared to Kaiming initialization, where the frequency of small outputs is less frequent in our method compared to Kaiming initialization. 

\footnotetext[1]{This package is publicly available at \url{https://seaborn.pydata.org/}.}

\subsection{Numerical Experiments}
\begin{figure}[h!]
\centering
    \subfloat[Train loss]{
    \includegraphics[width=0.45\columnwidth]{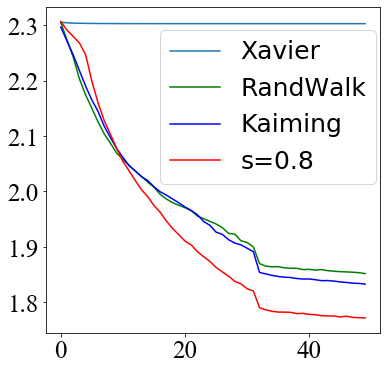}
    \label{fig:image_relu_1}
    }
    \hfill
    \subfloat[Test loss]{
    \includegraphics[width=0.45\columnwidth]{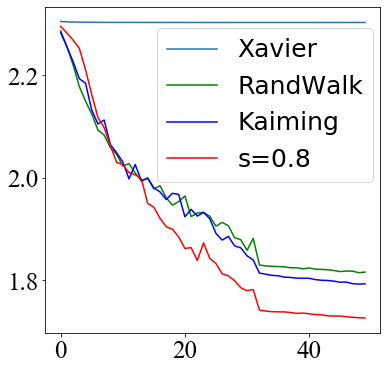}
    \label{fig:image_relu_2}
    }
    \hfill
    \subfloat[Train accuracy]{
    \includegraphics[width=0.45\columnwidth]{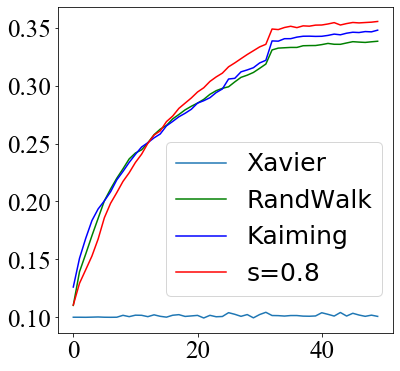}
    \label{fig:image_relu_3}
    }
    \hfill
    \subfloat[Test accuracy]{
    \includegraphics[width=0.45\columnwidth]{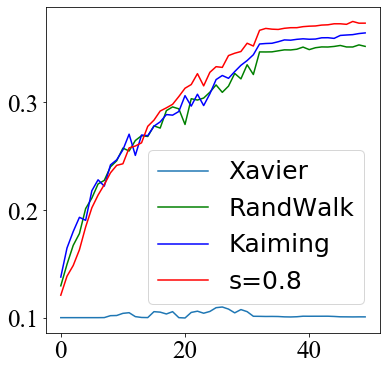}
    \label{fig:image_relu_4}
    }\\
    \hfill\\
    \begin{tabular}{|l|l|l|l|l|}
\hline
\multirow{2}{*}{} & train loss     &        & test loss       &        \\ \cline{2-5} 
                  & mean           & std    & mean            & std    \\ \hline
Xavier            & 2.3026         & 1.1405 & 2.3026          & 1.1985 \\ \hline
Randwalk          & 1.8519         & 0.1231 & 1.8158          & 0.1367 \\ \hline
Kaiming           & 1.833          & 0.098  & 1.793           & 0.109  \\ \hline
s=0.8             & \textbf{1.772} & 0.1294 & \textbf{1.7264} & 0.1469 \\ \hline
                  & train acc      &        & test acc        &        \\ \hline
                  & mean(\%)       & std    & mean(\%)        & std    \\ \hline
Xavier            & 10.07          & 0.0006 & 10.07           & 0.0014 \\ \hline
Randwalk          & 33.84          & 0.0346 & 35.21           & 0.0386 \\ \hline
Kaiming           & 34.8           & 0.0304 & 36.46           & 0.0329 \\ \hline
s=0.8             & \textbf{35.55} & 0.0431 & \textbf{37.37}  & 0.05   \\ \hline
\end{tabular}
\caption{Fully connected network with ReLU activation on CIFAR-10. The results are the $average$ over 10 samples. The x-axis is epoch number.}
\label{fig:image_relu}
\end{figure}

\begin{figure}[ht!]
\centering
    \subfloat[Train loss]{
    \includegraphics[width=0.45\columnwidth]{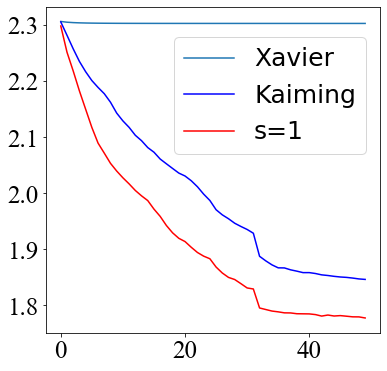}
    \label{fig:image_leaky_1}
    }
    \hfill
    \subfloat[Test loss]{
    \includegraphics[width=0.45\columnwidth]{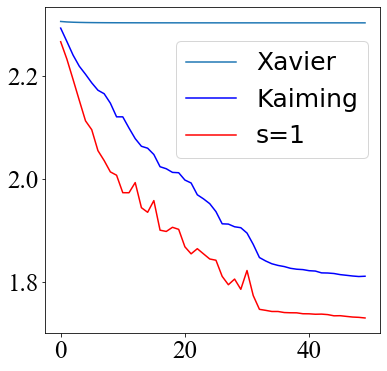}
    \label{fig:image_leaky_2}
    }
    \hfill
    \subfloat[Train accuracy]{
    \includegraphics[width=0.45\columnwidth]{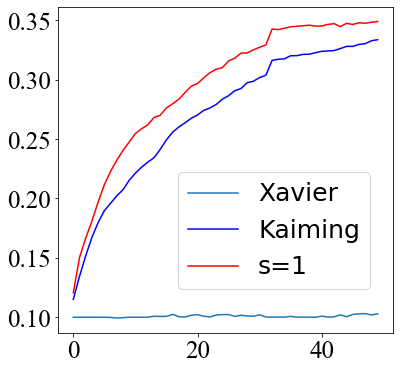}
    \label{fig:image_leaky_3}
    }
    \hfill
    \subfloat[Test accuracy]{
    \includegraphics[width=0.45\columnwidth]{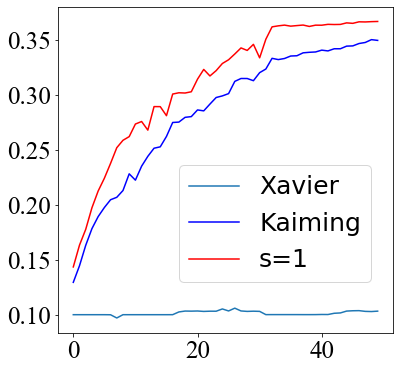}
    \label{fig:image_leaky_4}
    }\\
    \hfill\\  
    \hfill\\
\begin{tabular}{|l|l|l|l|l|}
\hline
\multirow{2}{*}{} & train loss      &        & test loss       &        \\ \cline{2-5} 
                  & mean            & std    & mean            & std    \\ \hline
Xavier            & 2.3016          & 0.0001 & 2.3026          & 0.0001 \\ \hline
Kaiming           & 1.8459          & 0.1266 & 1.8108          & 0.1383 \\ \hline
s=1               & \textbf{1.7771} & 0.1668 & \textbf{1.7298} & 0.1796 \\ \hline
\multirow{2}{*}{} & train acc       &        & test acc        &        \\ \cline{2-5} 
                  & mean(\%)        & std    & mean(\%)        & std    \\ \hline
Xavier            & 10.29           & 0.0147 & 10.32           & 0.0163 \\ \hline
Kaiming           & 33.36           & 0.0376 & 34.93           & 0.0433 \\ \hline
s=1               & \textbf{34.9}   & 0.0512 & \textbf{36.65}  & 0.0568 \\ \hline
\end{tabular}
\caption{Fully connected network with Leaky ReLU on CIFAR-10. The plots are averages over 10 runs, where mean and standard deviation (std) are also reported. The $x$-axis is the epoch number.}
\label{fig:image_leaky}
\end{figure}
\begin{figure}[ht!]
\centering
    \subfloat[Train loss]{
    \includegraphics[width=0.45\columnwidth]{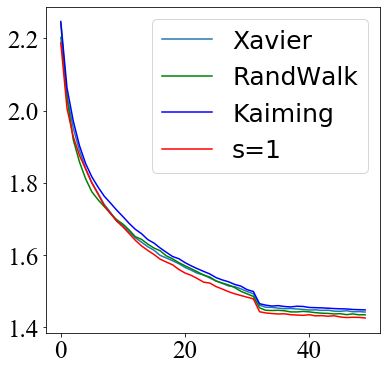}
    \label{fig:image_linear_1}
    }
    \hfill
    \subfloat[Test loss]{
    \includegraphics[width=0.45\columnwidth]{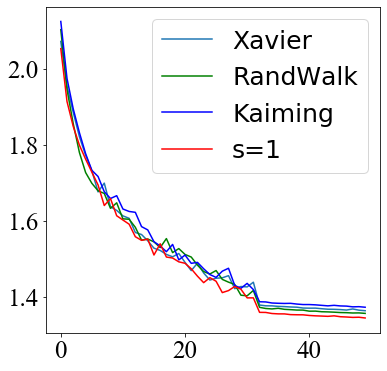}
    \label{fig:image_linear_2}
    }
    \hfill
    \subfloat[Train accuracy]{
    \includegraphics[width=0.45\columnwidth]{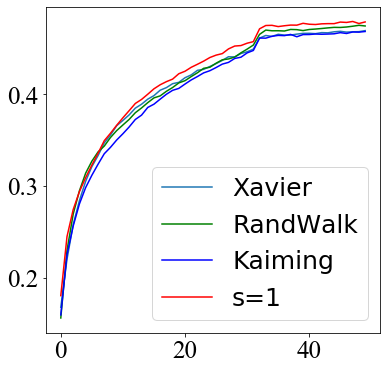}
    \label{fig:image_linear_3}
    }
    \hfill
    \subfloat[Test accuracy]{
    \includegraphics[width=0.45\columnwidth]{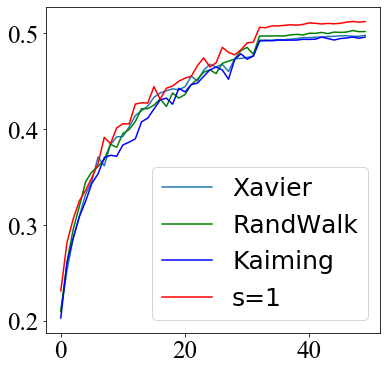}
    \label{fig:image_linear_4}
    }\\
    \hfill\\
\begin{tabular}{|l|l|l|l|l|}
\hline
\multirow{2}{*}{} & train loss      &        & test loss       &        \\ \cline{2-5} 
                  & mean            & std    & mean            & std    \\ \hline
Xavier            & 1.442           & 0.0169 & 1.3636          & 0.0103 \\ \hline
Randwalk          & 1.4344          & 0.0128 & 1.3565          & 0.0104 \\ \hline
Kaiming           & 1.4479          & 0.0124 & 1.3727          & 0.0088 \\ \hline
s=1               & \textbf{1.4255} & 0.0236 & \textbf{1.3445} & 0.0239 \\ \hline
\multirow{2}{*}{} & train acc       &        & test acc        &        \\ \cline{2-5} 
                  & mean(\%)        & std    & mean(\%)        & std    \\ \hline
Xavier            & 46.92           & 0.0059 & 49.79           & 0.0043 \\ \hline
Randwalk          & 47.43           & 0.0033 & 50.18           & 0.0031 \\ \hline
Kaiming           & 46.82           & 0.0053 & 49.61           & 0.0066 \\ \hline
s=1               & \textbf{47.87}  & 0.0096 & \textbf{51.22}  & 0.0114 \\ \hline
\end{tabular}
\caption{Linear fully connected network on CIFAR-10. The results are the $average$ over 10 samples. The x-axis is epoch number.}
\label{fig:image_linear}
\end{figure}

\begin{figure}[ht!]
\centering
    \subfloat[Train loss]{
    \includegraphics[width=0.45\columnwidth]{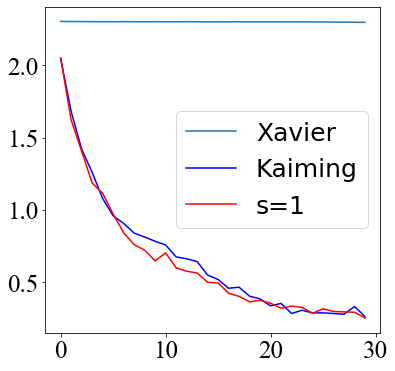}
    \label{fig:leaky_1}
    }
    \hfill
    \subfloat[Test loss]{
    \includegraphics[width=0.45\columnwidth]{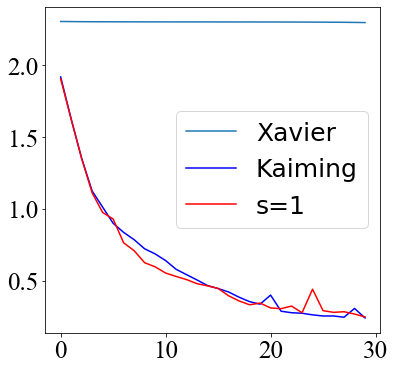}
    \label{fig:leaky_2}
    }
    \hfill
    \subfloat[Train accuracy]{
    \includegraphics[width=0.45\columnwidth]{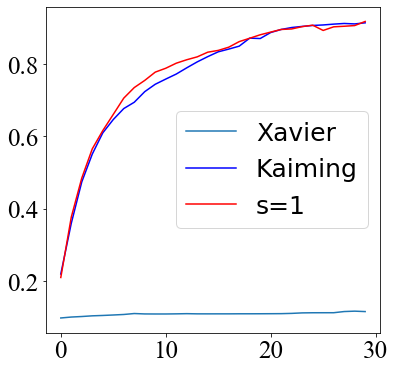}
    \label{fig:leaky_3}
    }
    \hfill
    \subfloat[Test accuracy]{
    \includegraphics[width=0.45\columnwidth]{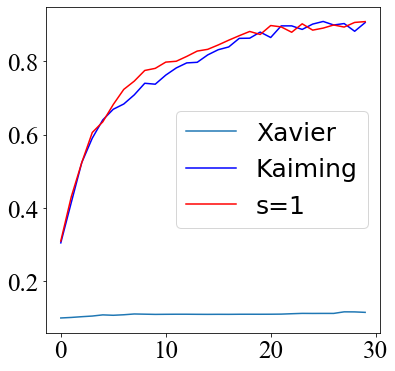}
    \label{fig:leaky_4}
    }\\
    \hfill\\
    
\begin{tabular}{|l|l|l|l|l|}
\hline
\multirow{2}{*}{} & \multicolumn{2}{l|}{train loss} & \multicolumn{2}{l|}{test loss} \\ \cline{2-5} 
                  & mean                & std       & mean               & std       \\ \hline
Xavier            & 2.2975              & 0.0157    & 2.2966             & 0.0196    \\ \hline
Kaiming           & 0.2613              & 0.1027    & 0.2394             & 0.1096    \\ \hline
s=1               & \textbf{0.2518}     & 0.1459    & \textbf{0.2464}    & 0.1496    \\ \hline
                  & \multicolumn{2}{l|}{train acc}  & \multicolumn{2}{l|}{test acc}  \\ \hline
                  & mean(\%)            & std       & mean(\%)           & std       \\ \hline
Xavier            & 11.63               & 0.0502    & 11.57              & 0.0452    \\ \hline
Kaiming           & 91.4                & 0.0446    & 90.55              & 0.0479    \\ \hline
s=1               & \textbf{91.78}      & 0.0575    & \textbf{90.82}     & 0.0578    \\ \hline
\end{tabular}
\caption{Fully connected network with Leaky ReLU on MNIST. The results are the $average$ over 20 samples. The x-axis is epoch number.}
\label{fig:leaky}
\end{figure}

\begin{figure}[ht!]
\centering
    \subfloat[Train loss]{
    \includegraphics[width=0.45\columnwidth]{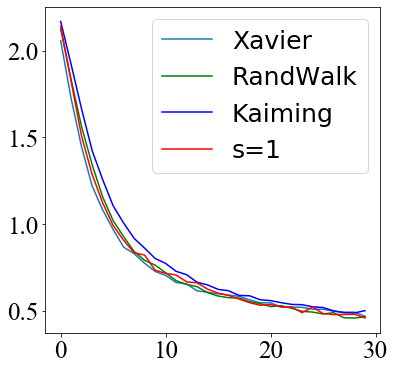}
    \label{fig:linear_1}
    }
    \hfill
    \subfloat[Test loss]{
    \includegraphics[width=0.45\columnwidth]{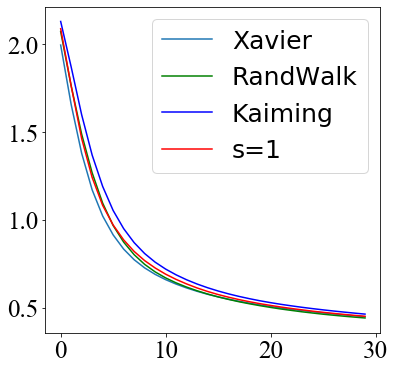}
    \label{fig:linear_2}
    }
    \hfill
    \subfloat[Train accuracy]{
    \includegraphics[width=0.45\columnwidth]{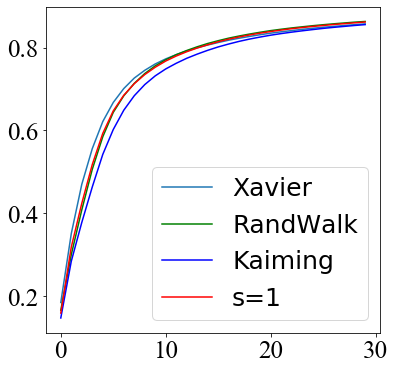}
    \label{fig:linear_3}
    }
    \hfill
    \subfloat[Test accuracy]{
    \includegraphics[width=0.45\columnwidth]{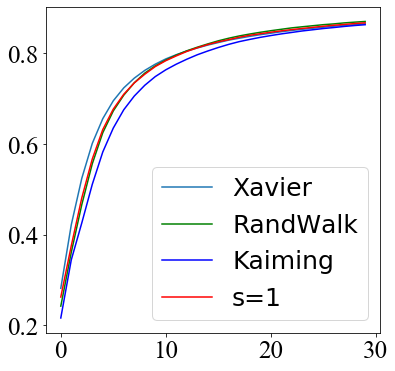}
    \label{fig:linear_4}
    }
    \\
    \hfill\\
\begin{tabular}{|l|l|l|l|l|}
\hline
\multirow{2}{*}{} & \multicolumn{2}{l|}{train loss} & \multicolumn{2}{l|}{test loss}     \\ \cline{2-5} 
                  & mean                & std       & mean                 & std         \\ \hline
Xavier            & 0.4686              & 0.0567    & 0.4485               & 0.0311      \\ \hline
Randwalk          & 0.4652              & 0.042     & \textbf{0.4414}      & 0.0317      \\ \hline
Kaiming           & 0.5011              & 0.0519    & 0.4634               & 0.0399      \\ \hline
s=1               & \textbf{0.4586}     & 0.0371    & 0.451                & 0.0348      \\ \hline
                  & \multicolumn{2}{l|}{train acc}  & \multicolumn{2}{l|}{test acc (\%)} \\ \hline
                  & mean(\%)            & std       & mean(\%)             & std         \\ \hline
Xavier            & \textbf{85.71}      & 0.0084    & 86.52                & 0.0084      \\ \hline
Randwalk          & 86.29               & 0.0096    & \textbf{87.22}       & 0.0104      \\ \hline
Kaiming           & 85.53               & 0.0121    & 86.33                & 0.12        \\ \hline
s=1               & 86.16               & 0.008     & 86.74                & 0.0089      \\ \hline
\end{tabular}
\caption{Linear fully connected network on MNIST. The results are the $average$ loss over 20 runs. The x-axis is epoch number.}
\label{fig:linear}
\end{figure}

\textbf{CIFAR-10 dataset.\protect\footnotemark[2]} 
CIFAR-10 dataset consists of 60000 32x32 colour images in 10 classes, with 6000 images per class. There are 50000 training images and 10000 test images. For ReLU activations, we compare our initializaion method with Kaiming and Xavier method as well as with the random walk initialization. However for the Leaky ReLU activation, we compare our new method with Kaiming and Xavier method only as the parameters of random walk initialization are not available for Leaky ReLU initialization. For understanding the effect of initialization on training, we report the first 50 epochs in the training process where we train our networks with stochastic gradient descent (SGD) using a constant stepsize. We tuned the SGD stepsize and used the same stepsize for each method.


Figure~\ref{fig:image_relu} shows the results of a fully connected network with ReLU activation. 
For our method in the ReLU case, we set $\sigma^2=\frac{2}{d}+\frac{3}{d^2}$, which preserves the moment $s\approx 0.8$ according to Corollary \ref{coro-crit-sigma-asymp}. 
Figure~\ref{fig:image_linear} displays the results of network with linear activation on CIFAR-10 with a similar setup where 
we set $\sigma^2=\frac{1}{d}+\frac{1}{2d^2}$ which corresponds to the choice of $s\approx 1$. Similarly, Figure \ref{fig:leaky} reports the corresponding results for Leaky ReLU. In all cases (linear, ReLU and Leaky ReLU activations), we use two convolutional layers,  20 fully-connected layers and $d=64$ for all hidden layers in this network. We consider four criteria for comparison: train loss, test loss, train accuracy, and test accuracy. We observe that our method performs no worse than other methods (Xavier initialization, Kaiming initialization, and Random walk initialization) and in many cases leads to an improvement.
\footnotetext[2]{This dataset can be downloaded from \url{https://www.cs.toronto.edu/~kriz/cifar.html}.}

\textbf{MNIST dataset.\protect\footnotemark[3]} 
\footnotetext[3]{This dataset can be downloaded from \url{http://yann.lecun.com/exdb/mnist/}.}
MNIST database is a database of handwritten digits with a training set of 60,000 examples, and a test set of 10,000 examples. The setup is similar to our experiments for CIFAR-10. In our results, we consider 20 runs. Figure \ref{fig:relu} (reported in the main text), Figure \ref{fig:leaky} and Figure \ref{fig:linear} show the performance of our method of the fully connected network with ReLU, Leaky ReLU, and linear activations respectively in the first 30 epochs. 
For the ReLU and Leaky ReLU case, we use 20 layers with $d=64$. For the linear case, we use 30 layers with $d=64$. Similar to the CIFAR-10 experiments, we set $s\approx 0.8$ for ReLU, $s=1$ for Leaky ReLU and linear case. 

Both MNIST and CIFAR-10 experiments are implemented by the Python package \verb+torch+.\protect\footnotemark[4] Our experiments are trained on Nvidia GTX 1080Ti GPU. Each experiment of MNIST takes around 3-4 hours, and each experiment of CIFAR-10 takes around 6-7 hours. 

\footnotetext[4]{This package is publicly available at \url{https://pytorch.org/}.}

\end{document}